\documentclass{article} % For LaTeX2e
\usepackage{iclr2027_conference,times}
\usepackage{titletoc}
\usepackage{amsmath,amsfonts,bm}

\newcommand{\newterm}[1]{{\bf #1}}

\def\Secref#1{Section~\ref{#1}}
\def\eqref#1{equation~\ref{#1}}
\def\Eqref#1{Equation~\ref{#1}}
\def\1{\bm{1}}

\DeclareMathAlphabet{\mathsfit}{\encodingdefault}{\sfdefault}{m}{sl}
\SetMathAlphabet{\mathsfit}{bold}{\encodingdefault}{\sfdefault}{bx}{n}

\newcommand{\E}{\mathbb{E}}

\usepackage{hyperref}
\usepackage{url}
\usepackage{comment}
\usepackage{amssymb}
\usepackage{mathrsfs}
\usepackage{amsthm}
\usepackage{enumitem}   % [label=(\roman*)] lists in the theory appendix
\usepackage{graphicx}
\usepackage{float}
\usepackage{verbatim}
\usepackage{booktabs}
\usepackage{amsmath}
\usepackage{amsfonts}
\usepackage{nicefrac}
\usepackage{microtype}
\usepackage{xcolor}
\usepackage{colortbl}
\usepackage{xspace}
\usepackage{multirow}
\usepackage{cleveref}
\usepackage{bm}

\definecolor{bestC}{RGB}{255,220,220}
\definecolor{secondC}{RGB}{255,235,204}
\definecolor{thirdC}{RGB}{255,249,204}
\newcommand{\fst}[1]{\cellcolor{bestC}\textbf{#1}}
\newcommand{\snd}[1]{\cellcolor{secondC}#1}
\newcommand{\trd}[1]{\cellcolor{thirdC}#1}

\newtheorem{theorem}{Theorem}
\newtheorem{proposition}[theorem]{Proposition}

\newtheorem{corollary}[theorem]{Corollary}

\newtheorem{assumption}{Assumption}

\newcommand{\method}{\textsc{CorrFlow}\xspace}

\newcommand{\heading}[1]{\vspace{2pt}\noindent\textbf{#1.}}

\title{\method: Self-Correcting Flow Matching}

\author{Hangwei Zhang$^{1}$, Nathan Rouill\'{e}$^{1,2}$, Chuanxia Zheng$^{1}$ \\
$^{1}$PVG, Nanyang Technological University \\
$^{2}$T\'{e}l\'{e}com Paris, Institut Polytechnique de Paris
}

\iclrfinalcopy % Uncomment for camera-ready version, but NOT for submission.
\begin{document}

\maketitle
% \iclrfinalcopy would otherwise label this arXiv version "Published as a conference paper at ICLR 2027".
\lhead{Preprint}

\begin{abstract}

% problem
Flow-matching models generate samples through iterative velocity predictions,
but typically do not revisit these predictions using the intermediate
representations that produced them.
% solution
We introduce \method, a self-correcting flow-matching framework
that reuses this information through a lightweight refiner.
% key observation and mechanism
Under rectified-flow interpolation, the error of the one-step clean
prediction equals the realized velocity residual scaled by the noise level.
Motivated by this identity, the refiner learns to predict residual
corrections from the base model's clean prediction, conditioning,
noise level and hidden states from the same forward pass.
The correction is added to the base velocity before each sampling update,
so subsequent predictions are made from the corrected state.
% efficiency
This requires no additional base-model evaluations
and only one lightweight refiner evaluation per step.
% theoretical support
Our theory gives conditions under which feature reuse outperforms
raw-input affine predictors, refiner training outperforms continued
base-model training at matched compute, and deployed correction lowers
velocity error under its own sampler law.
% empirical evidence
Experiments on class-conditional image generation,
text-to-video generation and image-to-video generation
show consistent improvements over the corresponding base models,
including gains in visual quality and physical consistency.
Project Page: \url{https://fogradio.github.io/corrflow-project}.

% People learn from mistakes by revisiting a wrong answer together with the question and reasoning that produced it.
% Flow-matching generative models discard precisely this information.
% At every sampling step, the model produces a velocity estimate with a residual error, yet that error is immediately lost.
% We show that this error is directly recoverable.
% In rectified flow, the error of the one-step clean prediction is exactly proportional to the realized velocity residual, with the noise level determining the scale.
% This identity is model-agnostic and motivates \method. We attach a lightweight refiner, about one tenth the size of the base diffusion transformer, that predicts the residual left by the base model from its prediction, text condition, noise level and hidden states from the same forward pass.
% The predicted residual corrects the base velocity inside the sampling loop, so each correction also improves the state encountered at subsequent steps.
% Our theoretical analyses establish that reusing base-model features can outperform a strictly richer class of raw-input predictors and that, under a matched compute budget, training the refiner is more effective than continued training of the base model. 
% Consistent improvements across multiple image and video generation benchmarks validate the effectiveness and generality of \method.
\end{abstract}

\section{Introduction}
\label{sec:intro}

% \begin{quote}
% \itshape
% ``Learning without thought is labor lost;
% thought without learning is perilous.''
% \par
% \end{quote}
% \vspace{-1em}
% \begin{flushright}
%     --- Confucius, \textit{Analects} 2.15, {\small Translated by James Legge}
% \end{flushright}
Flow-matching models~\citep{lipman2023flow,liu2022flow} generate samples by learning a velocity field that transports noise toward the data distribution.
Their simple regression objective has shown strong results in
image~\citep{esser2024scaling},
video~\citep{wan2025wan},
3D~\citep{xiang2024structured}
and robotic action~\citep{black2024pi_0}.
Yet generated content still exhibits structural errors, semantic mismatches and physical inconsistencies~\citep{zheng2025vbench,sun2025t2v,motamed2026generative,zhang2025boosting}.
Recent approaches improve generation quality through 
large-scale training~\citep{esser2024scaling},
representation alignment~\citep{yu2025repa},
self-supervised feature learning~\citep{chefer2026selfflow}
and reward-based optimization~\citep{yuan2026inference}.
They focus on improving the model's ability by learning only from the final answer, which is limited.
This limitation can be attributed,
in part,
to the fact that the model must \emph{predict a single velocity to determine the next sample state,
without any opportunity to review and revise its own predictions.}

% A regular float remains intact when the author block changes the first-page layout.
\begin{figure}[!htbp]
    \centering
    \includegraphics[width=0.8\linewidth]{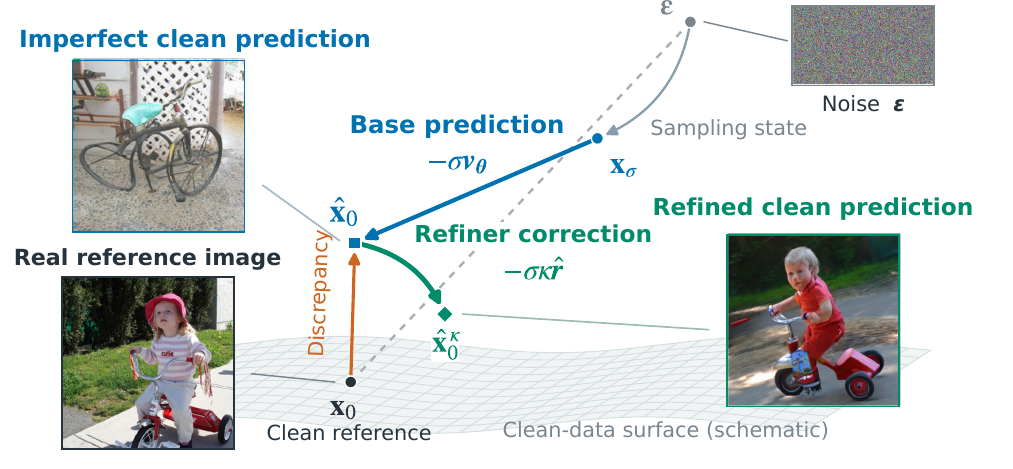}
    \caption{\textbf{\method learns from its own imperfect predictions.}
    The base model's clean prediction (blue) deviates from the clean reference
    (orange arrow).
    % Using this prediction, generation conditions and base-model
    % hidden states, the refiner estimates the velocity residual $\hat{\bm{r}}$.
    The resulting correction shifts the prediction closer to the clean reference by
    $-\sigma\kappa\hat{\bm{r}}$ (green arrow), where $\sigma$ is the noise level
    and $\kappa$ controls the correction strength.
    The initial noise is obtained by adding Gaussian noise to the clean reference.
    %Here, for simplicity, we show a single-step sampling, but the same principle applies to multi-step generation.
    }
    \label{fig:mistakeflow_intuition}
\end{figure}

We therefore exploit the property that learning should be ``\emph{refine-in-loop}'',
in the sense that if a student makes a mistake,
e.g., in an examination,
they should revisit the answer they wrote, the question they were asked, and the reasoning they followed.
By examining these pieces of context together,
the student can trace a mistake and determine how to revise it,
rather than struggling with a final answer~\citep{madaan2023self,zou2026transformer}.
Inspired by the success of \emph{self-refinement} in large language models (LLMs)~\citep{madaan2023self,shinn2023reflexion},
we apply this principle to flow matching.
A flow-based generator already computes analogous information at each sampling step:
an estimate of the clean sample$\rightarrow$answer,
the generation conditions$\rightarrow$question
and intermediate hidden states$\rightarrow$internal state.
The challenge is to define a correction target that can be supervised during training and predicted from these quantities without access to a clean reference during generation.

In this paper,
we introduce \emph{\method},
a self-correcting flow-matching framework that attaches a lightweight \emph{refiner} to the base generator.
Our \emph{core insight} is that the model's one-step clean prediction provides a concrete object for the review,
with a direct mathematical connection to the correction we seek~\citep{chen2022analog}.
Under rectified-flow interpolation,
the error between this prediction and the clean sample is exactly \emph{realized velocity residual} (\cref{eq:residual}) scaled by the noise level.
This identity connects a mistake in the predicted sample to a correction in the velocity field.
We first theoretically prove the relationship between the realized velocity residual and the clean prediction (\cref{sec:method_prelim}).
In practice, the refiner is a transformer that predicts the velocity residual from three sources of information:
  \emph{i) the current answer}, given by the one-step clean prediction,
  \emph{ii) the circumstances}, given by the generation condition and noise level,
  and \emph{iii) the internal state}, given by hidden states from the base model~\citep{epstein2023diffusion}.
  During joint training,
  clean samples and their sampled noise provide residual targets for the refiner,
  while the base model and refiner optimize separate objectives.
  During generation,
  the predicted residual adjusts the velocity used by the sampler before each state update.
  All refiner inputs reuse the same base-model forward pass,
  and each correction requires one additional refiner evaluation.

\begin{comment}
In practice,
a \emph{lightweight add-on transformer} is trained to predict the \emph{velocity residual}.
It takes as input with three parts of information available for review:
\emph{i) the current answer}, given by the one-step clean prediction,
\emph{ii) the circumstances}, given by the generation condition and noise level
and \emph{iii) the internal state}, given by hidden states from the base model~\citep{epstein2023diffusion},
and outputs a velocity residual that corrects the base model's prediction.
Specifically,
during joint training,
clean samples and their sampled noise provide residual targets for the refiner,
while the base model and refiner are each trained on their own objective under a shared training budget.
During generation,
the predicted residual adjusts the velocity used by the sampler.
The next prediction is made from the updated state,
so review and revision become part of the generation process itself.
All refiner inputs reuse the same base-model forward pass,
and each correction adds a lightweight refiner evaluation.
This design gives the additional network a focused task:
learn to revise the base model's current answer using the work already done to produce it.
\end{comment}

We empirically apply \method to a wide range of applications,
including image generation,
text-to-video generation
and image-to-video generation.
\Cref{fig:mistakeflow_intuition} illustrates how this review shifts the model's clean prediction through a correction of its velocity on image generation.
We find that training a refiner improves the base model's performance on aesthetic quality and physics consistency of generated content.
%Moreover, we also show that the refiner's effectiveness is robust to the base model's architecture and training data.
The strong performance across multiple tasks and datasets shows that velocity residual correction is generic.
%and we expect it to be applicable to other flow-based generation tasks as well. 
Our contributions are summarized as follows:
 \begin{itemize}
     \item We introduce \method, a self-correcting flow-matching framework
     that learns velocity residuals from the base model's current prediction,
  and applies the resulting corrections within the sampling loop.     \item We theoretically establish conditions under which reusing base-model
representations outperforms raw-input affine predictors, refiner training
achieves greater residual-risk reduction than continued base-model
training at matched compute, and deployed correction lowers velocity
error under its own sampler-induced state law.
     \item We evaluate \method on image and video generation benchmarks,
     demonstrating improvements in visual quality and physics consistency.
     %Controlled ablations assess the roles of the clean prediction, conditioning information and hidden-state fusion in the refiner.
 \end{itemize}

\section{Related Works}
\label{sec:related_works}

% \subsection{Improving Flow-Matching Generative Models}
% \label{sec:related_improving_flow}

\heading{Improving Flow-Matching Generative Models}
Many approaches have been proposed to improve flow-matching generative models.
For representation learning, REPA~\citep{yu2025repa} aligns hidden states with pretrained visual features,
while iREPA~\citep{singh2026irepa} emphasizes their spatial structure.
VA-VAE~\citep{yao2025vavae} aligns VAE latents with visual features,
REPA-E~\citep{leng2025repae} jointly tunes the VAE and diffusion transformer through alignment
and RAE~\citep{zheng2026rae} generates directly in pretrained representation spaces.
Other methods reuse features available within the generative model: SRA and LayerSync~\citep{jiang2026sra,haghighi2026layersync} use deeper features to guide shallower representations, and SRA$^2$~\citep{wang2026sra2} draws on existing VAE features.
Reward-based methods Flow-GRPO and DanceGRPO~\citep{liu2025flowgrpo,xue2025dancegrpo} use stochastic sampling for online reinforcement learning,
while AWM~\citep{xue2026awm} and DiffusionNFT~\citep{zheng2026diffusionnft} incorporate rewards into flow-matching objectives through advantage weighting and paired generations, respectively.
\method takes a complementary approach: it trains a \emph{lightweight added-on refiner} to predict \emph{velocity residuals} from the base model's current prediction and hidden states.
These residuals directly correct the velocity within the sampling loop,
without requiring external representation teachers or reward models for refiner training.

% \subsection{Physics-Aware Video Generation}
% \label{sec:related_physics_video}
\heading{Physics-Aware Video Generation}
Several recent works have explored physics-aware video generation with flow-matching models.
WISA~\citep{wang2025wisa},
PhysVid~\citep{pathak2026physvid}
and PhyParam~\citep{li2026phyparam} inject hierarchical, temporally local, or instance-level physical conditions.
PhysGen~\citep{liu2024physgen} and PSIVG~\citep{foo2026psivg} use simulated dynamics as guidance.
Representation-based methods transfer or jointly learn motion, geometry, semantics and latent dynamics,
including VideoREPA~\citep{zhang2025videorepa},
PhysAlign~\citep{xiong2026physalign},
VideoJAM~\citep{chefer2025videojam},
PHANTOM~\citep{shen2026phantom},
DreamWorld~\citep{tan2026dreamworld}
and MMPhysVideo~\citep{lin2026mmphysvideo}.
Other approaches use language or feedback, with PhyT2V~\citep{xue2025phyt2v} and DiffPhy~\citep{zhang2025diffphy} introducing inferred physical rules,
PhysMaster~\citep{ji2025physmaster} and PhyGDPO~\citep{cai2025phygdpo} applying pairwise or groupwise preference learning
and WMReward~\citep{yuan2026inference} steering denoising trajectories with a latent-world-model reward at inference.
In contrast, \method learns to correct the base model's own velocity residuals from its predictions and hidden states,
without explicit physical parameters or external physics evaluators.
\section{\method}
\label{sec:methodology}

\begin{figure}[!htbp]
    \centering
    \includegraphics[width=\linewidth]{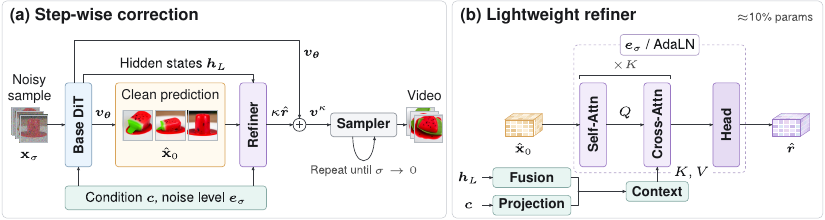}
    \caption{\textbf{\method overview for video generation.}
    (a) At noise level $\sigma$, the refiner predicts a residual $\hat{\bm{r}}$ from the base DiT's clean prediction
    $\hat{\mathbf{x}}_0=\mathbf{x}_\sigma-\sigma\bm{v}_{\bm{\theta}}$, selected hidden states $\bm{h}_L$,
    conditioning tokens $\bm{c}$ and noise-level embedding $\bm{e}_\sigma$.
    Sampling corrects the base velocity as $\bm{v}^\kappa=\bm{v}_{\bm{\theta}}+\kappa\hat{\bm{r}}$ with strength $\kappa$.
    (b) The refiner alternates self-attention over clean-prediction tokens with
    cross-attention to projected conditioning and fused hidden-state tokens.
    Noise conditioning modulates its blocks and sample-space output head,
    while 3D rotary embeddings encode video token coordinates.
    Frames show outputs and cubes denote sample tensors.}
    \label{fig:mistakeflow_pipeline}
\end{figure}

\method augments a base flow-matching model $\bm{f}_{\mathrm{base}}$ with a lightweight \emph{refiner}
that corrects its velocity $\bm{v}_{\bm{\theta}}$ at each sampling step.
We use bold upright symbols for samples, bold italic symbols for other
vectors and tensors and ordinary symbols for scalars.
The refiner takes the one-step clean prediction $\hat{\mathbf{x}}_0=\mathbf{x}_\sigma-\sigma\bm{v}_{\bm{\theta}}$,
hidden states $\{\bm{h}_\ell\}_{\ell \in L}$,
conditioning $\bm{c}$
and noise-level embedding $\bm{e}_\sigma$ to predict a velocity residual:
\begin{equation}
  \hat{\bm{r}}
  = \bm{f}_{\mathrm{refiner}}
    (\hat{\mathbf{x}}_0,\{\bm{h}_\ell\}_{\ell \in L},\bm{c},\bm{e}_\sigma),
  \qquad
  \bm{v}^{\kappa}
  = \bm{v}_{\bm{\theta}}+\kappa\hat{\bm{r}},
\end{equation}
where $\kappa$ controls the correction strength.

% This section makes the idea precise. \Secref{sec:method_prelim} fixes the flow
% matching setting and defines the quantity we call a mistake.
% \Secref{sec:method_inputs} explains what the refiner is allowed to observe and
% why. \Secref{sec:method_arch} describes the refiner network, and
% \Secref{sec:method_training} gives the objective it is trained on and the way it
% is run inside a sampler.
%This paragraph can be compressed or deleted I guess.

\subsection{Preliminaries: flow matching and the residual it leaves behind}
%This preliminaries subsection can also be compressed.
\label{sec:method_prelim}

Let $\mathbf{x}_0 \in \mathbb{R}^D$ be the continuous representation of a clean data sample,
with encoded conditioning $\bm{c}$, such as class or text embeddings.
The representation uses the sample's native coordinates or a learned feature space,
and $D$ denotes its dimension.
Let $\bm{\varepsilon} \sim \mathcal{N}(0, I_D)$ be independent noise in the same space.
For a noise level $\sigma \in [0,1]$ the rectified-flow interpolant and its derivative are
\begin{equation}
  \mathbf{x}_\sigma = (1-\sigma)\, \mathbf{x}_0 + \sigma\, \bm{\varepsilon},
  \qquad
  \frac{\mathrm{d} \mathbf{x}_\sigma}{\mathrm{d}\sigma} = \bm{\varepsilon} - \mathbf{x}_0 .
  \label{eq:interpolant}
\end{equation}
The target flow velocity is $\bm{v} = \bm{\varepsilon} - \mathbf{x}_0$,
and the Bayes velocity field is
$\bm{v}^\star(\mathbf{x},\sigma,\bm{c}) = \E[\,\bm{\varepsilon} - \mathbf{x}_0 \mid \mathbf{x}_\sigma = \mathbf{x}, \bm{c}\,]$.
The base diffusion model $\bm{f}_\text{base}$ is trained by
\begin{equation}
  \mathcal{L}_{\mathrm{base}}(\bm{\theta})
  = \E_{\mathbf{x}_0, \bm{c}, \bm{\varepsilon}, \sigma}\,
    w(\sigma)\, \big\lVert \bm{v}_{\bm{\theta}}(\mathbf{x}_\sigma,\sigma,\bm{c}) - \bm{v} \big\rVert^2 ,
  \label{eq:base_loss}
\end{equation}
with $w(\sigma)$ the noise-level weight of the base recipe,
where $\bm{v}_{\bm{\theta}}$ is the output of $\bm{f}_\text{base}$,
parameterized by $\bm{\theta}$.
Generation integrates the learned velocity field $\mathrm{d}\mathbf{x}/\mathrm{d}\sigma = \bm{v}_{\bm{\theta}}(\mathbf{x},\sigma,\bm{c})$ from
$\sigma = 1$ to $\sigma = 0$ with explicit Euler steps on a shifted grid.
The base model approximates the Bayes velocity field, while each sampled
data--noise pair provides a realized target velocity.
We call the difference between this target and the base prediction the
\newterm{realized velocity residual},
\begin{equation}
  \bm{r} \;=\; \bm{v} - \bm{v}_{\bm{\theta}}(\mathbf{x}_\sigma,\sigma,\bm{c})
       \;=\; (\bm{\varepsilon} - \mathbf{x}_0) - \bm{v}_{\bm{\theta}}(\mathbf{x}_\sigma,\sigma,\bm{c}).
  \label{eq:residual}
\end{equation}
This is the mistake we want the refiner
$\bm{f}_\text{refiner}$
to predict.
The base model's one-step clean prediction is available at generation time.
Its discrepancy from the clean target is proportional to the realized
velocity residual, as the following identity shows.

\paragraph{Clean-prediction identity.}
Taking a single Euler step from $\mathbf{x}_\sigma$ all the way to $\sigma = 0$ gives the
base model's \newterm{one-step clean prediction}, its best current guess at the
finished sample,
\begin{equation}
  \hat{\mathbf{x}}_0(\mathbf{x}_\sigma,\sigma,\bm{c}) \;=\; \mathbf{x}_\sigma - \sigma\, \bm{v}_{\bm{\theta}}(\mathbf{x}_\sigma,\sigma,\bm{c}) .
  \label{eq:zhat}
\end{equation}
Substituting \Eqref{eq:interpolant} into \Eqref{eq:zhat} and using
\Eqref{eq:residual},
\begin{equation}
  \hat{\mathbf{x}}_0 - \mathbf{x}_0
  = \sigma\big( (\bm{\varepsilon} - \mathbf{x}_0) - \bm{v}_{\bm{\theta}} \big)
  = \sigma\, \bm{r} ,
  \qquad\text{equivalently}\qquad
  \bm{r} = \frac{\hat{\mathbf{x}}_0 - \mathbf{x}_0}{\sigma}, \quad \sigma > 0 .
  \label{eq:identity}
\end{equation}
This identity holds for any predictor $\bm{v}_{\bm{\theta}}$ and shows that the clean-prediction error is the velocity residual
$\bm{r}$
scaled by the noise level
$\sigma$.
It therefore motivates using $\hat{\mathbf{x}}_0$ as the refiner's primary input, expressing the prediction in the clean data space.
Although $\hat{\mathbf{x}}_0$ is available at every sampling step, its error depends on the unknown clean target.
The refiner thus learns to estimate the residual from this prediction
$\hat{\mathbf{x}}_0$ and the accompanying features $\bm{h}_L$.

\subsection{Refiner Inputs}
\label{sec:method_inputs}

Focusing on the current answer in isolation is insufficient,
because the same prediction may be correct under one condition and a mistake under another,
or an artifact that must be repaired at high noise may be a legitimate detail at low noise.
Therefore, the refiner is given three views of the moment the mistake was made:
\begin{description}
  \item [The current answer.] The one-step clean prediction $\hat{\mathbf{x}}_0$ of \Eqref{eq:zhat}, tokenized according to the sample structure to form the refiner's trunk tokens.
  This is the base model's current estimate of the clean sample.
  \item [The circumstances.] The conditioning tokens $\bm{c}$ that encode the task requirements, and the noise level $\sigma$, which identifies the current position along the denoising trajectory.
  \item [The internal state.] The base model's hidden states $\{\bm{h}_\ell\}_{\ell \in L}$, where $L$ indexes transformer blocks selected uniformly across the model's depth. These readouts capture different stages of the computation that produced $\bm{v}_{\bm{\theta}}$, and they record how the prediction was assembled rather than only what it was.
\end{description}
Formally,
let $W = (\mathbf{x}_\sigma, \sigma, \bm{c})$
denote the base model's query state,
and let $\mathcal{Z}(W) = (\hat{\mathbf{x}}_0, \{\bm{h}_\ell\}_{\ell \in L}, \sigma, \bm{c})$
be the refiner's information.
The refiner is a network $\bm{f}_\text{refiner}$ with parameter vector $\bm{\phi}$,
written as $\bm{g}_{\bm{\phi}}$, and is trained to predict the residual,
$
  \hat{\bm{r}} \;=\; \bm{g}_{\bm{\phi}}(\mathcal{Z}(W)) \;\approx\; \bm{r} .
$
Every component of $\mathcal{Z}$ is from the base model's conditioning and its existing forward pass.
Constructing them for the refiner requires \emph{no additional base-model evaluation} and \emph{no access to the clean target at generation time}.

\subsection{Refiner architecture}
\label{sec:method_arch}

\emph{\method} is designed to be a \emph{lightweight add-on transformer} to a base flow-matching model,
which estimates the realized velocity residual $\bm{r}$.
It has its own parameters,
totaling approximately $10\%$ of the base model's parameter count.
It carries the aforementioned information as inputs.

\heading{Token embeddings} Given the clean prediction $\hat{\mathbf{x}}_0$ (\cref{eq:zhat}), the refiner $\bm{f}_\text{refiner}$ first tokenizes it according to the sample structure into $N$ tokens and projects them to the refiner width $d$,
obtaining the \underline{answer} tokens $\bm{t}_\text{ans} \in \mathbb{R}^{N \times d}$.
Our image and video implementations use patches for this tokenization.
To keep the refiner and base model on a common clock, the refiner also consumes the base model's own \underline{noise-level} embedding $\bm{e}_\sigma$ at the current noise level $\sigma$.

We also construct the context tokens $\bm{t}_\text{ctx} \in \mathbb{R}^{(N_c + N) \times d}$ from the conditioning $\bm{c}$ and the hidden states $\{\bm{h}_\ell\}_{\ell \in L}$, where $N_c$ is the number of conditioning tokens.
In particular,
a linear projection maps $\bm{c}$ to the \underline{conditioning} tokens
$\bm{t}_\text{c} \in \mathbb{R}^{N_c \times d}$.
At each token position, we concatenate the hidden states from the selected
layers along the feature dimension and project them to the refiner width $d$
using a multilayer perceptron with a GELU activation,
\begin{equation}
  \bm{t}_{\mathrm{fuse}}
  = \mathrm{MLP}\Bigl(
      \mathrm{concat}_{\ell \in L}\,\bm{h}_\ell
    \Bigr)
  \in \mathbb{R}^{N \times d}.
  \label{eq:hfuse}
\end{equation}
This fusion across layers at fixed token positions is a deliberate choice.
It keeps the \underline{hidden-state} tokens at $N$ tokens instead of $|L| \cdot N$, so attention cost does not grow with the number of readouts,
and it lets the refiner learn which depth to trust at each location rather than committing to a fixed average across depths.
To combine the two \underline{context} tokens for the refiner's cross-attention, we concatenate them along the token axis and apply layer normalization,
\begin{equation}
  \bm{t}_\text{ctx}
  = \mathrm{LayerNorm}\Bigl(
      \bigl[\, \bm{t}_\text{c} + \bm{s}_{\mathrm{cond}} \;,\; \bm{t}_\text{fuse} + \bm{s}_{\mathrm{state}} \,\bigr]
    \Bigr)
  \in \mathbb{R}^{(N_c + N) \times d},
  \label{eq:context}
\end{equation}
where $\bm{s}_{\mathrm{cond}}$ and $\bm{s}_{\mathrm{state}}$ are learned segment embeddings that mark the two sources.

\heading{Alternating-Attention blocks}
We build our \emph{lightweight add-on transformer} with a stack of $B$ Alternating Attention Blocks (AABs),
making the refiner focus on the current answer while consulting the context in an alternate manner.
Specifically,
even-indexed blocks apply \emph{self-attention} to the \underline{answer} tokens $\bm{t}_\text{ans}$,
using rotary position embeddings defined on the token coordinates of the sample representation.
Our image and video implementations use two-dimensional spatial coordinates and three-dimensional spatiotemporal coordinates, respectively.
Odd-indexed blocks apply \emph{cross-attention} to the \underline{context} tokens $\bm{t}_\text{ctx}$,
which lets each token incorporate the task requirements and the base model's intermediate representations.
%Note that, we apply positional encoding within answer-token self-attention and introduce no additional positional encoding in cross-attention.
Each attention layer is followed by a feed-forward network.
Following the base model's adaptive-normalization design, each block
uses shift, scale and gate parameters derived from the \underline{noise-level}
embedding $\bm{e}_\sigma$.
Finally, a modulated linear head projects the updated answer tokens
to residual tokens, which are reassembled into the sample's original shape.

%\cxz{please double check this part. Since, our method should suitable for arbitrary data types. Hence, I ignore the video form, and use only data samples.}

\subsection{Training and Inference}
\label{sec:method_training}

\paragraph{Objective.}
We train the base model
$\bm{f}_\text{base}$
and the refiner 
$\bm{f}_\text{refiner}$
on the same minibatch using separate objectives and a single backward pass.
In particular,
we define the refiner's loss by a combination of a magnitude term and a direction term,
\begin{equation}
  \mathcal{L}_{\mathrm{ref}}(\bm{\phi})
  = \E\Big[\;
      \big\lVert \hat{\bm{r}} - \mathrm{sg}[\bm{r}] \big\rVert^2
      \;+\; \lambda
            \big(1 - \cos\big(\hat{\bm{r}},\, \mathrm{sg}[\bm{r}]\big)\big)
    \Big],
  \label{eq:ref_loss}
\end{equation}
where $\mathrm{sg}[\cdot]$ denotes the \texttt{stop-gradient} operator,
$\hat{\bm{r}} = \bm{f}_\text{refiner}(\mathrm{sg}[\mathcal{Z}(W)])$ is the refiner output,
% \cxz{How to understand the stop gradient for the context and velocity residual}
with $\bm{r}$ the realized velocity residual defined in \Eqref{eq:residual}
and $\lambda$ controls the weight of the direction term.
We calculate the cosine similarity by flattening the overall residual field into a single vector per sample.
The two terms play complementary roles.
The squared-error term fits the residual vector, coupling its magnitude and direction.
When residual directions vary across samples with similar refiner inputs, averaging these targets can produce a correction with a smaller magnitude.
The cosine term explicitly rewards alignment with the target direction and complements this vector regression while the correction magnitude is learned.
The total training objective is
\begin{equation}
  \mathcal{L} = \mathcal{L}_{\mathrm{base}}(\bm{\theta})
              + \lambda_{\mathrm{ref}}\, \mathcal{L}_{\mathrm{ref}}(\bm{\phi}),
  \label{eq:total_loss}
\end{equation}
where $\bm{\theta}$ and $\bm{\phi}$ are the base and refiner parameters, respectively.
Here $\lambda_{\mathrm{ref}}$ controls the relative weight of the refiner objective.

%\paragraph{Why the refiner inputs are detached.}
%Every tensor the refiner reads is a stop-gradient copy, so
%$\mathcal{L}_{\mathrm{ref}}$ never reaches $\bm{\theta}$ and
%$\mathcal{L}_{\mathrm{base}}$ never reaches $\bm{\phi}$. This enforces a clean
%contract. The base remains a complete first-order velocity predictor trained by
%\Eqref{eq:base_loss} alone, and the refiner remains a correction that can be
%scaled down or removed at any time without leaving the base in a degraded state.
%Without the stop-gradient the base could quietly delegate part of its prediction
%to the refiner, and the pair would only ever be usable together.

\begin{comment}
\paragraph{Two timescales.}
A single optimizer holds two parameter groups, with learning rates
$\eta_{\bm{\theta}}$ for $\bm{\theta}$ and $\eta_{\bm{\phi}}$ for $\bm{\phi}$, where
$\eta_{\bm{\phi}} > \eta_{\bm{\theta}}$. Detaching the gradient
does not make the two problems independent. As $\bm{\theta}$ drifts, so do
$\hat{\mathbf{x}}_0$, $\bm{h}_L$ and the residual $\bm{r}$ itself, and the refiner regresses onto a
target field that moves under it. The higher refiner learning rate is intended
to help it track changes in this target as the base model evolves.
\end{comment}

\paragraph{Inference.}
The refiner is applied inside the ordinary sampling loop rather than as a
post-process on a finished sample.
At each denoising step the sampler evaluates the base model on the current sample,
captures $\bm{e}_\sigma$ and $\{\bm{h}_\ell\}_{\ell \in L}$ from that
same evaluation,
assembles $\hat{\mathbf{x}}_0$ from \Eqref{eq:zhat}
and runs the refiner once on $\mathcal{Z}(W)$.
The predicted residual is added to the base velocity with a gain $\kappa$,
\begin{equation}
  \bm{v}^{\kappa} = \bm{v}_{\bm{\theta}}(\mathbf{x}_\sigma,\sigma,\bm{c}) + \kappa\, \bm{f}_{\text{refiner}}(\mathcal{Z}(W)) ,
  \label{eq:corrected}
\end{equation}
and the corrected velocity is what the scheduler integrates.
Because $\hat{\mathbf{x}}_0$ and $\{\bm{h}_\ell\}_{\ell \in L}$ are obtained from the state the sampler has actually reached,
each correction influences subsequent refiner inputs.
This adds one refiner forward pass per step without additional
base-model evaluations.

%NEW
\section{Theoretical Analysis}
\label{sec:theory}

We analyze three complementary questions: whether frozen base-model features
can improve residual prediction beyond raw-input predictors, whether refiner
training can outperform continued base-model training at matched budget and
whether a trained correction improves velocity error under the state law it
induces at deployment.

\heading{Shared residual target}
\label{sec:theory_setting}
Let $\mathcal X$ be a finite-dimensional Hilbert sample space,
let $P$ be the law of the query state
$W=(\mathbf{x}_\sigma,\sigma,\bm c)$ and let $\widetilde P$ be the joint law of $(W,\bm r)$ at the fixed base
checkpoint used below. Assume $\bm r\in L^2(\widetilde P;\mathcal X)$ and write
$\bar{\bm r}(W)=\E_{\widetilde P}[\bm r\mid W]$. For any
$\bm f\in L^2(P;\mathcal X)$, define
\[
  \mathcal R_P(\bm f)
  :=\E_{\widetilde P}\|\bm r-\bm f(W)\|^2,
  \qquad
  \mathcal R_{\mathrm{irr}}
  :=\E_{\widetilde P}\|\bm r-\bar{\bm r}(W)\|^2.
\]

\begin{proposition}[Same-target risk decomposition]
\label{prop:pythagoras}
Every such predictor satisfies
\begin{equation}
  \mathcal R_P(\bm f)
  =
  \mathcal R_{\mathrm{irr}}
  +\E_P\|\bm f(W)-\bar{\bm r}(W)\|^2.
  \label{eq:risk_decomposition}
\end{equation}
\end{proposition}
The cross term vanishes because
$\E_{\widetilde P}[\bm r-\bar{\bm r}(W)\mid W]=0$.
See equations \ref{eq:p2-app-orthogonality} and \ref{eq:p2-app-risk-geometry}.

\heading{Frozen-feature separation}
\label{sec:theory_p1}
In the scalar Gaussian witness of Appendix~\ref{p1:sec:setup}, let
$c_1,c_2$ be independent uniform signs, put $p=c_1c_2$ and take
$x_0=\delta p+\xi$, where $\delta>0$ and $\xi\sim\mathcal N(0,1)$.
The frozen representation exposes the interaction through $h_L=p$.
With $\tau_\sigma=\sigma/[(1-\sigma)^2+\sigma^2]$, the one-parameter
refiner is $g_b=-b\tau_\sigma p$, whereas $\mathcal A$ contains raw-affine
functions of $(x_\sigma,c_1,c_2)$ with arbitrary $\sigma$-dependent
coefficients.

\begin{theorem}[Frozen-feature advantage]
\label{thm:frozen_feature}
For an explicit $\Delta_{\mathrm{aff}}>0$ defined in
Appendix~\ref{p1:sec:setup},
\[
  \mathcal R_P(g_\delta)=\mathcal R_{\mathrm{irr}},
  \qquad
  \inf_{a\in\mathcal A}\mathcal R_P(a)
  =
  \mathcal R_{\mathrm{irr}}+\Delta_{\mathrm{aff}}.
\]
If $\widehat b$ is $\epsilon_{\mathrm{opt}}$-suboptimal for
$\mathcal J_\lambda(b)
=\mathcal R_P(g_b)
+\lambda\E_{\widetilde P}L_{\cos}(g_b,r)$, with $\lambda=0.1$, then
\begin{equation}
  2\lambda+\epsilon_{\mathrm{opt}}<\Delta_{\mathrm{aff}}
  \quad\Longrightarrow\quad
  \mathcal R_P(g_{\widehat b})
  <
  \inf_{a\in\mathcal A}\mathcal R_P(a).
  \label{eq:frozen_feature_separation}
\end{equation}
\end{theorem}

\heading{Matched-resource separation}
\label{sec:theory_p2}
Returning to the finite-dimensional space $\mathcal X$, we compare two runs
from the same checkpoint $\bm v_0$: the refiner correction
$F_{\mathrm r}(\bm\alpha)
=\bm g_{\bm\phi_0+\bm\alpha}(\mathcal Z(W))$
and continued base training
$F_{\mathrm b}(\bm\alpha)
=\bm v_{\bm\theta_0+\bm\alpha}(W)-\bm v_0(W)$, with $F_j(0)=0$.
Under a common ledger $\mathcal C$, arm $j$ takes
$T_j=\lfloor C_{\mathrm{train}}/c_j\rfloor$ gradient steps and returns
$\bm f_{j,\eta}$. Write
$\Delta_j=\sup_{\eta\in\mathcal E_j}
[\mathcal R_P(0)-\mathcal R_P(\bm f_{j,\eta})]$,
$\Phi_j=DF_j(0)$, $H_j=\Phi_j^*\Phi_j$ and
$\bm b_j=\Phi_j^*\bar{\bm r}$.

\begin{theorem}[Finite-ledger refiner separation]
\label{thm:p2-matched-resource}
Under the local Jacobian, angular-gradient, radius and spectral-stability
conditions of Appendix~\ref{p2:app:full-setting}, both armwise optima are
attained and the actual reduction $\Delta_j$ differs from its tangent
approximation $\Delta_j^{\mathrm{lin}}$ by at most an explicit finite error
$e_j$:
\[
  |\Delta_j-\Delta_j^{\mathrm{lin}}|\le e_j,
  \qquad
  \Delta_j^{\mathrm{lin}}
  =
  \sup_{\eta\in\mathcal E_j}
  \bm b_j^*H_j^\dagger
  [I-(I-\eta H_j)^{2T_j}]\bm b_j.
\]
Consequently,
\begin{equation}
  \Delta_{\mathrm r}^{\mathrm{lin}}-e_{\mathrm r}
  >
  \Delta_{\mathrm b}^{\mathrm{lin}}+e_{\mathrm b}
  \Longrightarrow
  \min_{\eta\in\mathcal E_{\mathrm r}}
  \mathcal R_P(\bm f_{\mathrm r,\eta})
  <
  \min_{\eta\in\mathcal E_{\mathrm b}}
  \mathcal R_P(\bm f_{\mathrm b,\eta}).
  \label{eq:p2-actual-order}
\end{equation}
\end{theorem}

\heading{Own-law deployment}
\label{sec:theory_own_law}
Write $\bm g_{\bm\phi}$ for the deployed correction
$W\mapsto\bm g_{\bm\phi}(\mathcal Z(W))$.
Let $\mathbf X^\kappa$ be the exact trajectory from the prescribed source
under the field $\bm v_0+\kappa\bm g_{\bm\phi}$. For $\Sigma$ drawn
independently of the source and condition from a law $q$ on $[0,1]$, set
$Q_\kappa=\operatorname{Law}
(\mathbf X^\kappa(\Sigma),\Sigma,\bm C)$ and define
$\mathcal V_Q(\kappa)
=\E_Q\|\bar{\bm r}-\kappa\bm g_{\bm\phi}\|^2$.
Here $\bar{\bm r}$ denotes the deployment version fixed in
Appendix~\ref{p2:app:deployment-setting}.

\begin{theorem}[Own-law velocity improvement]
\label{thm:own-law-improvement}
Under the deployment assumptions of Appendix~\ref{p2:app:deployment},
let $M$ be the first-order alignment margin after trajectory-shift cost and
$N$ the second-order correction and alignment-shift cost, as defined in
Appendix~\ref{p2:app:transfer}. Then, for every $\kappa\in[0,\bar\kappa]$,
\[
  \mathcal V_{Q_\kappa}(\kappa)-\mathcal V_{Q_0}(0)
  \le -\kappa M+\kappa^2N.
\]
If $M>0$ (hence $N>0$), then
\begin{equation}
0<\kappa\le\bar\kappa,\quad \kappa<\frac{M}{N}
\quad\Longrightarrow\quad
\mathcal V_{Q_\kappa}(\kappa)<\mathcal V_{Q_0}(0).
\label{eq:main-own-law-gain}
\end{equation}
\end{theorem}

\section{Experiments}
\label{sec:experiments}

We evaluate \method on a wide range of tasks,
including class-conditional image generation, text-to-video generation and image-to-video generation.
The image experiments test whether learned velocity corrections improve an established flow-matching model, while the video experiments examine visual quality, temporal consistency and physical plausibility.
Red, orange and yellow highlight the best, second-best and third-best scores in each column, respectively.

\subsection{Image Generation}
\label{sec:image_generation}

\heading{Settings}
We study class-conditional generation on ImageNet at $256\times256$ resolution with 1,000 classes.
Our base model is built upon RAEv2 with a frozen DINOv3-L-based representation autoencoder and a pretrained diffusion transformer.
Starting from its 80-epoch checkpoint, we jointly fine-tune the base transformer and train a randomly initialized refiner for another 80 epochs on the ImageNet training split, denoted by $80+80$ in Table~\ref{tab:imagenet_generation_comparison_compact}.
The refiner has 10 transformer blocks with width 768 and 12 attention heads, and reads hidden states from base transformer blocks 3, 11, 19 and 27, using one-based indexing.
We retain the base model's training objective and train the refiner to predict the realized velocity residual with squared error and a cosine-direction term weighted by $0.1$.
Both networks use GMuon optimizer with a global batch size of 1,024. The evaluation settings are identical to those in \citet{singh2026improved}.

\begin{table*}[!htbp]
    \centering
    \caption{
        Comparison on ImageNet-256~\citep{deng2009imagenet} generation.
        Representation Fr\'echet Distance ($\mathrm{FD}_r$) is reported in three feature spaces.
    }
    \label{tab:imagenet_generation_comparison_compact}
    \resizebox{\textwidth}{!}{
    \begin{tabular}{@{}lcccccc@{}}
        \toprule
        \multirow{2}{*}{\textbf{Method}}
        & \multirow{2}{*}{\textbf{Epochs}}
        & \multirow{2}{*}{\textbf{Guidance}}
        & \multirow{2}{*}{\textbf{IS $\uparrow$}}
        & \multicolumn{3}{c}{\textbf{Representation Fr\'echet Distance ($\mathrm{FD}_r$) $\downarrow$}} \\
        \cmidrule(lr){5-7}
        & & &
        & \textbf{ConvNeXt}~\citep{liu2022convnet} & \textbf{DINOv2}~\citep{oquabdinov2} & \textbf{CLIP}~\citep{radford2021learning} \\
        \midrule

        % Source row: noguid-ep160_ema.
        RAEv2~\citep{singh2026improved}
        & 160 & No & 233.9
        & 1.28 & 2.22 & 7.18 \\

        SelfFlow~\citep{chefer2026selfflow}
        & 160 & No & 146.4
        & 4.81 & 7.08 & 12.13 \\
        \midrule
        \textbf{Ours}
        & 80+80 & No & 234.0
        & \trd{1.24} & 2.21 & \trd{7.14} \\
        \midrule
        SiT-XL/2~\citep{ma2024sit}
        & 800 & Yes & 256.7
        & 2.02 & 7.89 & 17.69 \\

        RAE-XL~\citep{zheng2026rae}
        & 800 & Yes & \trd{261.0}
        & 1.79 & \trd{2.11} & 7.87 \\

        % Source row: in256-k7-ft-copilot-ep80_ema.
        RAEv2~\citep{singh2026improved}
        & 160 & Yes & \snd{265.0}
        & \snd{0.74} & \snd{1.08} & \snd{4.85} \\

        SelfFlow~\citep{chefer2026selfflow}
        & 160 & Yes & 146.8
        & 4.72 & 6.98 & 12.16 \\
        \midrule
        \textbf{Ours}
        & 80+80 & Yes & \fst{266.6}
        & \fst{0.73} & \fst{1.03} & \fst{4.84} \\

        \bottomrule
    \end{tabular}
    }
\end{table*}

\heading{Results}
\Cref{tab:imagenet_generation_comparison_compact} compares \method with RAEv2 at a matched budget of 160 total epochs, because our $80+80$ schedule spends 80 epochs on the base checkpoint and 80 epochs on joint training.
Under this budget, \method improves over RAEv2 on IS and on all three $\mathrm{FD}_r$ metrics, both with and without guidance, so the gain comes from the learned residual corrections rather than from a longer training schedule.
With guidance, \method leads on every metric in the table, surpassing SiT-XL/2 and RAE-XL trained for 800 epochs.
It also outperforms the evaluated SelfFlow variant in both guidance settings.
Additional results are provided in \Cref{sec:appendix_imagenet_results}.

\subsection{Text-to-Video Generation}
\label{sec:text_to_video_generation}

\heading{Settings}
We use Wan2.1-T2V-1.3B~\citep{wan2025wan} as the base model and jointly adapt its transformer with the refiner on video--text pairs from WISA-80K~\citep{wang2025wisa} and OpenVidHD~\citep{nan2025openvid}.
The refiner has 10 transformer blocks with width 1,024 and 16 attention heads, and reads hidden states from base blocks 4, 12, 20 and 30, using one-based indexing.
The final training stage uses 66,572 clips for five epochs with AdamW and a global batch size of 16, retaining the base flow-matching objective and a residual objective with cosine-direction weight $0.1$.
We generate 81-frame videos at $480\times832$ resolution with 50 Euler steps and evaluate visual quality and temporal consistency with VBench~\citep{huang2024vbench}, and physical plausibility and prompt adherence with VideoPhy2~\citep{bansal2026videophy} and FinePhyEval~\citep{pothiraj2026physics}.
\Cref{sec:appendix_t2v_implementation} provides the training and evaluation details.

\begin{table*}[t]
\centering
\caption{Physics-consistent text-to-video generation on VideoPhy2~\citep{bansal2026videophy} and FinePhyEval~\citep{pothiraj2026physics}. Higher is better for all metrics.}
\label{tab:physics_generation}
\label{tab:videophy2}
\label{tab:finephyeval}
\setlength{\tabcolsep}{4pt}
\resizebox{\textwidth}{!}{%
\begin{tabular}{@{}l*{10}{c}@{}}
\toprule
\multirow{3}{*}{\textbf{Method}}
& \multicolumn{6}{c}{\textbf{VideoPhy2}}
& \multicolumn{4}{c}{\textbf{FinePhyEval}} \\
\cmidrule(lr){2-7}\cmidrule(lr){8-11}
& \multicolumn{3}{c}{ALL} & \multicolumn{3}{c}{HARD}
& \multirow{2}{*}{Object} & \multirow{2}{*}{Action}
& \multirow{2}{*}{Physics} & \multirow{2}{*}{Overall} \\
\cmidrule(lr){2-4}\cmidrule(lr){5-7}
& PC & SA & Joint & PC & SA & Joint & & & & \\
\midrule
Wan 2.1-14B~\citep{wan2025wan}
& 60.5 & \trd{29.0} & \trd{24.7} & 45.6 & \trd{8.3} & 3.9
& \trd{0.86} & \snd{0.53} & \fst{0.46} & \snd{0.59} \\

\midrule

PhysRAG-5B~\citep{cheng2026physrag}
& \trd{65.6} & 24.7 & 21.0 & \snd{59.4} & 7.2 & 4.4
& \snd{0.88} & \trd{0.49} & 0.37 & 0.55 \\

VideoREPA-5B~\citep{zhang2025videorepa}
& \fst{73.6} & 25.8 & 23.6 & \fst{62.8} & 6.7 & \trd{5.0}
& 0.82 & 0.37 & 0.30 & 0.47 \\

Wan 2.1-1.3B~\citep{wan2025wan}
& 57.8 & \snd{30.0} & \snd{24.8} & 36.7 & \fst{11.7} & \snd{5.6}
& 0.84 & \trd{0.49} & 0.36 & 0.54 \\

DreamWorld-1.3B~\citep{tan2026dreamworld}
& 53.7 & 27.0 & 20.9 & 35.0 & 7.8 & 3.9
& \fst{0.90} & \snd{0.53} & \trd{0.38} & \trd{0.58} \\

VPT-1.3B~\citep{zheng2026enhancing}
& 55.1 & 22.5 & 22.2 & 39.4 & \trd{8.3} & \snd{5.6}
& 0.82 & 0.44 & 0.34 & 0.50 \\

\midrule
\textbf{\method (Ours)-1.3B}
& \snd{65.7} & \fst{31.3} & \fst{27.0} & \trd{51.9} & \snd{11.0} & \fst{6.9}
& \fst{0.90} & \fst{0.55} & \snd{0.44} & \fst{0.60} \\

\bottomrule
\end{tabular}%
}
\end{table*}

\begin{table*}[!htbp]
    \centering
    \caption{VBench~\citep{huang2024vbench} comparison on selected metrics (0--100, higher is better).}
    \label{tab:vbench_compact}
    \footnotesize
    \setlength{\tabcolsep}{2.25pt}
    \renewcommand{\arraystretch}{1.15}
    \begin{tabular*}{\textwidth}{@{\extracolsep{\fill}}l*{10}{c}}
        \toprule
        \textbf{Method} & \textbf{SC} & \textbf{TF} & \textbf{MS} & \textbf{AQ} & \textbf{IQ} & \textbf{OC} & \textbf{SR} & \textbf{TS} & \textbf{QS} & \textbf{Total} \\
        \midrule
        VideoREPA-5B~\citep{zhang2025videorepa} & \trd{95.47} & 98.23 & \snd{98.54} & 57.19 & 59.80 & 67.56 & 51.85 & 22.71 & 79.96 & 76.67 \\
        PhysRAG-5B~\citep{cheng2026physrag} & \snd{96.31} & 98.71 & 96.93 & 56.80 & \trd{65.56} & \trd{80.50} & \snd{77.33} & 23.18 & \trd{83.07} & \trd{80.08} \\
        Wan 2.1-1.3B~\citep{wan2025wan} & 91.83 & \snd{99.17} & 96.51 & 57.24 & 59.14 & 76.09 & 62.37 & 22.23 & 79.81 & 76.93 \\
        VideoJAM-1.3B~\citep{chefer2025videojam} & 91.51 & \trd{99.13} & 96.05 & \trd{58.41} & 62.68 & 79.22 & 66.17 & \snd{24.03} & 81.18 & 78.76 \\
        DreamWorld-1.3B~\citep{tan2026dreamworld} & 93.62 & 98.81 & \trd{98.07} & \snd{59.59} & \snd{66.76} & \snd{81.32} & 70.47 & \trd{23.82} & \snd{83.49} & \snd{80.97} \\
        VPT-1.3B~\citep{zheng2026enhancing} & 92.65 & 98.64 & 97.85 & 58.09 & 63.68 & 73.43 & \trd{73.35} & 22.52 & 82.26 & 78.79 \\
        \midrule
        \textbf{\method (Ours)-1.3B} & \fst{96.68} & \fst{99.81} & \fst{98.85} & \fst{61.65} & \fst{68.66} & \fst{82.81} & \fst{77.88} & \fst{24.05} & \fst{84.69} & \fst{81.86} \\
        \bottomrule
    \end{tabular*}
\end{table*}

% Source: qualititive_cases/{vbench,finephyeval,videophy2}/, generated by
% paper/tools/qualitative_comparison/make_t2v_qualitative_figure.py.
\begin{figure}[!htbp]
    \centering
    \includegraphics[width=\textwidth]{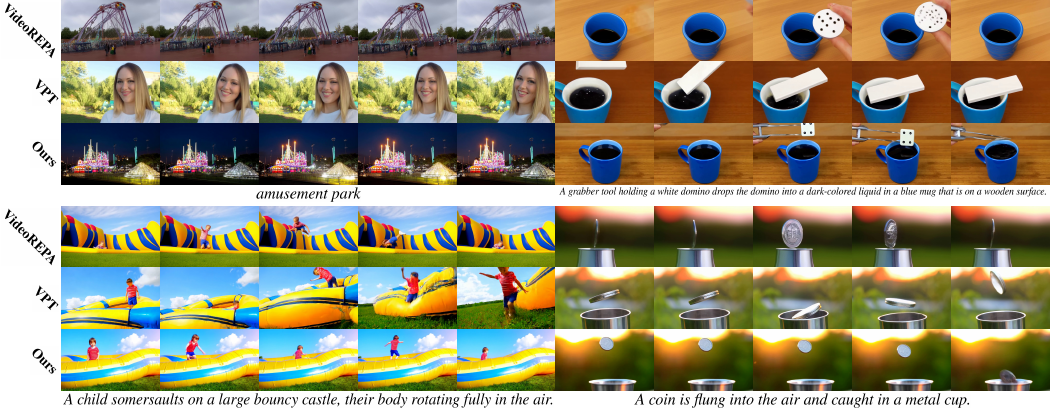}
    \caption{
        \textbf{Qualitative comparison on text-to-video generation.}
        Each scenario compares VideoREPA-5B, VPT-1.3B and \method (Ours) under the same text prompt,
        with frames sampled at $0\%$, $25\%$, $50\%$, $75\%$ and $100\%$ of each generated video.
        The prompts are drawn from VBench (amusement park), FinePhyEval (a domino dropped into a mug)
        and VideoPhy2 (a somersault on a bouncy castle, a coin caught in a metal cup).
    }
    \label{fig:t2v_qualitative_comparison}
\end{figure}

\heading{Results}
\Cref{tab:physics_generation} reports the VideoPhy2 and FinePhyEval results for physics-consistent text-to-video generation.
\method achieves the highest VideoPhy2 Joint scores on both ALL and HARD, increasing them from $24.8$ to $27.0$ and from $5.6$ to $6.9$, respectively, over Wan2.1-T2V-1.3B.
It also achieves the highest FinePhyEval Overall score of $0.60$, with improvements in object, action and physics scores over the same base model.
On VBench (\cref{tab:vbench_compact}), \method leads the compared methods on Quality Score and Overall Score, reaching $84.69$ and $81.86$.
\Cref{fig:t2v_qualitative_comparison} shows the same advantage in the generated frames, where \method carries each prompted interaction to its physical outcome, dropping the domino into the mug and landing the coin in the metal cup, while the baselines lose the object or drift away from the prompted action.
These results show that the adapted model improves both general video quality and the joint satisfaction of semantic and physical requirements.
\Cref{sec:appendix_t2v_results} analyzes the full metric breakdown.

\subsection{Image-to-Video Generation}
\label{sec:image_to_video_generation}

\heading{Settings}
We use Wan2.2-TI2V-5B~\citep{wan2025wan} as the base model, with the same training data, objectives, optimization and five-epoch schedule as in \Cref{sec:text_to_video_generation}.
We enable first-frame image conditioning for $40\%$ of training examples and use text-only conditioning for the remainder.
The refiner has 13 transformer blocks with width 2,048 and 32 attention heads, reading hidden states from the same base blocks as in text-to-video generation.
We evaluate on the 198 image--text pairs of Physics-IQ Verified~\citep{radsch2026physics, motamed2026generative}, generating 81-frame videos at $704\times1248$ resolution with the same sampling configuration as in text-to-video generation.
\Cref{sec:appendix_i2v_implementation} details the image conditioning, training and evaluation.

\begin{table}[!htbp]
    \centering
    \caption{Physics-IQ Verified~\citep{radsch2026physics} results on the image-to-video track. All evaluations are without LLM usage. All scores are percentages, and higher is better for every metric.}
    \label{tab:physics_iq_verified}
    \footnotesize
    \setlength{\tabcolsep}{3pt}
    \resizebox{\linewidth}{!}{%
    \begin{tabular}{@{}lccccc@{}}
        \toprule
        \textbf{Method} & \textbf{Spatial} & \textbf{Spatio-temporal} & \textbf{Weighted Spatial} & \textbf{MSE Score} & \textbf{Verified Score} \\
        \midrule
        P-Video-19B~\citep{pruna2026pvideo}
        & \trd{38.60} & 16.40 & 22.90 & 23.30 & 25.30 \\
        Wan2.2-5B~\citep{wan2025wan}
        & 35.89 & \fst{26.44} & \trd{22.96} & 25.54 & \snd{27.71} \\
        WMReward-5B~\citep{yuan2026inference}
        & 36.47 & \trd{22.38} & 19.79 & \snd{27.26} & 26.48 \\
        VLIPP-5B~\citep{yang2025vlipp}
        & \snd{40.51} & 19.09 & \snd{23.66} & \trd{26.48} & \trd{27.43} \\
        PhysRAG-5B~\citep{cheng2026physrag}
        & 36.46 & 22.07 & 20.22 & 23.41 & 25.54 \\
        \midrule
        \textbf{\method (Ours)-5B}
        & \fst{43.43} & \snd{26.37} & \fst{24.74} & \fst{29.24} & \fst{30.95} \\
        \bottomrule
    \end{tabular}%
    }
\end{table}

% Source: outputs/physics_iq_verified_ti2v/, same runs as the rows of
% Table~\ref{tab:physics_iq_verified} for WMReward-5B, PhysRAG-5B and \method.
\begin{figure}[!htbp]
    \centering
    \includegraphics[width=\textwidth]{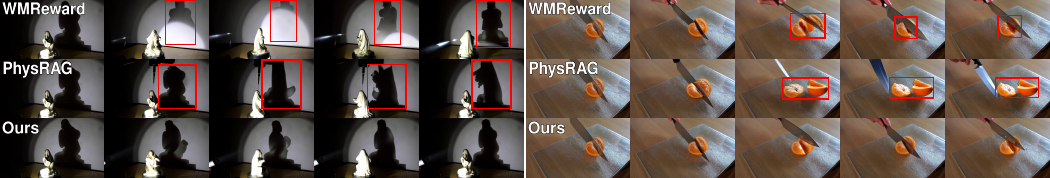}
    \caption{
        \textbf{Qualitative comparison on the Physics-IQ Verified image-to-video track.}
        Each scenario compares WMReward-5B, PhysRAG-5B and \method (Ours), which share the Wan2.2-TI2V-5B base model and receive the same settings.
        Temporal order is from left to right.
        The left scenario shows a statuette's shadow under a moving light,
        and the right scenario shows a knife slicing a mandarin.
    }
    \label{fig:i2v_qualitative_comparison}
\end{figure}

\heading{Results}
\Cref{tab:physics_iq_verified} reports the Physics-IQ Verified results.
\method achieves the highest Verified Score of $30.95$, and leads on Spatial, Weighted Spatial and MSE Score.
These results indicate closer agreement with the ground-truth videos in motion coverage and frame appearance.
\Cref{fig:i2v_qualitative_comparison} illustrates shadow and object consistency, with the statuette's shadow tracking the moving light and the blade separating the mandarin into halves.
\Cref{sec:appendix_i2v_results} provides the metric breakdown and qualitative analysis.

\subsection{Ablation Study}
\label{sec:ablation}

We conduct a number of ablations to nanalyze the contribution of each component.
All variants are built on WAN-2.1-T2V-1.3B~\citep{wan2025wan} and follow the same training protocol.
Specifically, each model is trained for five epochs on the same 1,000 samples from WISA-80K~\citep{wang2025wisa}. We evaluate on a disjoint set of 50 WISA-80K samples selected to provide approximately uniform coverage over the physical categories.
We report visual quality and consistency metrics from VBench, together with its aggregate Quality Score.
Results are reported in Table~\ref{tab:ablation}.

% Values are percentages rounded to three significant digits.
% Ranking colors use the original scores in mistakeflow_ablation_full.tex.
% Keep the caption and every ablation row together across page-layout changes.
\begin{table}[!htbp]
    \centering
    \caption{Component ablations (\%, higher is better). Colors rank the four-decimal scores in Table~\ref{tab:ablation_full}, which also expands the metric names.}
    \label{tab:ablation}
    \footnotesize
    \setlength{\tabcolsep}{2pt}
    \begin{tabular*}{\linewidth}{@{\extracolsep{\fill}}lcccccc@{}}
        \toprule
        \textbf{Method} & \textbf{AQ} & \textbf{BC} & \textbf{IQ} & \textbf{OC} & \textbf{SC} & \textbf{QS} \\
        \midrule
        w/o time-step
        & \trd{51.5} & 92.4 & \snd{63.7} & 58.5 & 91.2 & 79.7 \\

        w/o self-attention
        & 48.9 & 92.1 & 60.4 & 58.1 & 90.0 & 78.5 \\

        w/o cross-attention
        & 48.5 & 91.4 & 60.7 & 58.2 & 90.2 & 78.5 \\

        w/o hidden-state MLP
        & 49.3 & 92.4 & 63.6 & 58.2 & 90.9 & 79.3 \\

        w/o refiner
        & 51.3 & \trd{92.5} & 62.0 & \trd{59.1} & \trd{91.3} & 78.6 \\

        $\bm{v}$ input
        & 50.4 & 92.4 & \trd{63.7} & \snd{60.2} & \snd{91.7} & \snd{79.8} \\

        w/o condition
        & \snd{51.5} & \snd{92.7} & 61.2 & 58.1 & 91.0 & \trd{79.8} \\

        \textbf{Ours}
        & \fst{52.6}
        & \fst{93.6}
        & \fst{64.3}
        & \fst{61.8}
        & \fst{92.6}
        & \fst{80.3} \\
        \bottomrule
    \end{tabular*}
\end{table}

The full model achieves the highest Quality Score and leads on all visual quality and consistency metrics.
Removing cross-attention causes the largest score drop, closely followed by self-attention, highlighting the value of context and spatiotemporal interactions in the base model's prediction.
Removing either the refiner or hidden-state MLP reduces Quality Score and Overall Consistency, supporting refinement and fusion of hidden states across depths.
The remaining ablations produce smaller but consistent score decreases.
Removing noise-level conditioning or replacing the clean prediction $\hat{\mathbf{x}}_0$ with the velocity prediction $\bm{v}_{\bm{\theta}}$ lowers aggregate performance, supporting the value of noise-level information and explicit clean-sample estimates.
Removing text conditioning lowers Quality Score and notably weakens Overall Consistency, suggesting its role in semantic alignment.

\section{Conclusion}
\label{sec:conclusion}

We introduced \method, a self-correcting flow-matching framework that learns velocity residuals from the base model's clean prediction, conditioning and hidden states.
A lightweight refiner corrects the velocity within the sampling loop.
Our theory identifies conditions under which feature reuse improves residual prediction, refiner training outperforms continued base-model training at matched compute, and the deployed correction lowers own-law velocity
error.
We demonstrate that \method achieves state-of-the-art performance across multiple benchmarks including image generation, text-to-video generation, and image-to-video generation.
We hope \method will be a useful tool for the community to advance research in flow-matching-based generative modeling.

\clearpage

%\section*{AI Use Statement}
%\label{sec:ai_assistance}

%We discussed ideas with large language models while developing the theoretical derivations. The derivations themselves were completed and verified by the authors.Beyond this, we did not use generative AI tools for any other tasks for which the ICLR policy requires disclosure, including data generation, hypothesis formulation, experimental design, method implementation, translation, data processing, and result interpretation.We used large language models mainly for tasks with recommended disclosure: editing the manuscript for readability, searching for relevant prior work, and drafting parts of the paper. The authors revised every AI-assisted passage against our own method, code, and results, and verified every AI-suggested reference against the original publication. We take full responsibility for the final content of this work, including all AI-assisted text.

\bibliography{iclr2027_conference,chuanxia_general,chuanxia_specific}
\bibliographystyle{iclr2027_conference}

\newpage
\appendix

\startcontents[appendix]
\section*{Appendix Contents}
\printcontents[appendix]{}{1}{\setcounter{tocdepth}{2}}
\newpage

% ---------------------------------------------------------------------------
% Perspective 1: restricted-class residual-risk advantage from frozen base
% features. Transcribed from mistake_flow_theory/results/perspective_1/idea_1.
% Main-text summary: Section~\ref{sec:theory_p1}.
% ---------------------------------------------------------------------------

\section{Frozen-Feature Analysis}
\label{p1:sec:setup}

This appendix establishes an exact population-risk separation for a linear
readout of a frozen nonlinear feature relative to predictors that are affine
in the raw inputs.

\subsection{Setting and assumptions}

We work in one active coordinate of the sample space. Samples, velocities,
residuals and hidden readouts in this construction are scalar and use
ordinary symbols. The conditioning vector is bold. Let
$\bm{c}=(c_1,c_2)$, write $p:=\psi(\bm{c}):=c_1c_2$ and fix $\delta>0$.
For $\sigma\in[0,1]$, define
\[
s_\sigma^2:=(1-\sigma)^2+\sigma^2,
\qquad
\ell_\sigma:=\frac{2\sigma-1}{s_\sigma^2},
\qquad
\tau_\sigma:=\frac{\sigma}{s_\sigma^2}.
\]

\begin{assumption}[Conditional-Gaussian XOR witness]
\label{p1:assump:conditional-gaussian-law}
The variables $c_1,c_2$ are independent and uniform on $\{-1,1\}$.
The variables $\xi,\varepsilon\sim\mathcal N(0,1)$ are independent of
each other and of $(\bm{c},\sigma)$, and
\[
x_0:=\delta p+\xi,
\qquad
x_\sigma:=(1-\sigma)x_0+\sigma\varepsilon,
\qquad
u:=\varepsilon-x_0.
\]
\end{assumption}

\begin{assumption}[Independent Noise Level]
\label{p1:assump:full-noise-law}
The variable $\sigma\sim q$ is independent of $(\bm{c},\xi,\varepsilon)$,
where $q$ is a Borel probability measure with topological support exactly
$[0,1]$.
\end{assumption}

Let $(\Omega,\mathscr F,\mathbb P)$ be the underlying probability space
carrying these variables. Write $W:=(x_\sigma,\sigma,\bm{c})$ and let
$P_W:=\mathbb P\circ W^{-1}$ be its law. The frozen base velocity, its
one-step clean prediction from \Eqref{eq:zhat} and the realized residual
label are
\[
v_{\bm\theta}(W):=\ell_\sigma x_\sigma,
\qquad
\widehat x_0(W):=x_\sigma-\sigma v_{\bm\theta}(W),
\qquad
r:=u-v_{\bm\theta}(W).
\]

\begin{assumption}[Frozen base and product feature]
\label{p1:assump:frozen-base-feature}
The base is frozen with velocity
$v_{\bm\theta}(W)=\ell_\sigma x_\sigma$. Its captured representation is
the deterministic map
$\Phi(W):=(\widehat x_0(W),h_L(W),\sigma,\bm{c})$ with hidden coordinate
$h_L(W)=c_1c_2$. The base velocity readout omits this coordinate, while
the scalar AdaLN-modulated class $\mathcal G$ may read it.
\end{assumption}

Define
\[
u^\star(W):=\mathbb E[u\mid W],
\qquad
\bar r(W):=\mathbb E[r\mid W]=u^\star(W)-v_{\bm\theta}(W).
\]
All expectations below are with respect to $\mathbb P$.
For every square-integrable $W$-measurable predictor $f$,
\[
\mathcal R(f):=\mathbb E[(f(W)-r)^2],
\qquad
\mathcal R_{\mathrm{irr}}
:=\mathbb E[(r-\bar r(W))^2].
\]
In the notation of Section~\ref{sec:theory},
$P=P_W$ and
$\widetilde P=\mathbb P\circ(W,r)^{-1}$.
In this scalar specialization, the frozen checkpoint
$\bm{v}_0$ corresponds to $v_{\bm\theta}$ and
$\mathcal R_P=\mathcal R$.\\
The one-parameter hidden-feature class, viewed as induced predictors on
$W$, is
\[
\mathcal G:=\{g_b:b\in\mathbb R\},
\qquad
g_b(W):=-b\tau_\sigma h_L(W).
\]
Each $g_b$ uses only $\Phi(W)$.
As an oracle enlargement of the intended raw linear readouts, let
\[
\mathcal A:=\left\{
a_0(\sigma)+a_x(\sigma)x_\sigma+a_1(\sigma)c_1+a_2(\sigma)c_2:
\begin{array}{l}
a_0,a_x,a_1,a_2\text{ are Borel measurable},\\[-1mm]
\text{and the induced predictor lies in }L^2(P_W)
\end{array}
\right\}.
\]
The intended one-scalar raw-input comparator families at the same
trainable-readout budget are subclasses of $\mathcal A$.
For nested-class comparisons we use the convention
$\inf\varnothing=+\infty$.

For scalars $a,b$, define the zero-stabilized cosine loss by
\[
\operatorname{cos}_0(a,b):=
\begin{cases}
ab/(|a||b|),&ab\ne0,\\
0,&ab=0,
\end{cases}
\qquad
L_{\cos}(a,b):=1-\operatorname{cos}_0(a,b).
\]
For $\lambda=0.1$, set
\[
\mathcal J_\lambda(b)
:=\mathbb E\!\left[(g_b(W)-r)^2
+\lambda L_{\cos}(g_b(W),r)\right].
\]

\begin{assumption}[Cosine-regularized population fit]
\label{p1:assump:population-optimization}
For a stated tolerance $\epsilon_{\mathrm{opt}}\ge0$, the trained
coefficient $\widehat b\in\mathbb R$ satisfies
\[
\mathcal J_\lambda(\widehat b)
\le \inf_{b\in\mathbb R}\mathcal J_\lambda(b)
+\epsilon_{\mathrm{opt}}.
\]
\end{assumption}

Define the structural constants
\[
S:=\mathbb E_{\sigma\sim q}[\tau_\sigma^2],
\qquad
\Delta_{\mathrm{aff}}
:=\delta^2\mathbb E_{\sigma\sim q}\!\left[
\tau_\sigma^2
\frac{s_\sigma^2}
{s_\sigma^2+(1-\sigma)^2\delta^2}
\right].
\]

\subsection{Main result}
\label{p1:sec:main-theorem}

\begin{theorem}[Detailed frozen-feature advantage]
\label{p1:thm:main}
Under Assumptions~\ref{p1:assump:conditional-gaussian-law}--
\ref{p1:assump:population-optimization}, the following statements hold.
\begin{enumerate}
\item For every $\sigma<1$,
\[
x_\sigma=\frac{s_\sigma^2}{1-\sigma}\widehat x_0,
\]
whereas $\widehat x_0=0$ at $\sigma=1$. On the full interval,
\[
\widehat x_0-x_0=\sigma r.
\]
Thus $r=(\widehat x_0-x_0)/\sigma$ for $\sigma>0$, with the direct
continuous endpoint values $r=\varepsilon$ at $\sigma=0$ and
$r=-x_0$ at $\sigma=1$. Moreover,
\[
u^\star(W)=\ell_\sigma x_\sigma-\tau_\sigma\delta p,
\qquad
\bar r(W)=-\tau_\sigma\delta p.
\]

\item The irreducible risk and structural constants satisfy
\[
\mathcal R_{\mathrm{irr}}
=\mathbb E_{\sigma\sim q}[s_\sigma^{-2}]
\in[1,2],
\qquad
0<S<\infty,
\qquad
0<\Delta_{\mathrm{aff}}<\infty.
\]
For every $b\in\mathbb R$,
\[
\mathcal R(g_b)
=\mathcal R_{\mathrm{irr}}+S(b-\delta)^2,
\qquad
\inf_{a\in\mathcal A}\mathcal R(a)
=\mathcal R_{\mathrm{irr}}+\Delta_{\mathrm{aff}}.
\]
In particular, $g_\delta=\bar r$ and
$\mathcal R(g_\delta)=\mathcal R_{\mathrm{irr}}$.

\item The fitted coefficient satisfies
\[
S(\widehat b-\delta)^2
\le2\lambda+\epsilon_{\mathrm{opt}}.
\]
If the static scalar condition
\[
2\lambda+\epsilon_{\mathrm{opt}}<\Delta_{\mathrm{aff}}
\]
holds, then
\[
\mathcal R(g_{\widehat b})
<\inf_{a\in\mathcal A}\mathcal R(a).
\]
For every intended equal-budget raw comparator class
$\mathcal B\subseteq\mathcal A$, the same condition also gives
\[
\mathcal R(g_{\widehat b})
<\inf_{f\in\mathcal B}\mathcal R(f).
\]

\item For every finite $d\ge1$, the embedding
$\iota_d(z)=(z,0,\ldots,0)$ preserves the feature-risk curve, the
raw-affine infimum, the cosine-regularized objective and the strict
trained-risk margin.
\end{enumerate}
These identities give the hidden-feature refiner's population
$L^2(\mathbb P)$ risk and its margin relative to the specified raw-affine
classes. The parameters $\delta,q,\lambda$ and $\epsilon_{\mathrm{opt}}$
determine the feature curvature $S$, the raw-affine gap
$\Delta_{\mathrm{aff}}$ and the additive excess-risk bound.
\end{theorem}

The product feature represents the conditional Bayes residual exactly.
The noise law $q$ weights the risk across noise levels, with full support
ensuring positive feature curvature $S$ and raw-affine gap
$\Delta_{\mathrm{aff}}$. The cosine weight $\lambda$ and optimization
tolerance $\epsilon_{\mathrm{opt}}$ bound the fitted predictor's excess
risk. The strict scalar condition places this bound below the raw-affine
gap, certifying the trained refiner's advantage.

Theorem~\ref{p1:thm:main} is the detailed version of the main-text
Theorem~\ref{thm:frozen_feature}.

\subsection{Proof}
\label{p1:sec:proof-details}

\begin{proof}
We prove the scalar conclusions in four steps. The finite-dimensional
preservation statement is verified in the embedding paragraph below.

\paragraph{Step 1: Residual interface and conditional Bayes residual.}
Since
\[
s_\sigma^2
=2\left(\sigma-\frac12\right)^2+\frac12\ge\frac12,
\]
all denominators below are nonzero.  Moreover,
\[
1-\sigma\ell_\sigma
=\frac{s_\sigma^2-\sigma(2\sigma-1)}{s_\sigma^2}
=\frac{1-\sigma}{s_\sigma^2},
\]
which proves the displayed clean-prediction formula and its inverse for
$\sigma<1$:
\[
\widehat x_0=\frac{1-\sigma}{s_\sigma^2}x_\sigma,
\qquad
x_\sigma=\frac{s_\sigma^2}{1-\sigma}\widehat x_0.
\]
Direct substitution gives
$\widehat x_0=x_0$ at $\sigma=0$ and $\widehat x_0=0$ at $\sigma=1$.
The interpolation identity also gives, without division,
\[
\widehat x_0-x_0
=x_\sigma-\sigma v_{\bm\theta}-x_0
=\sigma(\varepsilon-x_0-v_{\bm\theta})
=\sigma r.
\]
Consequently, $r=(\widehat x_0-x_0)/\sigma$ for every $\sigma>0$.
The direct residual definition gives $r=\varepsilon$ at zero and
$r=-x_0$ at one, which agree with the continuous endpoint extension.

To compute the conditional mean, center the observation as
\[
Y_\sigma:=x_\sigma-(1-\sigma)\delta p
=(1-\sigma)\xi+\sigma\varepsilon
\]
and introduce its orthogonal complement
\[
N_\sigma:=(1-\sigma)\varepsilon-\sigma\xi.
\]
Since $p$ is measurable with respect to $\bm{c}$ and
\[
Y_\sigma=x_\sigma-(1-\sigma)\delta p,
\qquad
x_\sigma=Y_\sigma+(1-\sigma)\delta p,
\]
the tuples $W=(x_\sigma,\sigma,\bm{c})$ and
$(Y_\sigma,\sigma,\bm{c})$ generate the same sigma-field.
Conditional on $(\sigma,\bm{c})$, the vector
$(Y_\sigma/s_\sigma,N_\sigma/s_\sigma)$ is standard Gaussian with
independent coordinates, because it is obtained from
$(\xi,\varepsilon)$ by the orthogonal transformation
\[
\begin{pmatrix}
Y_\sigma/s_\sigma\\
N_\sigma/s_\sigma
\end{pmatrix}
=\frac1{s_\sigma}
\begin{pmatrix}
1-\sigma&\sigma\\
-\sigma&1-\sigma
\end{pmatrix}
\begin{pmatrix}
\xi\\
\varepsilon
\end{pmatrix}.
\]
Because $(\xi,\varepsilon)$ is independent of $(\sigma,\bm{c})$, the
fixed-$t$ Gaussian regression above defines a conditional kernel for every
value of $(\sigma,\bm{c})$. The displayed coefficient functions are Borel
and integrable because $s_\sigma^2\ge1/2$. Since $W$ and
$(Y_\sigma,\sigma,\bm{c})$ generate the same sigma-field, these kernels give
the Borel conditional versions
\[
\mathbb E[\xi\mid W]
=\frac{1-\sigma}{s_\sigma^2}Y_\sigma,
\qquad
\mathbb E[\varepsilon\mid W]
=\frac{\sigma}{s_\sigma^2}Y_\sigma.
\]
Using $u=\varepsilon-\xi-\delta p$ and
$1+(1-\sigma)\ell_\sigma=\tau_\sigma$ gives
\[
u^\star(W)=\ell_\sigma x_\sigma-\tau_\sigma\delta p,
\qquad
\bar r(W)=u^\star(W)-v_{\bm\theta}(W)=-\tau_\sigma\delta p.
\]
Moreover, conditional on $(\sigma,\bm{c})$, the normalized coordinate
$N_\sigma/s_\sigma$ is independent of $Y_\sigma/s_\sigma$. Conditioning
additionally on $W$, equivalently on $(Y_\sigma,\sigma,\bm{c})$, therefore
leaves its Gaussian law unchanged. The same calculation yields
\[
r-\bar r(W)=\frac{N_\sigma}{s_\sigma^2},
\qquad
(r-\bar r(W))\mid W
\sim\mathcal N\!\left(0,s_\sigma^{-2}\right).
\]
Therefore
\begin{equation}
\mathcal R_{\mathrm{irr}}
=\mathbb E_{\sigma\sim q}[s_\sigma^{-2}].
\label{eq:p1:exact-irreducible-risk}
\end{equation}
Since $s_\sigma^2\in[1/2,1]$, this quantity belongs to $[1,2]$.

\paragraph{Step 2: Same-target decomposition and feature oracle.}
The following decomposition holds on any probability space with
$r\in L^2(\mathbb P)$ and $f(W)\in L^2(\mathbb P)$. Set
$\bar r(W):=\mathbb E[r\mid W]$ and $\eta:=r-\bar r(W)$.
Conditional Jensen's inequality gives
\[
\mathbb E[\bar r(W)^2]\le\mathbb E[r^2].
\]
Thus $\eta$ and $f(W)-\bar r(W)$ belong to $L^2$, and their product is
integrable by Cauchy--Schwarz. The defining orthogonality property of
conditional expectation gives
\[
\mathbb E[(f(W)-\bar r(W))\eta]
=\mathbb E[(f(W)-\bar r(W))\mathbb E[\eta\mid W]]=0.
\]
Expanding the square therefore yields
\[
\mathcal R(f)
=\mathcal R_{\mathrm{irr}}
+\mathbb E[(f(W)-\bar r(W))^2].
\]
This also proves the scalar case of \eqref{eq:risk_decomposition}. For a
finite-dimensional $\mathcal X$, choose an orthonormal basis and apply the
same identity to each coordinate of $\bm{r}$ and $\bm{f}(W)$. Conditional
expectation acts coordinatewise, and summing the coordinate identities
gives the squared-norm identity in \eqref{eq:risk_decomposition}.
Returning to the present Gaussian witness, since $p^2=1$ and
$g_b-\bar r=(\delta-b)\tau_\sigma p$, we obtain
\[
\mathcal R(g_b)
=\mathcal R_{\mathrm{irr}}+S(b-\delta)^2.
\]
In particular, $g_\delta=\bar r$ pointwise for the chosen Borel version.
The bound $0\le\tau_\sigma\le2$ gives $S\le4$.  Since
$\tau_\sigma^2$ is continuous and strictly positive on $(0,1]$, the full
support of $q$ gives $S>0$. More explicitly, for
$K=[1/2,3/4]$, one has $q(K)>0$, $\tau_\sigma^2\ge1/4$ on $K$ and hence
\[
S\ge\frac14q(K)>0.
\]

\paragraph{Step 3: Exact raw-affine projection.}
For each $t\in[0,1]$, let $\mathbb E_t$ denote expectation over
$(\bm{c},\xi,\varepsilon)$ with the noise level fixed at $\sigma=t$.
By independence of $\sigma$, every nonnegative or integrable Borel
function $H$ satisfies
\[
\mathbb E[H]
=\int_{[0,1]}\mathbb E_t[H]\,q(dt).
\]
Fix $t\in[0,1]$, put $m_t:=1-t$ and define
\[
V_t:=s_t^2+m_t^2\delta^2,
\qquad
d_t:=\delta^2\tau_t^2\frac{s_t^2}{V_t}.
\]
Under the fixed-noise law, the Gram matrix of
$(1,c_1,c_2,x_t)$ is diagonal with diagonal $(1,1,1,V_t)$, because
\[
\mathbb E_t[x_t]=0,
\quad
\mathbb E_t[x_t^2]=V_t,
\quad
\mathbb E_t[x_tc_1]=\mathbb E_t[x_tc_2]
=\mathbb E_t[c_1c_2]=0.
\]
The target-design moments are
\[
\mathbb E_t[\bar r]
=\mathbb E_t[\bar r c_1]
=\mathbb E_t[\bar r c_2]=0,
\qquad
\mathbb E_t[\bar r x_t]=-m_t\tau_t\delta^2,
\qquad
\mathbb E_t[\bar r^2]=\tau_t^2\delta^2.
\]
Consequently, for
$A_\alpha:=\alpha_0+\alpha_xx_t+\alpha_1c_1+\alpha_2c_2$,
completion of the square gives
\[
\mathbb E_t[(A_\alpha-\bar r)^2]
=\alpha_0^2+\alpha_1^2+\alpha_2^2
+V_t\left(\alpha_x+\frac{m_t\tau_t\delta^2}{V_t}\right)^2
+d_t.
\]
This identity also holds at $t=0,1$, where $d_0=0$ and $d_1=\delta^2$.
The pointwise minimizer is therefore
\[
a_0^\star=a_1^\star=a_2^\star=0,
\qquad
a_x^\star(t)=-\frac{m_t\tau_t\delta^2}{V_t}.
\]
Equivalently, let
\[
a_t^\star:=-\frac{m_t\tau_t\delta^2}{V_t}x_t,
\qquad
e_t^\star:=\bar r-a_t^\star.
\]
The displayed moment identities give
\[
\mathbb E_t[e_t^\star]
=\mathbb E_t[e_t^\star c_1]
=\mathbb E_t[e_t^\star c_2]
=\mathbb E_t[e_t^\star x_t]
=0.
\]
Thus $a_t^\star$ is the conditional $L^2$ projection of $\bar r$ onto
the raw-affine design, and
\[
\mathbb E_t[(e_t^\star)^2]=d_t.
\]
Because $V_t\ge s_t^2\ge1/2$, these coefficients are continuous and
bounded on $[0,1]$.  Moreover,
\[
\sup_{t\in[0,1]}\mathbb E_t[x_t^2]
=\sup_{t\in[0,1]}V_t
\le1+\delta^2.
\]
Thus the induced predictor
$a^\star(W)=a_x^\star(\sigma)x_\sigma$ is square-integrable and belongs
to $\mathcal A$.
For every $a\in\mathcal A$ with coefficient functions
$a_0,a_x,a_1,a_2$, the same diagonal Gram calculation applied to
$a(W)^2$ gives
\[
\int_{[0,1]}
\left(a_0(t)^2+a_1(t)^2+a_2(t)^2+V_ta_x(t)^2\right)q(dt)
=\mathbb E[a(W)^2]
<\infty.
\]
Hence the conditional square-completion identity may be integrated with
respect to $q(dt)$, giving
\[
\begin{aligned}
\mathbb E[(a(W)-\bar r(W))^2]
=\int_{[0,1]}\Bigg[&a_0(t)^2+a_1(t)^2+a_2(t)^2\\
&+V_t\left(
a_x(t)+\frac{(1-t)\tau_t\delta^2}{V_t}
\right)^2+d_t\Bigg]q(dt).
\end{aligned}
\]
Consequently,
\[
\mathbb E[(a(W)-\bar r(W))^2]
\ge\int_{[0,1]}d_t\,q(dt).
\]
The coefficient choice $a^\star$ makes every squared term vanish for all
$t\in[0,1]$, including both endpoints, so equality is attained. Therefore,
\[
\inf_{a\in\mathcal A}\mathbb E[(a(W)-\bar r(W))^2]
=\mathbb E_{t\sim q}[d_t]
=\Delta_{\mathrm{aff}}.
\]
Applying the same-target decomposition gives
\[
\inf_{a\in\mathcal A}\mathcal R(a)
=\mathcal R_{\mathrm{irr}}+\Delta_{\mathrm{aff}}.
\]
The bounds $0\le d_t\le\delta^2\tau_t^2\le4\delta^2$ prove finiteness.
Moreover, $d_t$ is continuous and strictly positive for every $t>0$.
Full support of $q$ therefore gives $\Delta_{\mathrm{aff}}>0$.
Explicitly, on $K=[1/2,3/4]$,
\[
d_t\ge\frac{\delta^2}{2(4+\delta^2)},
\qquad
\Delta_{\mathrm{aff}}
\ge\frac{\delta^2}{2(4+\delta^2)}q(K)>0.
\]
Equivalently, define the squared conditional correlation
\[
\rho_t^2:=\operatorname{Corr}_t(p,x_t)^2
=\frac{m_t^2\delta^2}{V_t}.
\]
The conditional gap satisfies
\[
d_t=\operatorname{Var}_t(\bar r)(1-\rho_t^2),
\]
so it is exactly the product-feature energy unexplained by the raw-affine
projection. The raw-affine projection explains the fraction $\rho_t^2$ of
the conditional Bayes-residual energy, while the hidden product feature
represents the target exactly.

\paragraph{Step 4: Transfer the fitted objective to strict risk separation.}
The zero-stabilized cosine loss is measurable and lies in $[0,2]$.
Therefore, for every $b$,
\[
\mathcal R(g_b)
\le\mathcal J_\lambda(b)
\le\mathcal R(g_b)+2\lambda.
\]
Using Assumption~\ref{p1:assump:population-optimization} and comparing
with the oracle coefficient $b=\delta$ gives
\[
\begin{aligned}
\mathcal R(g_{\widehat b})
&\le\mathcal J_\lambda(\widehat b)\\
&\le\inf_b\mathcal J_\lambda(b)+\epsilon_{\mathrm{opt}}\\
&\le\mathcal J_\lambda(\delta)+\epsilon_{\mathrm{opt}}\\
&\le\mathcal R_{\mathrm{irr}}+2\lambda+\epsilon_{\mathrm{opt}}.
\end{aligned}
\]
Combining this bound with the exact feature-risk curve yields
\begin{equation}
S(\widehat b-\delta)^2
\le2\lambda+\epsilon_{\mathrm{opt}},
\qquad
|\widehat b-\delta|
\le\sqrt{\frac{2\lambda+\epsilon_{\mathrm{opt}}}{S}}.
\label{eq:p1:trained-coefficient-radius}
\end{equation}
For every $b$ the exact class margin is
\begin{equation}
\inf_{a\in\mathcal A}\mathcal R(a)-\mathcal R(g_b)
=\Delta_{\mathrm{aff}}-S(b-\delta)^2.
\label{eq:p1:exact-feature-affine-margin}
\end{equation}
Since $S>0$ and $\Delta_{\mathrm{aff}}>0$, this identity gives the exact
coefficient interval for strict advantage:
\begin{equation}
\mathcal R(g_b)<\inf_{a\in\mathcal A}\mathcal R(a)
\quad\Longleftrightarrow\quad
|b-\delta|<\sqrt{\frac{\Delta_{\mathrm{aff}}}{S}}.
\label{eq:p1:positive-coefficient-region}
\end{equation}
The condition $2\lambda+\epsilon_{\mathrm{opt}}<\Delta_{\mathrm{aff}}$
places the fitted coefficient in this interval by
\eqref{eq:p1:trained-coefficient-radius}.
Evaluating at $\widehat b$ proves
\begin{equation}
\inf_{a\in\mathcal A}\mathcal R(a)-\mathcal R(g_{\widehat b})
\ge\Delta_{\mathrm{aff}}-(2\lambda+\epsilon_{\mathrm{opt}}).
\label{eq:p1:trained-risk-margin}
\end{equation}
Under the stated scalar condition this margin is strictly positive.  If
$\mathcal B\subseteq\mathcal A$, then
$\inf_{a\in\mathcal A}\mathcal R(a)
\le\inf_{f\in\mathcal B}\mathcal R(f)$, which proves the final claim.
The fourth conclusion follows from the finite-dimensional embedding
argument below.
\end{proof}

\subsection{Useful specializations}

\paragraph{Uniform noise law.}
If $q$ is uniform on $[0,1]$, use
\[
z:=2\sigma-1,
\qquad
s_\sigma^2=\frac{1+z^2}{2},
\qquad
d\sigma=\frac12dz.
\]
Then
\[
\mathcal R_{\mathrm{irr}}
=\int_{-1}^1\frac{1}{1+z^2}\,dz
=\frac\pi2,
\]
while the odd part of $(z+1)^2$ integrates to zero and
\[
S
=\int_{-1}^1\frac{(z+1)^2}{2(1+z^2)^2}\,dz
=\frac12\int_{-1}^1\frac{1}{1+z^2}\,dz
=\frac\pi4.
\]
Thus
\begin{equation}
\mathcal R_{\mathrm{irr}}=\frac{\pi}{2},
\qquad
S=\frac{\pi}{4},
\qquad
\mathcal R(g_b)=\frac{\pi}{2}+\frac{\pi}{4}(b-\delta)^2.
\label{eq:p1:uniform-noise-closed-forms}
\end{equation}
Consequently,
\[
|\widehat b-\delta|
\le\frac{2}{\sqrt\pi}
\sqrt{2\lambda+\epsilon_{\mathrm{opt}}}.
\]

\paragraph{Finite-dimensional embedding.}
\label{p1:app:step-008}
The witness uses one active sample coordinate so that every source of the
risk gap is visible. Fix a finite $d\ge1$ and define
\[
\iota_d:\mathbb R\to\mathbb R^d,
\qquad
\iota_d(z):=(z,0,\ldots,0),
\]
\[
\bm{r}^{(d)}:=\iota_d(r),
\qquad
\bar{\bm{r}}^{(d)}:=\iota_d(\bar r).
\]
For a scalar predictor $f$, write $\bm{f}^{(d)}:=\iota_d\circ f$ and set
\[
\mathcal R_d(\bm{f}^{(d)})
:=\mathbb E\!\left[
\|\bm{f}^{(d)}(W)-\bm{r}^{(d)}\|_2^2
\right].
\]
The embedded classes are
\[
\mathcal G_d:=\{\bm{g}_b^{(d)}:b\in\mathbb R\},
\qquad
\mathcal A_d:=\{\bm{a}^{(d)}:a\in\mathcal A\}.
\]
Since $\iota_d$ preserves inner products, norms and pointwise distances,
\[
\mathbb E[\bm{r}^{(d)}\mid W]=\bar{\bm{r}}^{(d)},
\qquad
\mathcal R_d(\bm{g}_b^{(d)})=\mathcal R(g_b),
\qquad
\mathcal R_d(\bm{a}^{(d)})=\mathcal R(a).
\]
The map $a\mapsto\bm{a}^{(d)}$ is a bijection from $\mathcal A$ to
$\mathcal A_d$. Therefore
\[
\inf_{\bm{a}^{(d)}\in\mathcal A_d}\mathcal R_d(\bm{a}^{(d)})
=\mathcal R_{\mathrm{irr}}+\Delta_{\mathrm{aff}},
\qquad
\bm{g}_\delta^{(d)}=\bar{\bm{r}}^{(d)}.
\]
For $\bm{y},\bm{z}\in\mathbb R^d$, define
\[
\operatorname{cos}_{0,d}(\bm{y},\bm{z}):=
\begin{cases}
\langle\bm{y},\bm{z}\rangle/(\|\bm{y}\|_2\|\bm{z}\|_2),
&\|\bm{y}\|_2\|\bm{z}\|_2>0,\\
0,&\|\bm{y}\|_2\|\bm{z}\|_2=0,
\end{cases}
\]
and $L_{\cos,d}:=1-\operatorname{cos}_{0,d}$. Define
\[
\mathcal J_{\lambda,d}(b)
:=\mathbb E\!\left[
\|\bm{g}_b^{(d)}(W)-\bm{r}^{(d)}\|_2^2
+\lambda L_{\cos,d}(\bm{g}_b^{(d)}(W),\bm{r}^{(d)})
\right].
\]
The same isometry gives
$\mathcal J_{\lambda,d}(b)=\mathcal J_\lambda(b)$ for every $b$.
Consequently, if
$2\lambda+\epsilon_{\mathrm{opt}}<\Delta_{\mathrm{aff}}$, then
\[
\mathcal R_d(\bm{g}_{\widehat b}^{(d)})
<\inf_{\bm{a}^{(d)}\in\mathcal A_d}
\mathcal R_d(\bm{a}^{(d)}),
\]
with margin at least
$\Delta_{\mathrm{aff}}-(2\lambda+\epsilon_{\mathrm{opt}})$,
independently of $d$. This is an isometric embedding of the scalar witness,
not a claim about a general multidimensional data model.

% ---------------------------------------------------------------------------
% Perspective 2: finite-budget matched-resource comparison. This section
% retains the positive chain in perspective_2/idea_3, Steps 1--7.
% ---------------------------------------------------------------------------

\section{Matched-Resource Comparison: Proof}
\label{p2:app:matched-resource-proof}

This section proves that a refiner's tangent risk-reduction advantage transfers
to its trained output when the advantage exceeds the two perturbation bounds,
as stated in Theorem~\ref{thm:p2-matched-resource}. Unsubscripted norms
of random correction maps are those of $L^2(P;\mathcal X)$. Norms inside an
expectation are sample-space norms, while parameter and operator norms have
their usual finite-dimensional meanings. Linear operators retain ordinary
uppercase symbols, while parameter vectors and vector-valued fields are bold.

\subsection{Full setting and local certificate}
\label{p2:app:full-setting}

Let $\mathcal X$ be a finite-dimensional real Hilbert space. Write
$W=(\mathbf{X}_\Sigma,\Sigma,\bm{C})$ for the current noisy sample,
noise level and conditioning vector, and let $P$ be its law. Let
$\bm{U}$ denote the realized interpolation velocity. At a fixed base
checkpoint $\bm{v}_0$, define
\begin{equation}
  \bm{r}=\bm{U}-\bm{v}_0(W),
  \qquad
  \widetilde P=\operatorname{Law}(W,\bm{r}),
  \qquad
  \bar{\bm{r}}(W)=\E_{\widetilde P}[\bm{r}\mid W],
  \qquad
  s=\|\bm{r}\|_{L^2(\widetilde P;\mathcal X)}<\infty.
  \label{eq:p2-residual-target}
\end{equation}
For $\bm{f}\in L^2(P;\mathcal X)$, set
\begin{equation}
  \mathcal R_P(\bm{f})=\E_{\widetilde P}\|\bm{r}-\bm{f}(W)\|^2.
  \label{eq:p2-risk}
\end{equation}
The refiner and continued-base correction families are
\begin{equation}
  F_{\mathrm r}(\bm{\alpha})(W)=\bm{g}_{\bm{\phi}_0+\bm{\alpha}}(\mathcal{Z}(W)),
  \qquad
  F_{\mathrm b}(\bm{\alpha})(W)
  =\bm{v}_{\bm{\theta}_0+\bm{\alpha}}(W)-\bm{v}_0(W),
  \label{eq:p2-correction-families}
\end{equation}
where $F_j(0)=0$ and $\mathcal{Z}(W)$ is computed only from the current state and frozen
base computation, as in \Secref{sec:method_inputs}.

A componentwise ledger
$\mathcal C=(C_{\mathrm{train}},C_{\mathrm{state}},C_{\mathrm{deploy}})$ is
fixed before either run. If one update of arm $j\in\{\mathrm r,\mathrm b\}$
costs $c_j>0$, its charged horizon is
\begin{equation}
  T_j=\left\lfloor\frac{C_{\mathrm{train}}}{c_j}\right\rfloor.
  \label{eq:p2-charged-horizon}
\end{equation}
Both runs must fit the declared optimizer-state and inference budgets. Their
population objectives are
\begin{equation}
  L_{\mathrm b}(\bm{\alpha})=\tfrac12\mathcal R_P(F_{\mathrm b}(\bm{\alpha})),
  \qquad
  L_{\mathrm r}(\bm{\alpha})=\tfrac12\mathcal R_P(F_{\mathrm r}(\bm{\alpha}))
  +\lambda_{\angle}\mathcal A(\bm{\alpha}),
  \quad \lambda_{\angle}=0.1,
  \label{eq:p2-objectives}
\end{equation}
where $\mathcal A$ is the population angular penalty induced by the
implemented thresholded cosine loss. For a candidate step $\eta>0$, define
\begin{equation}
  \bm{\alpha}_{j,0}^{\eta}=0,
  \qquad
  \bm{\alpha}_{j,t+1}^{\eta}
  =\bm{\alpha}_{j,t}^{\eta}-\eta\nabla L_j(\bm{\alpha}_{j,t}^{\eta}),
  \qquad
  \bm{f}_{j,\eta}=F_j(\bm{\alpha}_{j,T_j}^{\eta}).
  \label{eq:p2-actual-recursion}
\end{equation}

At the checkpoint, let
\begin{equation}
  \Phi_j=J_j(0),
  \qquad
  H_j=\Phi_j^*\Phi_j,
  \qquad
  \bm{b}_j=\Phi_j^*\bar{\bm{r}},
  \label{eq:p2-tangent-quantities}
\end{equation}
where $J_j(\bm{\alpha})=D F_j(\bm{\alpha})$. The zero-initialized tangent recursion is
\begin{equation}
  \bm{\beta}_{j,t+1}^{\eta}
  =(I-\eta H_j)\bm{\beta}_{j,t}^{\eta}+\eta \bm{b}_j,
  \qquad \bm{\beta}_{j,0}^{\eta}=0.
  \label{eq:p2-tangent-recursion}
\end{equation}

\begin{assumption}[Finite-budget local certificate]
\label{p2:assump:finite-budget-certificate}
For each arm there are finite $R_j,K_j,\Lambda_j\ge0$ such that $F_j$ is
continuously differentiable on $\|\bm{\alpha}\|\le R_j$ and
\begin{equation}
  \|J_j(\bm{\alpha})\|_{\mathrm{op}}\le K_j,
  \qquad
  \|J_j(\bm{\alpha})-J_j(\bm{\alpha}')\|_{\mathrm{op}}
  \le\Lambda_j\|\bm{\alpha}-\bm{\alpha}'\|.
  \label{eq:p2-local-jacobian}
\end{equation}
On the refiner ball, $\mathcal A$ is continuously differentiable and
$\|\nabla\mathcal A(\bm{\alpha})\|\le G_{\angle}<\infty$. Define
\begin{align}
  q_{\mathrm b}&=0,
  &q_{\mathrm r}&=\lambda_{\angle}G_{\angle},
  \nonumber\\
  \Gamma_j&=K_j(s+K_jR_j)+q_j,
  &\omega_j&=\Lambda_jR_j(s+K_jR_j)
  +\tfrac12K_j\Lambda_jR_j^2+q_j.
  \label{eq:p2-local-envelopes}
\end{align}
Each $\mathcal E_j(\mathcal C)$ is a nonempty compact subset of a common
compact range $\mathcal E_0(\mathcal C)\subset(0,\infty)$ fixed before the
runs. For a fixed $\zeta_{\mathrm{stab}}\in(0,2)$, every
$\eta\in\mathcal E_j(\mathcal C)$ satisfies
\begin{equation}
  \eta\|H_j\|_{\mathrm{op}}\le2-\zeta_{\mathrm{stab}},
  \qquad
  \eta T_j\Gamma_j\le R_j.
  \label{eq:p2-admissible-steps}
\end{equation}
\end{assumption}

The Jacobian bound $K_j$ controls output sensitivity, and $\Lambda_j$
controls the linearization remainder. The angular-gradient bound contributes
$q_j$ to the refiner's update and perturbation bounds. The two step-size
conditions keep the tangent recursion stable and both parameter sequences
inside the radius $R_j$ through the charged horizon $T_j$. Compact candidate
sets ensure that each optimized risk reduction is attained.

Define the actual and tangent risk reductions by
\begin{align}
  \Delta_j(\mathcal C,\eta)
  &=\mathcal R_P(0)-\mathcal R_P(\bm{f}_{j,\eta}),
  &\Delta_j(\mathcal C)
  &=\sup_{\eta\in\mathcal E_j(\mathcal C)}
    \Delta_j(\mathcal C,\eta),
  \nonumber\\
  \Delta_j^{\mathrm{lin}}(\mathcal C,\eta)
  &=\bm{b}_j^*H_j^\dagger
    [I-(I-\eta H_j)^{2T_j}]\bm{b}_j,
  &\Delta_j^{\mathrm{lin}}(\mathcal C)
  &=\sup_{\eta\in\mathcal E_j(\mathcal C)}
    \Delta_j^{\mathrm{lin}}(\mathcal C,\eta).
  \label{eq:p2-risk-reductions}
\end{align}
Finally, put
\begin{equation}
  \varepsilon_j(\mathcal C,\eta)
  =\tfrac12\Lambda_jR_j^2+K_j\eta T_j\omega_j,
  \qquad
  e_j(\mathcal C)
  =\sup_{\eta\in\mathcal E_j(\mathcal C)}
  2(s+K_jR_j)\varepsilon_j(\mathcal C,\eta).
  \label{eq:p2-error-envelopes}
\end{equation}

\subsection{Finite-budget perturbation and separation}

\begin{proposition}[Finite-budget tangent certificate]
\label{p2:prop:finite-budget-certificate}
Under Assumption~\ref{p2:assump:finite-budget-certificate}, for every
$j\in\{\mathrm r,\mathrm b\}$ and
$\eta\in\mathcal E_j(\mathcal C)$,
\begin{equation}
  \|\bm{\alpha}_{j,t}^{\eta}\|,
  \|\bm{\beta}_{j,t}^{\eta}\|
  \le t\eta\Gamma_j\le R_j
  \quad (0\le t\le T_j),
  \label{eq:p2-app-ball}
\end{equation}
and
\begin{equation}
  \|F_j(\bm{\alpha}_{j,T_j}^{\eta})
  -\Phi_j\bm{\beta}_{j,T_j}^{\eta}\|
  \le\varepsilon_j(\mathcal C,\eta).
  \label{eq:p2-app-output-error}
\end{equation}
The tangent reduction is exactly
\begin{equation}
  \mathcal R_P(0)-
  \mathcal R_P(\Phi_j\bm{\beta}_{j,T_j}^{\eta})
  =\bm{b}_j^*H_j^\dagger
  [I-(I-\eta H_j)^{2T_j}]\bm{b}_j.
  \label{eq:p2-app-exact-tangent-reduction}
\end{equation}
Consequently,
\begin{equation}
  |\Delta_j(\mathcal C,\eta)
  -\Delta_j^{\mathrm{lin}}(\mathcal C,\eta)|
  \le2(s+K_jR_j)\varepsilon_j(\mathcal C,\eta),
  \label{eq:p2-app-candidate-risk-error}
\end{equation}
and
\begin{equation}
  |\Delta_j(\mathcal C)-\Delta_j^{\mathrm{lin}}(\mathcal C)|
  \le e_j(\mathcal C).
  \label{eq:p2-app-optimized-risk-error}
\end{equation}
The actual and tangent optimized reductions are attained on their respective
candidate sets.
\end{proposition}

\begin{proof}
We derive the tangent reduction and control the perturbation caused by
the actual updates through the full charged horizon.

\paragraph{Step 1: Establish the common risk geometry.}
We first record the common risk geometry. With
$\bm{\eta}_r=\bm{r}-\bar{\bm{r}}(W)$, conditional expectation gives
$\E[\bm{\eta}_r\mid W]=0$. Hence every state-measurable
$\bm{h}\in L^2(P;\mathcal X)$ satisfies
\begin{equation}
  \E\langle \bm{h}(W),\bm{\eta}_r\rangle
  =\E\langle \bm{h}(W),\E[\bm{\eta}_r\mid W]\rangle=0.
  \label{eq:p2-app-orthogonality}
\end{equation}
Expanding the square yields
\begin{equation}
  \mathcal R_P(\bm{h})
  =\E\|\bm{r}-\bar{\bm{r}}\|^2+\|\bar{\bm{r}}-\bm{h}\|^2,
  \qquad
  \mathcal R_P(0)-\mathcal R_P(\bm{h})
  =2\langle\bar{\bm{r}},\bm{h}\rangle-\|\bm{h}\|^2.
  \label{eq:p2-app-risk-geometry}
\end{equation}
Jensen's inequality also gives $\|\bar{\bm{r}}\|\le s$.

\paragraph{Step 2: Keep the actual and tangent iterates in the certified ball.}
For an in-ball parameter, the line-segment formula and $F_j(0)=0$ give
\begin{equation}
  F_j(\bm{\alpha})=\int_0^1J_j(u\bm{\alpha})[\bm{\alpha}],du,
  \qquad
  \|F_j(\bm{\alpha})\|\le K_j\|\bm{\alpha}\|.
  \label{eq:p2-app-output-envelope}
\end{equation}
Let
\begin{equation}
  \bm{\vartheta}_j(\bm{\alpha})=
  \begin{cases}
    \lambda_{\angle}\nabla\mathcal A(\bm{\alpha}),&j=\mathrm r,\\
    0,&j=\mathrm b.
  \end{cases}
  \label{eq:p2-app-angular-force}
\end{equation}
Differentiating \eqref{eq:p2-app-risk-geometry} through $F_j$ gives
\begin{equation}
  \nabla L_j(\bm{\alpha})
  =J_j(\bm{\alpha})^*(F_j(\bm{\alpha})-\bar{\bm{r}})+\bm{\vartheta}_j(\bm{\alpha}).
  \label{eq:p2-app-actual-gradient}
\end{equation}
Thus $\|\nabla L_j(\bm{\alpha})\|\le\Gamma_j$ whenever
$\|\bm{\alpha}\|\le R_j$. Since $\|\Phi_j\|_{\mathrm{op}}\le K_j$, the tangent
gradient satisfies
\begin{equation}
  \|H_j\bm{\beta}-\bm{b}_j\|
  =\|\Phi_j^*(\Phi_j\bm{\beta}-\bar{\bm{r}})\|
  \le K_j(K_jR_j+s)\le\Gamma_j
  \label{eq:p2-app-tangent-gradient-envelope}
\end{equation}
whenever $\|\bm{\beta}\|\le R_j$. Starting from zero, simultaneous induction in
the two recursions now gives
\begin{equation}
  \|\bm{\alpha}_{j,t+1}^{\eta}\|,
  \|\bm{\beta}_{j,t+1}^{\eta}\|
  \le(t+1)\eta\Gamma_j.
  \label{eq:p2-app-invariant-induction}
\end{equation}
The pre-run radius condition makes the right side at most $R_j$ through
$T_j$, completing the induction and proving \eqref{eq:p2-app-ball}.

\paragraph{Step 3: Solve the tangent recursion and compute its risk reduction.}
We next solve the tangent recursion. The operator $H_j=\Phi_j^*\Phi_j$ is
positive semidefinite. Moreover,
$\bm{b}_j\perp\ker H_j$ because $\ker H_j=\ker\Phi_j$. If
$H_j\bm{u}_k=\lambda_k\bm{u}_k$, the stability certificate gives
\begin{equation}
  |1-\eta\lambda_k|\le1,
  \qquad
  \|(I-\eta H_j)^t\|_{\mathrm{op}}\le1.
  \label{eq:p2-app-stable-propagator}
\end{equation}
The null eigenspaces have zero loading. A finite geometric sum on every
positive eigenspace therefore yields
\begin{equation}
  \bm{\beta}_{j,t}^{\eta}
  =H_j^\dagger[I-(I-\eta H_j)^t]\bm{b}_j.
  \label{eq:p2-app-tangent-solution}
\end{equation}
Substitution into \eqref{eq:p2-app-risk-geometry} gives
\begin{equation}
  \sum_{k:\lambda_k>0}
  \frac{|\langle \bm{u}_k,\bm{b}_j\rangle|^2}{\lambda_k}
  [1-(1-\eta\lambda_k)^{2T_j}],
  \label{eq:p2-app-spectral-sum}
\end{equation}
which is exactly \eqref{eq:p2-app-exact-tangent-reduction}.

\paragraph{Step 4: Bound the parameter and output perturbations.}
It remains to transfer the tangent formula to the actual run. The Jacobian
Lipschitz condition and the integral formula give
\begin{equation}
  \|F_j(\bm{\alpha})-\Phi_j\bm{\alpha}\|
  \le\frac12\Lambda_j\|\bm{\alpha}\|^2.
  \label{eq:p2-app-taylor-remainder}
\end{equation}
At the same parameter define
$\bm{\xi}_j(\bm{\alpha})=\nabla L_j(\bm{\alpha})-(H_j\bm{\alpha}-\bm{b}_j)$. Direct subtraction of
the two gradients gives
\begin{align}
  \bm{\xi}_j(\bm{\alpha})
  ={}&[J_j(\bm{\alpha})-\Phi_j]^*(F_j(\bm{\alpha})-\bar{\bm{r}})
  \nonumber\\
  &+\Phi_j^*(F_j(\bm{\alpha})-\Phi_j\bm{\alpha})
  +\bm{\vartheta}_j(\bm{\alpha}).
  \label{eq:p2-app-gradient-defect}
\end{align}
For $\|\bm{\alpha}\|\le R_j$, the three displayed terms are bounded by
$\Lambda_jR_j(s+K_jR_j)$,
$\frac12K_j\Lambda_jR_j^2$ and $q_j$. Hence
\begin{equation}
  \|\bm{\xi}_j(\bm{\alpha})\|\le\omega_j.
  \label{eq:p2-app-gradient-defect-bound}
\end{equation}
Writing $\bm{d}_{j,t}=\bm{\alpha}_{j,t}^{\eta}-\bm{\beta}_{j,t}^{\eta}$ and subtracting the
two recursions gives
\begin{equation}
  \bm{d}_{j,t+1}=(I-\eta H_j)\bm{d}_{j,t}
  -\eta\bm{\xi}_j(\bm{\alpha}_{j,t}^{\eta}),
  \qquad \bm{d}_{j,0}=0.
  \label{eq:p2-app-parameter-defect-recursion}
\end{equation}
Equations~\ref{eq:p2-app-ball}, \ref{eq:p2-app-stable-propagator}
and~\ref{eq:p2-app-gradient-defect-bound} imply
\begin{equation}
  \|\bm{d}_{j,t}\|
  \le\eta\sum_{\ell=0}^{t-1}
  \|\bm{\xi}_j(\bm{\alpha}_{j,\ell}^{\eta})\|
  \le\eta t\omega_j.
  \label{eq:p2-app-parameter-defect-bound}
\end{equation}
Adding and subtracting $\Phi_j\bm{\alpha}_{j,T_j}^{\eta}$, then using
\eqref{eq:p2-app-taylor-remainder}, proves
\eqref{eq:p2-app-output-error}.

\paragraph{Step 5: Transfer the output bound to risk reductions.}
Let $\bm{f}=\bm{f}_{j,\eta}$ and
$\widetilde{\bm{f}}=\Phi_j\bm{\beta}_{j,T_j}^{\eta}$. Both have norm at most $K_jR_j$.
The polarization identity and \eqref{eq:p2-app-risk-geometry} give
\begin{align}
  |\mathcal R_P(\bm{f})-\mathcal R_P(\widetilde{\bm{f}})|
  &\le\|\bm{f}-\widetilde{\bm{f}}\|
  (2\|\bar{\bm{r}}\|+\|\bm{f}\|+\|\widetilde{\bm{f}}\|)
  \nonumber\\
  &\le2(s+K_jR_j)\varepsilon_j(\mathcal C,\eta),
  \label{eq:p2-app-risk-polarization}
\end{align}
which proves \eqref{eq:p2-app-candidate-risk-error}. Applying
$|\sup u-\sup \bm{v}|\le\sup|u-\bm{v}|$ on the same armwise candidate set proves
\eqref{eq:p2-app-optimized-risk-error}.

\paragraph{Step 6: Establish attainment on the candidate sets.}
Finally, the actual gradient is continuous on the certified ball. Finite
induction in \eqref{eq:p2-actual-recursion} makes the terminal actual output
continuous in $\eta$. The tangent terminal output is a finite matrix
polynomial in $\eta$. Both reductions are therefore continuous on the
nonempty compact set $\mathcal E_j(\mathcal C)$, so their optima are attained.
\end{proof}

\begin{proof}[Proof of Theorem~\ref{thm:p2-matched-resource}]
We transfer the certified tangent margin to the actual runs, then construct
an instance satisfying all the stated conditions.

\paragraph{Step 1: Transfer the tangent advantage to trained risk.}
Proposition~\ref{p2:prop:finite-budget-certificate} proves the finite-horizon
iterate and output bounds, the exact tangent formula and the optimized error
certificate. Its two-sided bounds and the strict margin give the positive chain
\begin{equation}
  \Delta_{\mathrm r}
  \ge\Delta_{\mathrm r}^{\mathrm{lin}}-e_{\mathrm r}
  >\Delta_{\mathrm b}^{\mathrm{lin}}+e_{\mathrm b}
  \ge\Delta_{\mathrm b}.
  \label{eq:p2-app-strict-chain}
\end{equation}
Both reductions use the common baseline $\mathcal R_P(0)$. Subtracting
\eqref{eq:p2-app-strict-chain} from that baseline and using attainment proves
\eqref{eq:p2-actual-order}.

\paragraph{Step 2: Construct and bound the angular objective.}
We finish by constructing the stated nonempty instance. Let
$(\zeta_1,\zeta_2)\sim\mathcal N(0,I_2)$ and let $\bm{e}\in\mathcal X$ be a
unit vector. For $\bm{\alpha}=(\alpha_1,\alpha_2)$, put
\begin{equation}
  s_{\bm{\alpha}}=\alpha_1\zeta_1+\alpha_2\zeta_2.
  \label{eq:p2-app-witness-scalar}
\end{equation}
With prediction $s_{\bm{\alpha}}\bm{e}$, target
$\beta\zeta_2\bm{e}$ and the implemented strict threshold $\tau$, the
population angular objective is exactly
\begin{equation}
  \mathcal A_\beta(\bm{\alpha})
  =1-\E\!\left[
  \operatorname{sgn}(s_{\bm{\alpha}}\zeta_2)
  \mathbf 1\{|s_{\bm{\alpha}}|>\tau,
  |\zeta_2|>\tau/\beta\}\right].
  \label{eq:p2-app-angular-objective}
\end{equation}
We verify the gradient envelope consumed below. Let
$\varphi(t)=(2\pi)^{-1/2}e^{-t^2/2}$ and
$\overline\Phi(t)=\int_t^\infty\varphi(s)\,\mathrm ds$ denote the standard
Gaussian density and upper-tail function. Write
$\bm{\alpha}=\rho(\cos\vartheta,\sin\vartheta)$ and
$t=\tau/\rho$. Conditioning the signed integrand on the standard normal
$Z_\vartheta=\cos\vartheta\,\zeta_1+\sin\vartheta\,\zeta_2$, define
$m_\vartheta(z)$ as its conditional expectation given
$Z_\vartheta=z$. Since $|m_\vartheta|\le1$, differentiation of
\begin{equation}
  \int_{|z|>t}m_\vartheta(z)\varphi(z)\,dz
  \label{eq:p2-app-angular-radial-integral}
\end{equation}
at the moving cutoff gives
\begin{equation}
  |\partial_\rho\mathcal A_\beta|
  \le\frac{2\tau}{\rho^2}\varphi(\tau/\rho)
  \le\frac{4}{e\sqrt{2\pi}\,\tau}.
  \label{eq:p2-app-angular-radial}
\end{equation}
For the angular derivative, use polar Gaussian coordinates
\begin{equation}
  (\zeta_1,\zeta_2)=R(\cos\phi,\sin\phi).
  \label{eq:p2-app-polar-gaussian}
\end{equation}
At each $R>\tau/\rho$, put
$\gamma_R=\arccos(\tau/(\rho R))$ and let $q_R(\phi)$ be the signed target
mask, which is bounded by one and satisfies
$q_R(\phi+\pi)=-q_R(\phi)$. The conditional angular average is
\begin{equation}
  \frac1\pi
  \int_{\vartheta-\gamma_R}^{\vartheta+\gamma_R}q_R(\phi)\,d\phi.
  \label{eq:p2-app-angular-arc}
\end{equation}
Translating this interval changes the integral at rate at most $2/\pi$.
Integration against the polar Gaussian density therefore gives
\begin{equation}
  \frac1\rho|\partial_\vartheta\mathcal A_\beta|
  \le\frac{2}{\pi\rho}e^{-\tau^2/(2\rho^2)}
  \le\frac{2}{\pi\sqrt e\,\tau}.
  \label{eq:p2-app-angular-tangential}
\end{equation}
For fixed $(\rho,\vartheta)$ with $\rho<\beta$, cutoff intersections occur at
at most four radii and have zero polar measure. The uniform translation bound
above therefore permits dominated differentiation in $R$. We also have
\begin{equation}
  |\mathcal A_\beta(\bm{\alpha})-1|
  \le2\overline\Phi(\tau/\rho)=o(\rho)
  \quad\text{as }\rho\downarrow0,
  \label{eq:p2-app-angular-origin}
\end{equation}
and both displayed derivative bounds converge uniformly to zero. Away from
the two cutoff collisions $\bm{\alpha}=\pm\beta(0,1)$, the same cutoff integrals
vary continuously. Thus on
every ball $\|\bm{\alpha}\|\le R<\beta$, the objective is continuously
differentiable, $\nabla\mathcal A_\beta(0)=0$ and
\begin{equation}
  \|\nabla\mathcal A_\beta(\bm{\alpha})\|
  \le\left(\frac4{e\sqrt{2\pi}}+
  \frac2{\pi\sqrt e}\right)\tau^{-1}<\tau^{-1}.
  \label{eq:p2-app-angular-envelope}
\end{equation}

\paragraph{Step 3: Realize the correction heads on a sampler state.}
We now realize these coordinates through a state of the
linear-interpolant flow. Draw
$\bm{\xi}\sim\mathcal N(0,I_d)$ for $d\ge2$, take $\mathbf{X}_0=0$ and deterministic $\bm{C}$,
let $\Sigma$ be uniform on $[0,1]$ and set
\begin{equation}
  \mathbf{X}_\sigma=\sigma\bm{\xi},
  \qquad \bm{U}=\bm{\xi}.
  \label{eq:p2-app-wan-law}
\end{equation}
Let $\zeta_1,\zeta_2$ be two orthonormal coordinates of $\bm{\xi}$. For
$\sigma>0$, define
\begin{equation}
  \bm{v}_0(\mathbf{X}_\sigma,\sigma,\bm{C})=\bm{\xi}-\beta\zeta_2\bm{e}.
  \label{eq:p2-app-base-witness}
\end{equation}
The time law assigns zero mass to $\sigma=0$, where the checkpoint can be
defined arbitrarily. The current state recovers $\bm{\xi}=\mathbf{X}_\sigma/\sigma$, so the
stored feature $(\zeta_1,\zeta_2)$ is a deterministic function of the current
base input. Moreover,
\begin{equation}
  \bm{r}=\bar{\bm{r}}=\beta\zeta_2\bm{e},
  \qquad s=\beta.
  \label{eq:p2-app-deterministic-residual}
\end{equation}
Use the actual two-coordinate heads
\begin{align}
  F_{\mathrm r}(\bm{\alpha})
  &=(\alpha_1\zeta_1+\alpha_2\zeta_2)\bm{e},
  \nonumber\\
  F_{\mathrm b}(\bm{\alpha})
  &=(\alpha_1\zeta_1+\sqrt\kappa\,\alpha_2\zeta_2)\bm{e},
  \qquad 0<\kappa<1.
  \label{eq:p2-app-linear-heads}
\end{align}
Their Jacobians are constant, so $K_j=1$ and $\Lambda_j=0$. Independence and
unit variances give
\begin{equation}
  H_{\mathrm r}=I_2,
  \quad \bm{b}_{\mathrm r}=(0,\beta)^\top,
  \qquad
  H_{\mathrm b}=\operatorname{diag}(1,\kappa),
  \quad \bm{b}_{\mathrm b}=(0,\sqrt\kappa\,\beta)^\top.
  \label{eq:p2-app-witness-gram}
\end{equation}

\paragraph{Step 4: Choose admissible parameters and equal resource charges.}
Set $\tau=10^{-8}$, $G_{\angle}=\tau^{-1}$,
$Q_{\angle}=\lambda_{\angle}G_{\angle}$, choose
$\beta\ge4Q_{\angle}$ and then choose
\begin{equation}
  \kappa=\frac14,
  \qquad \eta=\frac14,
  \qquad \zeta_{\mathrm{stab}}=1,
  \qquad R_{\mathrm b}=\frac\beta3,
  \qquad R_{\mathrm r}=\frac{\beta+Q_{\angle}}3.
  \label{eq:p2-app-witness-parameters}
\end{equation}
\Eqref{eq:p2-app-angular-envelope} applies on the refiner ball
because $R_{\mathrm r}\le5\beta/12<\beta$. Choose symmetric two-parameter
linear-head schedules with common measured update and state charges $c_*$ and
$m_*$. Let $d_*$ be the added refiner inference charge and declare
$\mathcal C=(c_*,m_*,d_*)$. Both arms fit this ledger and have horizon one.
Fixing the common candidate set to $\{1/4\}$ before either run gives
\begin{equation}
  \eta\|H_j\|_{\mathrm{op}}=\frac14\le1,
  \qquad
  \eta\Gamma_{\mathrm b}=R_{\mathrm b},
  \qquad
  \eta\Gamma_{\mathrm r}=R_{\mathrm r}.
  \label{eq:p2-app-witness-admissibility}
\end{equation}
Thus every local and resource condition is satisfied.

\paragraph{Step 5: Verify the strict risk-reduction margin.}
Finally, \eqref{eq:p2-app-exact-tangent-reduction} on the loaded eigendirection
gives
\begin{equation}
  \Delta_{\mathrm r}^{\mathrm{lin}}
  =\beta^2[1-(1-1/4)^2]=\frac7{16}\beta^2,
  \qquad
  \Delta_{\mathrm b}^{\mathrm{lin}}
  =\beta^2[1-(1-1/16)^2]=\frac{31}{256}\beta^2.
  \label{eq:p2-app-witness-reductions}
\end{equation}
Since $\Lambda_j=0$, we have
$\omega_{\mathrm b}=0$, $\omega_{\mathrm r}=Q_{\angle}$,
$\varepsilon_{\mathrm b}=0$ and
$\varepsilon_{\mathrm r}=Q_{\angle}/4$. Therefore
\begin{equation}
  e_{\mathrm b}=0,
  \qquad
  e_{\mathrm r}
  =2(\beta+R_{\mathrm r})\frac{Q_{\angle}}4
  =\frac{Q_{\angle}(4\beta+Q_{\angle})}{6}.
  \label{eq:p2-app-witness-errors}
\end{equation}
The tangent gap is $81\beta^2/256$, while
\begin{equation}
  \frac{e_{\mathrm r}}{\beta^2}
  =\frac16\frac{Q_{\angle}}\beta
  \left(4+\frac{Q_{\angle}}\beta\right)
  \le\frac{17}{96}<\frac{81}{256}.
  \label{eq:p2-app-witness-margin}
\end{equation}
This proves the strict margin and completes the theorem.
\end{proof}

% ---------------------------------------------------------------------------
% Perspective 2, deployment exports: gain algebra, coupled sampler laws and
% the endpoint bridge. G.3 uses a direct comparison under each sampler's law.
% Theorem~\ref{thm:own-law-improvement} summarizes the own-law certificate
% in Section~\ref{sec:theory}. The endpoint analysis remains appendix-only.
% ---------------------------------------------------------------------------

\section{Deployment Analysis: Gain, Sampler Law and Endpoint}
\label{p2:app:deployment}

This section establishes how the correction gain controls deployed velocity
error and endpoint distance. We first derive an exact gain quadratic at a
fixed state law. We then couple the base and corrected trajectories to
certify velocity improvement under each sampler's own law. Finally, we
analyze how velocity corrections accumulate along the trajectory and derive
an exact endpoint gain identity for affine--Gaussian flows.

\subsection{Deployment setting}
\label{p2:app:deployment-setting}

Fix the trained pair and let $\mathcal D_{\mathrm{dep}}$ denote the full
deployment domain. Write $\bm{g}_{\bm{\phi}}$ for the composite deployed correction
$W\mapsto \bm{g}_{\bm{\phi}}(\mathcal{Z}(W))$.
Fix a measurable reference velocity field
\[
\bm{u}^\star:\mathcal D_{\mathrm{dep}}\to\mathcal X
\]
such that
\[
\bm{u}^\star(W)=\E[\bm{U}\mid W]
\qquad P\text{-almost surely}.
\]
Set
\[
\bar{\bm{r}}=\bm{u}^\star-\bm{v}_0.
\]
For a state law $Q$ and a gain $\kappa\ge0$, define the deployed velocity
error relative to this reference field, together with its alignment and
energy coefficients,
\begin{equation}
  \mathcal V_Q(\kappa)=\E_Q\|\bar{\bm{r}}-\kappa \bm{g}_{\bm{\phi}}\|^2,
  \qquad
  A_Q=\E_Q\langle\bar{\bm{r}},\bm{g}_{\bm{\phi}}\rangle,
  \qquad
  B_Q=\E_Q\|\bm{g}_{\bm{\phi}}\|^2.
  \label{eq:p2-dep-coefficients}
\end{equation}
On the training law these connect to the risk used above. Taking
$\bm{h}=\kappa \bm{g}_{\bm{\phi}}$ in \eqref{eq:p2-app-risk-geometry} gives
\begin{equation}
  \mathcal R_P(\kappa \bm{g}_{\bm{\phi}})
  =\E_{\widetilde P}\|\bm{r}-\bar{\bm{r}}\|^2+\mathcal V_P(\kappa),
  \label{eq:p2-dep-risk-link}
\end{equation}
so $\mathcal V_P$ is the entire part of the training risk that the gain moves.

The deployed sampler determines its state law through its trajectory. Running
\Eqref{eq:corrected} from the prescribed Gaussian source with the composite
field $\bm{v}_0+\kappa \bm{g}_{\bm{\phi}}$, let $\mathbf{X}^\kappa(\sigma)$ denote the exact trajectory in
the sampler orientation $\sigma:1\to0$ and put
\begin{equation}
  Q_\kappa=\operatorname{Law}\bigl(\mathbf{X}^\kappa(\Sigma),\Sigma,\bm{C}\bigr),
  \qquad \Sigma\sim q.
  \label{eq:p2-dep-sampler-law}
\end{equation}
Here $q$ is a probability law on $[0,1]$, and $\Sigma$ is independent of the
source and condition. The law $Q_0$ is visited by the uncorrected sampler.
We compare these two deployment laws directly in
Appendix~\ref{p2:app:transfer}.

\begin{assumption}[Training-law moments]
\label{p2:assump:deployment-moments}
The realized residual satisfies
$\bm{r}\in L^2(\widetilde P;\mathcal X)$, and the deployed correction satisfies
$\bm{g}_{\bm{\phi}}\in L^2(P;\mathcal X)$.
The specified field $\bar{\bm{r}}$ agrees with
$\mathbb E_{\widetilde P}[\bm{r}\mid W]$ $P$-almost surely. Conditional Jensen's
inequality therefore gives $\bar{\bm{r}}\in L^2(P;\mathcal X)$.
\end{assumption}

\begin{assumption}[Deployment ODE well-posedness]
\label{p2:assump:deployment-ode}
For a fixed gain radius $\bar\kappa>0$, the family of deployed velocity fields
\[
\left\{
\bm{v}_0+\kappa \bm{g}_{\bm{\phi}}:\kappa\in[0,\bar\kappa]
\right\}
\]
is measurable in $\sigma$ and globally Lipschitz in the state, with a common
integrable Lipschitz envelope. Their values at one reference state admit a
common integrable envelope over the same gain range. Consequently, their exact
ODEs admit unique global solutions on $[0,1]$, and the laws $Q_\kappa$ of
\eqref{eq:p2-dep-sampler-law} are well-defined for every
$\kappa\in[0,\bar\kappa]$.
\end{assumption}

\subsection{Exact gain quadratic and improvement interval}
\label{p2:app:gain}

At a fixed state law, the deployed velocity error is an exact quadratic in
the gain. Residual--correction alignment $A_Q$ determines the initial
slope, and correction energy $B_Q$ determines the curvature. Together,
they specify the improvement interval and its optimal gain.

\begin{proposition}[Exact gain quadratic]
\label{p2:prop:gain-quadratic}
Let $Q$ be a probability law with
$\E_Q\|\bar{\bm{r}}\|^2<\infty$ and $\E_Q\|\bm{g}_{\bm{\phi}}\|^2<\infty$.
Then, for every $\kappa\in[0,\infty)$,
\begin{equation}
  \mathcal V_Q(\kappa)-\mathcal V_Q(0)
  =-2\kappa A_Q+\kappa^2B_Q.
  \label{eq:p2-dep-gain-quadratic}
\end{equation}
\end{proposition}

\begin{proof}
We verify finiteness of every term before expanding the squared velocity error.

\phantomsection
\label{p2:lem:gain-domain}
\paragraph{Step 1: Establish integrability on the deployment law.}
First, $A_Q$ is a finite
real number, $B_Q\in[0,\infty)$, $\mathcal V_Q(\kappa)<\infty$ for every finite
$\kappa\ge0$ and
\begin{equation}
  |A_Q|\le\mathcal V_Q(0)^{1/2}B_Q^{1/2}.
  \label{eq:p2-dep-cauchy-schwarz}
\end{equation}

By definition $\mathcal V_Q(0)=\E_Q\|\bar{\bm{r}}\|^2<\infty$ and
$B_Q=\E_Q\|\bm{g}_{\bm{\phi}}\|^2\in[0,\infty)$. Cauchy--Schwarz in
$L^2(Q;\mathcal X)$ gives
\[
  \E_Q|\langle\bar{\bm{r}},\bm{g}_{\bm{\phi}}\rangle|
  \le\bigl(\E_Q\|\bar{\bm{r}}\|^2\bigr)^{1/2}
     \bigl(\E_Q\|\bm{g}_{\bm{\phi}}\|^2\bigr)^{1/2}<\infty,
\]
which makes $A_Q$ finite and proves
\eqref{eq:p2-dep-cauchy-schwarz}. For a finite gain the pointwise bound
$\|\bar{\bm{r}}-\kappa \bm{g}_{\bm{\phi}}\|^2\le2\|\bar{\bm{r}}\|^2+2\kappa^2\|\bm{g}_{\bm{\phi}}\|^2$ has an
integrable right side, so $\mathcal V_Q(\kappa)<\infty$.

\paragraph{Step 2: Expand the gain quadratic at the fixed law.}
Pointwise in the state,
$\|\bar{\bm{r}}-\kappa \bm{g}_{\bm{\phi}}\|^2
 =\|\bar{\bm{r}}\|^2-2\kappa\langle\bar{\bm{r}},\bm{g}_{\bm{\phi}}\rangle+\kappa^2\|\bm{g}_{\bm{\phi}}\|^2$.
\hyperref[p2:lem:gain-domain]{Step~1} makes all three terms integrable and the left
side finite, so taking $Q$-expectations and subtracting
$\mathcal V_Q(0)$ gives
\eqref{eq:p2-dep-gain-quadratic}. Under $Q$, alignment sets the linear
coefficient and correction energy sets the quadratic coefficient.
\end{proof}

\begin{proposition}[Sharp gain boundaries]
\label{p2:prop:gain-boundaries}
Under Proposition~\ref{p2:prop:gain-quadratic}, exactly one of the following
holds.
\begin{enumerate}[label=(\roman*)]
\item If $B_Q=0$, then $\bm{g}_{\bm{\phi}}=0$ $Q$-almost surely, $A_Q=0$ and
  $\mathcal V_Q(\kappa)=\mathcal V_Q(0)$ for every $\kappa\ge0$.
\item If $B_Q>0$ and $A_Q\le0$, then
  $\mathcal V_Q(\kappa)>\mathcal V_Q(0)$ for every $\kappa>0$.
\item If $B_Q>0$ and $A_Q>0$, then
  \begin{equation}
    \mathcal V_Q(\kappa)<\mathcal V_Q(0)
    \quad\Longleftrightarrow\quad
    0<\kappa<\frac{2A_Q}{B_Q},
    \label{eq:p2-dep-gain-window}
  \end{equation}
  with equality exactly at the two endpoints $\kappa=0$ and
  $\kappa=2A_Q/B_Q$, and $\mathcal V_Q(\kappa)>\mathcal V_Q(0)$ for every
  larger gain.
\end{enumerate}
\end{proposition}

\begin{proof}
If $B_Q=0$, the nonnegative variable $\|\bm{g}_{\bm{\phi}}\|^2$ has zero expectation and
vanishes $Q$-almost surely. Then $A_Q=0$ and
$\bar{\bm{r}}-\kappa \bm{g}_{\bm{\phi}}=\bar{\bm{r}}$ $Q$-almost surely for every $\kappa$, which is
case (i).

Suppose $B_Q>0$. Factoring \eqref{eq:p2-dep-gain-quadratic} gives
\begin{equation}
  \mathcal V_Q(\kappa)-\mathcal V_Q(0)
  =\kappa\bigl(\kappa B_Q-2A_Q\bigr).
  \label{eq:p2-dep-gain-factored}
\end{equation}
If $A_Q\le0$ and $\kappa>0$, then $\kappa B_Q>0$ and $-2A_Q\ge0$, so both
factors are positive, which is case (ii). If $A_Q>0$ and $\kappa>0$, the sign
of \eqref{eq:p2-dep-gain-factored} is the sign of $\kappa B_Q-2A_Q$, which is
negative, zero, or positive according as $\kappa$ is below, equal to, or above
$2A_Q/B_Q$. The separate factor $\kappa$ supplies equality at $\kappa=0$. This
is case (iii).
\end{proof}

\begin{corollary}[Unique optimal gain]
\label{p2:cor:optimal-gain}
If $A_Q>0$ and $B_Q>0$, the unique minimizer of
$\kappa\mapsto\mathcal V_Q(\kappa)$ on $[0,\infty)$ is
\begin{equation}
  \kappa^\star=\frac{A_Q}{B_Q},
  \qquad
  \mathcal V_Q(\kappa^\star)
  =\mathcal V_Q(0)-\frac{A_Q^2}{B_Q}\ \ge0.
  \label{eq:p2-dep-optimal-gain}
\end{equation}
\end{corollary}

\begin{proof}
Completing the square in \eqref{eq:p2-dep-gain-quadratic} gives
\[
  \mathcal V_Q(\kappa)
  =\mathcal V_Q(0)-\frac{A_Q^2}{B_Q}
   +B_Q\Bigl(\kappa-\frac{A_Q}{B_Q}\Bigr)^2 .
\]
The last term is nonnegative and vanishes at exactly one scalar, which is
positive because $A_Q>0$ and therefore lies in the allowed domain.
Substitution gives the displayed value, and
\eqref{eq:p2-dep-cauchy-schwarz} gives $A_Q^2/B_Q\le\mathcal V_Q(0)$, which is
the stated nonnegativity.
\end{proof}

\subsection{Own-law velocity improvement}
\label{p2:app:transfer}

The next theorem establishes a gain interval over which the corrected
sampler achieves lower velocity error than the base sampler, with each field
evaluated under its own state law. Coupling the trajectories through the
same initial noise and condition lets us control their displacement using
spatial regularity. Combining this displacement bound with the fixed-law
gain identity gives the comparison between $Q_\kappa$ and $Q_0$.

\paragraph{Coupled trajectories.}
Let $\mathbf{X}_1$ be the prescribed Gaussian source and $\bm{C}$ the shared condition.
For each $\kappa\in[0,\bar\kappa]$, use the exact descending ODEs
\begin{align}
  \frac{d\mathbf{X}^0(\sigma)}{d\sigma}
    &=\bm{v}_0(\mathbf{X}^0(\sigma),\sigma,\bm{C}),
    & \mathbf{X}^0(1)&=\mathbf{X}_1, \\
  \frac{d\mathbf{X}^\kappa(\sigma)}{d\sigma}
    &=\bm{v}_0(\mathbf{X}^\kappa(\sigma),\sigma,\bm{C})
      +\kappa \bm{g}_{\bm{\phi}}(\mathbf{X}^\kappa(\sigma),\sigma,\bm{C}),
    & \mathbf{X}^\kappa(1)&=\mathbf{X}_1.
  \label{eq:p2-dep-coupled-odes}
\end{align}
Sample the same independent $\Sigma\sim q$ for both trajectories. Their
respective marginals are $Q_0$ and $Q_\kappa$ as defined in
\eqref{eq:p2-dep-sampler-law}. Write $A_0:=A_{Q_0}$ for the alignment along
the baseline trajectory.

\begin{assumption}[Uniform bounds and spatial regularity]
\label{p2:assump:deployment-regularity}
The bounds below control the magnitude and spatial variation of the residual
$\bar{\bm{r}}$. The reference velocity $\bm{u}^\star$ determines this field through
$\bar{\bm{r}}=\bm{u}^\star-\bm{v}_0$.
There is a deterministic common domain that, on an event of probability one,
contains the coupled trajectories simultaneously for every
$\sigma\in[0,1]$ and $\kappa\in[0,\bar\kappa]$.
On this domain, finite nonnegative constants $R_r,G,L_r,L_g,L_0$ satisfy
\begin{equation}
  \|\bar{\bm{r}}(\mathbf{x},\sigma,\bm{c})\|\le R_r,
  \qquad
  \|\bm{g}_{\bm{\phi}}(\mathbf{x},\sigma,\bm{c})\|\le G,
  \label{eq:p2-dep-field-bounds}
\end{equation}
and, for any two states $\mathbf{x},\mathbf{y}$ in the domain at the same $(\sigma,\bm{c})$,
\begin{align}
  \|\bar{\bm{r}}(\mathbf{x},\sigma,\bm{c})-\bar{\bm{r}}(\mathbf{y},\sigma,\bm{c})\|
    &\le L_r\|\mathbf{x}-\mathbf{y}\|, \\
  \|\bm{g}_{\bm{\phi}}(\mathbf{x},\sigma,\bm{c})-\bm{g}_{\bm{\phi}}(\mathbf{y},\sigma,\bm{c})\|
    &\le L_g\|\mathbf{x}-\mathbf{y}\|, \\
  \|\bm{v}_0(\mathbf{x},\sigma,\bm{c})-\bm{v}_0(\mathbf{y},\sigma,\bm{c})\|
    &\le L_0\|\mathbf{x}-\mathbf{y}\|.
  \label{eq:p2-dep-spatial-regularity}
\end{align}
The constants are uniform in the noise level, condition, source realization
and gain throughout the stated range.
\end{assumption}

Assumption~\ref{p2:assump:deployment-ode} supplies existence and uniqueness.
Assumption~\ref{p2:assump:deployment-regularity} quantifies how the correction
moves the trajectory and how this displacement affects residual error and
alignment across the common deployment domain.

The following theorem gives the full quantitative certificate summarized
in the main-text Theorem~\ref{thm:own-law-improvement}.
\begin{theorem}[Detailed own-law velocity certificate]
\label{p2:prop:generated-order}
Under Assumptions~\ref{p2:assump:deployment-ode} and
\ref{p2:assump:deployment-regularity}, define
\begin{equation}
  H_0(\sigma):=
  \begin{cases}
    \dfrac{e^{L_0(1-\sigma)}-1}{L_0},&L_0>0,\\[6pt]
    1-\sigma,&L_0=0,
  \end{cases}
  \qquad
  K_Q:=G\,\E_{\Sigma\sim q}[H_0(\Sigma)].
  \label{eq:p2-dep-displacement-constant}
\end{equation}
Set
\begin{equation}
  L_V:=2R_rL_r,
  \qquad L_A:=L_rG+R_rL_g,
  \qquad M:=2A_0-L_VK_Q,
  \qquad N:=G^2+2L_AK_Q.
  \label{eq:p2-dep-own-law-constants}
\end{equation}
Then for every $\kappa\in[0,\bar\kappa]$,
\begin{equation}
  \boxed{
  \mathcal V_{Q_\kappa}(\kappa)-\mathcal V_{Q_0}(0)
  \le-\kappa M+\kappa^2N.}
  \label{eq:p2-dep-own-law-bound}
\end{equation}
If $M>0$, then $N>0$, and every gain satisfying
\begin{equation}
  0<\kappa\le\bar\kappa,
  \qquad \kappa<\frac{M}{N}
  \label{eq:p2-dep-certificate}
\end{equation}
gives the strict comparison
\begin{equation}
  \boxed{\mathcal V_{Q_\kappa}(\kappa)<\mathcal V_{Q_0}(0).}
  \label{eq:p2-dep-generated-order}
\end{equation}
Each velocity field is therefore evaluated under the state law generated by
its own sampler.
\end{theorem}

\begin{proof}
Fix $\kappa\in[0,\bar\kappa]$ and use the common-source coupling throughout.

\paragraph{Step 1: Bound the trajectory displacement.}
The integral form of \eqref{eq:p2-dep-coupled-odes}, with the common source
cancelled, gives
\begin{align}
  \|\mathbf{X}^\kappa(\sigma)-\mathbf{X}^0(\sigma)\|
  &\le\int_\sigma^1
    \|\bm{v}_0(\mathbf{X}^\kappa(s),s,\bm{C})-\bm{v}_0(\mathbf{X}^0(s),s,\bm{C})\|\,ds\notag\\
  &\quad+\kappa\int_\sigma^1\|\bm{g}_{\bm{\phi}}(\mathbf{X}^\kappa(s),s,\bm{C})\|\,ds\notag\\
  &\le L_0\int_\sigma^1\|\mathbf{X}^\kappa(s)-\mathbf{X}^0(s)\|\,ds
       +\kappa G(1-\sigma).
  \label{eq:p2-dep-displacement-integral}
\end{align}
To handle the descending orientation, put $t=1-\sigma$ and
$\varphi(t)=\|\mathbf{X}^\kappa(1-t)-\mathbf{X}^0(1-t)\|$. Then
\[
  \varphi(t)\le L_0\int_0^t\varphi(s)\,ds+\kappa Gt.
\]
For $L_0>0$, Gr\"onwall's inequality yields
\begin{align*}
  \varphi(t)
  &\le\kappa G\left(t+L_0\int_0^t s e^{L_0(t-s)}\,ds\right)\\
  &=\kappa G\frac{e^{L_0t}-1}{L_0}.
\end{align*}
For $L_0=0$, the preceding integral inequality directly gives
$\varphi(t)\le\kappa Gt$. Thus
\begin{equation}
  \|\mathbf{X}^\kappa(\sigma)-\mathbf{X}^0(\sigma)\|
  \le\kappa GH_0(\sigma),
  \qquad
  \E\|\mathbf{X}^\kappa(\Sigma)-\mathbf{X}^0(\Sigma)\|\le\kappa K_Q.
  \label{eq:p2-dep-coupled-displacement}
\end{equation}
The expectation uses the shared source, condition and sampled noise level.

\paragraph{Step 2: Control the change in baseline error.}
Suppressing the fixed arguments $(\sigma,\bm{c})$, the residual bounds give
\begin{align*}
  \bigl|\|\bar{\bm{r}}(\mathbf{x})\|^2-\|\bar{\bm{r}}(\mathbf{y})\|^2\bigr|
  &\le\bigl(\|\bar{\bm{r}}(\mathbf{x})\|+\|\bar{\bm{r}}(\mathbf{y})\|\bigr)
        \|\bar{\bm{r}}(\mathbf{x})-\bar{\bm{r}}(\mathbf{y})\|\\
  &\le2R_rL_r\|\mathbf{x}-\mathbf{y}\|=L_V\|\mathbf{x}-\mathbf{y}\|.
\end{align*}
Applying this inequality to the coupled states in
\eqref{eq:p2-dep-coupled-displacement} gives
\begin{align}
  \mathcal V_{Q_\kappa}(0)-\mathcal V_{Q_0}(0)
  &=\E\left[
    \|\bar{\bm{r}}(\mathbf{X}^\kappa(\Sigma),\Sigma,\bm{C})\|^2
    -\|\bar{\bm{r}}(\mathbf{X}^0(\Sigma),\Sigma,\bm{C})\|^2\right]\notag\\
  &\le L_V\E\|\mathbf{X}^\kappa(\Sigma)-\mathbf{X}^0(\Sigma)\|
   \le\kappa L_VK_Q.
  \label{eq:p2-dep-baseline-shift}
\end{align}

\paragraph{Step 3: Control the change in alignment.}
At the same fixed $(\sigma,\bm{c})$,
\begin{align*}
  &\bigl|\langle\bar{\bm{r}}(\mathbf{x}),\bm{g}_{\bm{\phi}}(\mathbf{x})\rangle
         -\langle\bar{\bm{r}}(\mathbf{y}),\bm{g}_{\bm{\phi}}(\mathbf{y})\rangle\bigr|\\
  &\quad\le
    \bigl|\langle\bar{\bm{r}}(\mathbf{x})-\bar{\bm{r}}(\mathbf{y}),\bm{g}_{\bm{\phi}}(\mathbf{x})\rangle\bigr|
    +\bigl|\langle\bar{\bm{r}}(\mathbf{y}),\bm{g}_{\bm{\phi}}(\mathbf{x})-\bm{g}_{\bm{\phi}}(\mathbf{y})\rangle\bigr|\\
  &\quad\le(L_rG+R_rL_g)\|\mathbf{x}-\mathbf{y}\|=L_A\|\mathbf{x}-\mathbf{y}\|.
\end{align*}
Taking expectations under the coupling therefore yields
\begin{equation}
  A_{Q_\kappa}
  \ge A_0-L_A\E\|\mathbf{X}^\kappa(\Sigma)-\mathbf{X}^0(\Sigma)\|
  \ge A_0-\kappa L_AK_Q.
  \label{eq:p2-dep-alignment-shift}
\end{equation}

\paragraph{Step 4: Combine the fixed-law gain and the trajectory shift.}
The uniform bounds give $B_{Q_\kappa}\le G^2$ and
$\mathcal V_{Q_\kappa}(0)\le R_r^2$ for every allowed gain, including zero.
\hyperref[p2:lem:gain-domain]{Step~1} of the proof of
Proposition~\ref{p2:prop:gain-quadratic} makes all the risks and coefficients finite.
We may therefore apply Proposition~\ref{p2:prop:gain-quadratic} at the single
law $Q=Q_\kappa$ and then add the change in baseline error:
\begin{align}
  &\mathcal V_{Q_\kappa}(\kappa)-\mathcal V_{Q_0}(0)\notag\\
  &\quad=
    \bigl[\mathcal V_{Q_\kappa}(\kappa)-\mathcal V_{Q_\kappa}(0)\bigr]
    +\bigl[\mathcal V_{Q_\kappa}(0)-\mathcal V_{Q_0}(0)\bigr]\notag\\
  &\quad\le-2\kappa A_{Q_\kappa}+\kappa^2B_{Q_\kappa}
            +\kappa L_VK_Q\notag\\
  &\quad\le-2\kappa(A_0-\kappa L_AK_Q)+\kappa^2G^2
            +\kappa L_VK_Q\notag\\
  &\quad=-\kappa(2A_0-L_VK_Q)+\kappa^2(G^2+2L_AK_Q)
         =-\kappa M+\kappa^2N.
  \label{eq:p2-dep-own-law-combination}
\end{align}
This proves \eqref{eq:p2-dep-own-law-bound}.
If $M>0$, then $A_0>0$. Were $G=0$, the correction would vanish along the
baseline trajectory and force $A_0=0$. Hence $G>0$ and $N\ge G^2>0$.
For the gains in \eqref{eq:p2-dep-certificate},
$-\kappa M+\kappa^2N=\kappa(-M+\kappa N)<0$, proving
\eqref{eq:p2-dep-generated-order}.
\end{proof}

\paragraph{Meaning of the certificate.}
The constant $K_Q$ converts the correction gain into a bound on expected
trajectory displacement. The constants $L_V$ and $L_A$ translate this
displacement into bounds on the change in squared baseline residual and
residual--correction alignment, respectively.
The first-order margin $M=2A_0-L_VK_Q$ balances the alignment benefit along
the baseline trajectory against the bound on baseline-error change caused by
trajectory displacement. The condition $M>0$ ensures that the alignment
benefit exceeds this displacement cost. The second-order coefficient
$N=G^2+2L_AK_Q$ accounts for correction magnitude and the bound on alignment
change along the shifted trajectory. Choosing $\kappa<M/N$ makes the
certified first-order benefit dominate these second-order costs.
The resulting bound accounts for both the correction and its induced
trajectory shift, establishing lower velocity error under the corrected
sampler's own state law. Appendix~\ref{p2:app:endpoint} next analyzes how
velocity corrections accumulate into endpoint displacement.

\subsection{From velocity error to endpoint distance}
\label{p2:app:endpoint}

The own-law velocity certificate of Theorem~\ref{p2:prop:generated-order}
does not by itself imply endpoint improvement. Endpoint displacement depends on how velocity errors accumulate over the
trajectory. The following proposition identifies the role of temporal
cancellation, which motivates the endpoint gain analysis below.

\begin{proposition}[Temporal cancellation and endpoint displacement]
\label{p2:prop:unbridged-reversal}
There is a one-dimensional instance with state-independent ideal, base and
corrected fields, satisfying
Assumption~\ref{p2:assump:deployment-ode} on one sampler interval, in which the
corrected field has strictly smaller integrated squared velocity error and a
strictly smaller one-sided path-length bound than the base field, while its
terminal $W_2$ distance from the ideal law is strictly larger.
\end{proposition}

\begin{proof}
Reparameterize the sampler interval by $t=1-\sigma\in[0,1]$, so that $t=0$ is
the prescribed source and $t=1$ is the generated endpoint. This substitution
preserves squared velocity error. Let $\mu$ be any square-integrable law on
$\mathbb R$, draw $Z\sim\mu$ and use the same $Z$ as the source of all three
flows. Take the ideal drift to be zero, take
\begin{equation}
  b(t)=
  \begin{cases}
    +1,&0\le t<\tfrac12,\\
    -1,&\tfrac12\le t\le1,
  \end{cases}
  \label{eq:p2-dep-ce-drift}
\end{equation}
as the base drift, fix any gain $\kappa_{\mathrm{ce}}\in(0,\bar\kappa]$ and
define the correction by $g(t)=(\tfrac12-b(t))/\kappa_{\mathrm{ce}}$, so that
the corrected drift is the constant $\tfrac12$. All three drifts are
measurable, bounded and state independent, hence Lipschitz in the state with
constant zero, so they meet Assumption~\ref{p2:assump:deployment-ode}.

Direct integration of $\dot X^\star=0$, $\dot X^0=b$ and
$\dot X^{\mathrm{ce}}=\tfrac12$ from the common source gives
\begin{equation}
  X_t^\star=Z,
  \qquad
  X_t^0=
  \begin{cases}
    Z+t,&0\le t\le\tfrac12,\\
    Z+1-t,&\tfrac12\le t\le1,
  \end{cases}
  \qquad
  X_t^{\mathrm{ce}}=Z+\frac t2,
  \label{eq:p2-dep-ce-flows}
\end{equation}
so that $X_1^\star=X_1^0=Z$ and $X_1^{\mathrm{ce}}=Z+\tfrac12$. The base
drift integrates to zero because its two half-interval contributions cancel.

Because both velocity errors are state independent, evaluating them under any
common state law gives
\begin{equation}
  \int_0^1\E|0-b(t)|^2\,dt=1,
  \qquad
  \int_0^1\E\Bigl|0-\tfrac12\Bigr|^2dt=\frac14,
  \label{eq:p2-dep-ce-velocity}
\end{equation}
and the one-sided path-length bounds from the common-source coupling are
$\int_0^1|b|\,dt=1$ and $\int_0^1\tfrac12\,dt=\tfrac12$. For the endpoints,
\eqref{eq:p2-dep-ce-flows} gives $\pi_1^\star=\pi_1^0=\mu$ while
$\pi_1^{\mathrm{ce}}$ is $\mu$ translated by $\tfrac12$. The translation
coupling gives $W_2(\pi_1^{\mathrm{ce}},\pi_1^\star)\le\tfrac12$, and for any
finite-cost coupling $(Y,Y')$ of the two laws Jensen's inequality gives
$\E|Y-Y'|^2\ge|\E Y-\E Y'|^2=\tfrac14$. Hence
\begin{equation}
  W_2(\pi_1^0,\pi_1^\star)=0,
  \qquad
  W_2(\pi_1^{\mathrm{ce}},\pi_1^\star)=\frac12 .
  \label{eq:p2-dep-ce-endpoint}
\end{equation}
These values prove the stated strict inequalities. The signed drift integral
determines endpoint displacement in this construction, while squared
velocity error measures the drift's energy over time.
\end{proof}

For affine--Gaussian flows with a common linear part, the correction shifts
the endpoint mean by an explicit displacement. We use this structure to
derive an exact endpoint gain identity. The common linear part propagates
the correction through the fundamental solution, and a state-independent
correction preserves covariance by translating each conditional law. The
square-integrability conditions make the averaged endpoint quadratic finite.

\begin{assumption}[Exact affine--Gaussian endpoint specialization]
\label{p2:assump:affine-endpoint}
Conditionally on $\bm{C}=\bm{c}$, the ideal and base sampler ODEs are affine Gaussian
with a common linear part $M(\sigma)\mathbf{x}$, a common conditional endpoint
covariance and an additive drift difference. For $P_{\bm{C}}$-almost every $\bm{c}$,
the common linear coefficient and the ideal additive drift are integrable
in $\sigma$ on $[0,1]$, so the ideal ODE has a unique global solution from
the prescribed Gaussian source. The deployed correction
$\bm{g}_{\bm{\phi}}(\sigma,\bm{c})$ is state independent. Let $\Psi(\sigma,s)$ be the common
fundamental solution of $\dot{\mathbf{x}}=M(\sigma)\mathbf{x}$, let $\bm{m}_0^\star(\bm{c})$ and $\bm{m}_0^0(\bm{c})$
be the ideal and base conditional endpoint means and put
$\bm{E}_0(\bm{c})=\bm{m}_0^\star(\bm{c})-\bm{m}_0^0(\bm{c})$ and
\begin{equation}
  \bm{G}(\bm{c})=\int_1^0\Psi(0,s)\,\bm{g}_{\bm{\phi}}(s,\bm{c})\,\mathrm ds,
  \label{eq:p2-dep-displacement}
\end{equation}
in the sampler orientation. Both $\bm{E}_0$ and $\bm{G}$ belong to
$L^2(P_{\bm{C}};\mathcal X)$.
\end{assumption}

\begin{proposition}[Endpoint displacement and covariance preservation]
\label{p2:prop:endpoint-displacement}
Under Assumptions~\ref{p2:assump:deployment-ode} and
\ref{p2:assump:affine-endpoint}, for $P_{\bm{C}}$-almost every $\bm{c}$ and every finite
gain $\kappa$ in the exact-flow domain, the conditional endpoint law
$\pi_0^\kappa(\cdot\mid \bm{c})$ of the corrected sampler satisfies
\begin{equation}
  \bm{m}_0^\kappa(\bm{c})=\bm{m}_0^0(\bm{c})+\kappa \bm{G}(\bm{c}),
  \qquad
  \operatorname{Cov}\bigl(\mathbf{X}_0^\star(\bm{c})\bigr)
  =\operatorname{Cov}\bigl(\mathbf{X}_0^0(\bm{c})\bigr)
  =\operatorname{Cov}\bigl(\mathbf{X}_0^\kappa(\bm{c})\bigr)
  =\Xi_0(\bm{c}).
  \label{eq:p2-dep-endpoint-mean}
\end{equation}
\end{proposition}

\begin{proof}
Fix a condition $\bm{c}$ and suppress it. Since the correction is state
independent, the base and corrected conditional fields are
$\mathbf{x}\mapsto M(\sigma)\mathbf{x}+\bm{b}_0(\sigma)$ and
$\mathbf{x}\mapsto M(\sigma)\mathbf{x}+\bm{b}_0(\sigma)+\kappa \bm{g}_{\bm{\phi}}(\sigma,\bm{c})$, respectively.
Let $\mathbf{X}_1$ be the prescribed
common Gaussian source. Variation of constants applied directly in the
descending orientation gives
\[
  \mathbf{X}_0^0=\Psi(0,1)\mathbf{X}_1+\int_1^0\Psi(0,s)\bm{b}_0(s)\,\mathrm ds,
\]
\[
  \mathbf{X}_0^\kappa
  =\Psi(0,1)\mathbf{X}_1+\int_1^0\Psi(0,s)\bigl[\bm{b}_0(s)+\kappa \bm{g}_{\bm{\phi}}(s,\bm{c})\bigr]\mathrm ds
  =\mathbf{X}_0^0+\kappa \bm{G}(\bm{c}).
\]
Taking conditional expectations gives the first display of
\eqref{eq:p2-dep-endpoint-mean}. The same identity is a pathwise translation
under the common-source coupling, and $\kappa \bm{G}(\bm{c})$ is deterministic once $\bm{c}$
is fixed, so
$\mathbf{X}_0^\kappa-\bm{m}_0^\kappa(\bm{c})=\mathbf{X}_0^0-\bm{m}_0^0(\bm{c})$ and the corrected and base
conditional covariances agree. The ideal and base endpoint covariances agree by
Assumption~\ref{p2:assump:affine-endpoint}, which gives the common $\Xi_0(\bm{c})$.

The terminal-value problem fixes the integration direction from one to zero.
For example, if $M\equiv0$ and $\bm{g}_{\bm{\phi}}(s,\bm{c})\equiv \bm{g}$, then
\eqref{eq:p2-dep-displacement} gives $\bm{G}(\bm{c})=-\bm{g}$, and the corrected descending
flow ends at $\mathbf{X}_0^0-\kappa \bm{g}$.
\end{proof}

\begin{proposition}[Exact endpoint gain quadratic]
\label{p2:prop:endpoint-identity}
Under Assumption~\ref{p2:assump:affine-endpoint} and
Proposition~\ref{p2:prop:endpoint-displacement}, put
\begin{equation}
  A_{\mathrm{end}}=\E_{\bm{C}}\langle \bm{E}_0,\bm{G}\rangle,
  \qquad
  B_{\mathrm{end}}=\E_{\bm{C}}\|\bm{G}\|^2\ \ge0 .
  \label{eq:p2-dep-endpoint-coefficients}
\end{equation}
Then for every finite gain $\kappa$ in the exact-flow domain,
\begin{equation}
  \E_{\bm{C}}W_2^2\bigl(\pi_0^\kappa,\pi_0^\star\bigr)
  -\E_{\bm{C}}W_2^2\bigl(\pi_0^0,\pi_0^\star\bigr)
  =-2\kappa A_{\mathrm{end}}+\kappa^2B_{\mathrm{end}} .
  \label{eq:p2-dep-endpoint-quadratic}
\end{equation}
\end{proposition}

\begin{proof}
We first compute the transport distance for a common covariance, then
apply the mean displacement and average over conditions.

\phantomsection
\label{p2:lem:common-covariance-w2}
\paragraph{Step 1: Compute the Wasserstein distance for a common covariance.}
For any $\bm{m},\bm{n}\in\mathcal X$ and any positive-semidefinite covariance operator
$\Xi$ on $\mathcal X$,
\begin{equation}
  W_2^2\bigl(\mathcal N(\bm{m},\Xi),\mathcal N(\bm{n},\Xi)\bigr)=\|\bm{m}-\bm{n}\|^2 .
  \label{eq:p2-dep-w2-identity}
\end{equation}
The identity includes singular $\Xi$ and the Dirac case $\Xi=0$.

Let $\bm{Z}\sim\mathcal N(0,\Xi)$, which is well defined for singular $\Xi$, and set
$\mathbf{X}=\bm{m}+\bm{Z}$ and $\mathbf{Y}=\bm{n}+\bm{Z}$. This is a coupling of the two laws with
$\E\|\mathbf{X}-\mathbf{Y}\|^2=\|\bm{m}-\bm{n}\|^2$, which gives the upper bound. For the matching lower
bound, let $(\widetilde{\mathbf{X}},\widetilde{\mathbf{Y}})$ be any coupling. If its quadratic cost
is infinite the bound is immediate, and otherwise Jensen's inequality for
$\bm{z}\mapsto\|\bm{z}\|^2$ gives
$\E\|\widetilde{\mathbf{X}}-\widetilde{\mathbf{Y}}\|^2\ge\|\E(\widetilde{\mathbf{X}}-\widetilde{\mathbf{Y}})\|^2=\|\bm{m}-\bm{n}\|^2$.
Taking the infimum over couplings gives
\eqref{eq:p2-dep-w2-identity} for every positive-semidefinite $\Xi$.

\paragraph{Step 2: Apply the conditional endpoint mean displacement.}
Proposition~\ref{p2:prop:endpoint-displacement} gives
$\bm{m}_0^\star(\bm{c})-\bm{m}_0^\kappa(\bm{c})=\bm{E}_0(\bm{c})-\kappa \bm{G}(\bm{c})$ and a covariance $\Xi_0(\bm{c})$
shared by all three conditional endpoint laws. \hyperref[p2:lem:common-covariance-w2]{Step~1}
therefore gives, for $P_{\bm{C}}$-almost every $\bm{c}$,
\begin{equation}
  W_2^2\bigl(\pi_0^\kappa(\cdot\mid \bm{c}),\pi_0^\star(\cdot\mid \bm{c})\bigr)
  =\|\bm{E}_0(\bm{c})-\kappa \bm{G}(\bm{c})\|^2,
  \qquad
  W_2^2\bigl(\pi_0^0(\cdot\mid \bm{c}),\pi_0^\star(\cdot\mid \bm{c})\bigr)
  =\|\bm{E}_0(\bm{c})\|^2 .
  \label{eq:p2-dep-conditional-w2}
\end{equation}
These identities express the conditional endpoint distances as squared
norms of mean displacements.

\paragraph{Step 3: Average the exact quadratic over conditions.}
We now average over conditions. Assumption~\ref{p2:assump:affine-endpoint}
makes $\bm{E}_0$ and $\bm{G}$ strongly measurable and square integrable, so both maps in
\eqref{eq:p2-dep-conditional-w2} are measurable, and
$\|\bm{E}_0-\kappa \bm{G}\|^2\le2\|\bm{E}_0\|^2+2\kappa^2\|\bm{G}\|^2$ makes both integrable.
Cauchy--Schwarz gives
$\E_{\bm{C}}|\langle \bm{E}_0,\bm{G}\rangle|
 \le(\E_{\bm{C}}\|\bm{E}_0\|^2)^{1/2}(\E_{\bm{C}}\|\bm{G}\|^2)^{1/2}<\infty$.
Every term is therefore a finite real number, and the pointwise identity
$\|\bm{E}_0-\kappa \bm{G}\|^2-\|\bm{E}_0\|^2=-2\kappa\langle \bm{E}_0,\bm{G}\rangle+\kappa^2\|\bm{G}\|^2$
may be integrated. Substituting
\eqref{eq:p2-dep-conditional-w2} gives
\eqref{eq:p2-dep-endpoint-quadratic}.
\end{proof}

\begin{corollary}[Endpoint gain window]
\label{p2:cor:endpoint-window}
Under Proposition~\ref{p2:prop:endpoint-identity}, the corrected sampler has
strictly smaller averaged squared endpoint distance than the base sampler
exactly when $B_{\mathrm{end}}>0$, $A_{\mathrm{end}}>0$ and
\begin{equation}
  0<\kappa<\frac{2A_{\mathrm{end}}}{B_{\mathrm{end}}},
  \label{eq:p2-dep-endpoint-window}
\end{equation}
with equality at both endpoints of that interval and strict degradation beyond
it. The unique optimal gain is $A_{\mathrm{end}}/B_{\mathrm{end}}$. If
$B_{\mathrm{end}}=0$, then $\bm{G}=0$ almost surely and every gain ties. If
$\bm{E}_0=0$ almost surely, then $A_{\mathrm{end}}=0$ and every positive gain with
$B_{\mathrm{end}}>0$ strictly increases the endpoint distance.
\end{corollary}

\begin{proof}
\Eqref{eq:p2-dep-endpoint-quadratic} has the same shape as
\eqref{eq:p2-dep-gain-quadratic} with $A_{\mathrm{end}}$ and
$B_{\mathrm{end}}$ in place of $A_Q$ and $B_Q$, so
Proposition~\ref{p2:prop:gain-boundaries} and
Corollary~\ref{p2:cor:optimal-gain} apply verbatim to the scalar function
$\kappa\mapsto\E_{\bm{C}}W_2^2(\pi_0^\kappa,\pi_0^\star)$. For the last statement, if
$\bm{E}_0=0$ almost surely then $\langle \bm{E}_0,\bm{G}\rangle=0$ almost surely, so
$A_{\mathrm{end}}=0$ and case (ii) of
Proposition~\ref{p2:prop:gain-boundaries} applies.
\hyperref[p2:lem:common-covariance-w2]{Step~1} of the proof of
Proposition~\ref{p2:prop:endpoint-identity} covers every positive-semidefinite
conditional covariance $\Xi_0$, including singular cases.
\end{proof}

\section{Implementation Details}
\label{sec:appendix_implementation_details}

We expand the settings in Sections~\ref{sec:image_generation}--\ref{sec:image_to_video_generation}, covering model configurations, optimization, inference and evaluation.

\subsection{Image Generation}
\label{sec:appendix_image_implementation}

% Local sources: RAEv2/ckpts/stage2/in256-k7-ft-copilot/config.yaml,
% RAEv2/src/stage2/copilot/{model,losses,capture}.py and
% RAEv2/scripts/stage2/pbs_imagenet_k7_ft_copilot_8gpu.pbs.
\paragraph{Model configuration.}
We build on the released ImageNet-256 RAEv2 checkpoint~\citep{singh2026improved} trained for 80 epochs.
The representation autoencoder uses the DINOv3-L encoder with seven-layer feature aggregation and a pretrained ViT-XL decoder.
We keep the encoder, decoder and latent normalization statistics fixed throughout adaptation.
Each image is represented by a $16\times16$ grid of 1,024-dimensional latent vectors.
The base transformer processes these vectors with latent patch size 1, using 28 blocks of width 1,440 with 20 attention heads, followed by a two-block diffusion head of width 2,048 with 16 attention heads.
Its internal guidance branch reads the features after block 8.
The refiner uses 10 blocks of width 768, 12 attention heads and an MLP expansion ratio of 4.
It takes the predicted clean latent, noise-level and class embeddings and hidden states from base blocks 3, 11, 19 and 27, with block indices counted from one.
We initialize the refiner from scratch and zero-initialize its output projection, so its initial velocity correction is zero.

\paragraph{Training.}
The base retains its flow-matching objective, including supervision of the internal guidance branch.
Training noise levels follow the base configuration: we draw a logit-normal variable with underlying Gaussian mean 0 and standard deviation 1, then apply the same factor-8 shift used for sampling below.
The refiner minimizes the mean squared error to the realized velocity residual, together with a cosine-direction term of weight $0.1$.
Its overall loss weight is 1.
We detach the residual target and all base-model inputs to the refiner, so the refiner loss updates only the refiner and the base is updated by its own objective.
Both networks use GMuon with momentum $0.95$, Nesterov updates, zero weight decay and gradient-norm clipping at 1.
The base and refiner learning rates warm up from zero over $0.2$ epochs to $2\times10^{-5}$ and $4\times10^{-4}$, respectively, then decay linearly to $2\times10^{-6}$ and $4\times10^{-5}$ at the end of the 80-epoch adaptation.
We train with a global batch size of 1,024 on eight NVIDIA H200 GPUs, using bfloat16 mixed precision, 64 images per GPU per microbatch and two gradient-accumulation steps.
% Runtime source: RAEv2/ckpts/stage2/in256-k7-ft-copilot/pbs-53054.gaas.log.
The full 80-epoch adaptation comprises 100,080 optimizer steps and takes approximately 42 hours on a single node, including periodic evaluation and checkpointing.
Class-conditioning dropout is $0.1$, and exponential moving average (EMA) weights are maintained separately for both networks with decay $0.9995$.

% Sampling sources: RAEv2/configs/stage2/sampling/imagenet-dinov3l-k7-ft-copilot.yaml,
% RAEv2/src/stage2/transport/{sampler,transport}.py and
% RAEv2/src/stage2/copilot/runner.py.
\paragraph{Sampling and deployment.}
At test time, we use the EMA weights and generate 50,000 images, with 50 images for each of the 1,000 classes.
We start from standard Gaussian latent noise and take 100 Euler steps on the base model's shifted noise schedule.
Specifically, a uniform grid $s_i=1-i/100$, for $i=0,\ldots,100$, is mapped to noise levels $\sigma_i=8s_i/(1+7s_i)$.
The shift factor 8 follows the released configuration, $\sqrt{(1024\cdot16\cdot16)/4096}$.
The model predicts clean latents, which the sampler converts to velocities using a denominator $\max(\sigma,0.05)$ for numerical stability.
We evaluate both without guidance and with RAEv2's internal guidance at scale $1.78$ over $\sigma\in[0.1,1]$.
In both settings, the refiner correction uses gain $\kappa=1$ within $\sigma\in[0.15,0.92]$ and is zero outside this interval.
When guidance is enabled, we add the refiner correction to the velocity after the guidance operation.
We decode the final latent with the frozen decoder and compute all metrics from the resulting 8-bit RGB images.

% Published rows: FD-loss (arXiv:2604.28190v1), Tables 4 and F.6,
% and RAEv2 (arXiv:2605.18324v1), Table 7.
% These are published re-evaluations, not new runs of these seven models.
\paragraph{Baselines.}
The published-baseline rows in Table~\ref{tab:imagenet_generation_comparison} cover SiT-XL/2, DDT-XL, SiT-XL/2-REPA, LightningDiT, REG, REPA-E and RAE-XL.
We use their reported evaluations from FD-loss~\citep{yang2026representationfrechetlossvisual}, whose representation-distance results are also reproduced by RAEv2~\citep{singh2026improved}.
That evaluation generates 50,000 images from official checkpoints with the corresponding official sampling code, balancing the sample count across classes and using method-specific inference settings.
Table~\ref{tab:appendix_image_baseline_protocols} summarizes the sampling budgets and guidance mechanisms.
The six VAE-based methods use 250 sampling steps with classifier-free guidance (CFG), while RAE-XL uses 50 Euler steps and AutoGuidance from a separately trained weaker model.
These published results therefore retain their original sampling budgets and guidance configurations.

\begin{table}[!htb]
    \centering
    \small
    \caption{Evaluation configurations for all seven published image-generation baselines in Table~\ref{tab:imagenet_generation_comparison}. Steps count sampler updates. CFG combines conditional and unconditional predictions, while AutoGuidance combines a main model with a weaker model.}
    \label{tab:appendix_image_baseline_protocols}
    \begin{tabular}{llcl}
        \toprule
        Method & Image representation & Steps & Guidance \\
        \midrule
        SiT-XL/2~\citep{ma2024sit} & SD-VAE & 250 & CFG \\
        DDT-XL~\citep{wang2026ddt} & SD-VAE & 250 & CFG \\
        SiT-XL/2-REPA~\citep{yu2025repa} & SD-VAE & 250 & CFG \\
        LightningDiT~\citep{yao2025vavae} & VA-VAE & 250 & CFG \\
        REG~\citep{wu2026representation} & SD-VAE & 250 & CFG \\
        REPA-E~\citep{leng2025repae} & Jointly tuned VAE & 250 & CFG \\
        RAE-XL~\citep{zheng2026rae} & DINOv2-B RAE & 50 & AutoGuidance \\
        \bottomrule
    \end{tabular}
\end{table}

% Local baseline sources: table source-row comments, the joint-training run,
% RAEv2/configs/stage2/training/imagenet-dinov3l-k7-self-flow-ft.yaml,
% RAEv2/ckpts/stage2/in256-k7-self-flow-ft/offline_eval/configs/65790.gaas/,
% and RAEv2/scripts/eval/imagenet_self_flow_cfg_metrics_manifest.json.

The RAEv2 models with different training epochs use the same frozen autoencoder, EMA weights, 100-step sampler, class-balanced sample count and guidance configuration as \method.
This comparison isolates the effect of deploying the refiner at a fixed base checkpoint.
For SelfFlow~\citep{chefer2026selfflow}, we evaluate a local adaptation of its self-supervised flow-matching objective to the same frozen DINOv3-L representation autoencoder.
This baseline uses a 30-block transformer of width 1,280 with 20 attention heads, a dual-noise mixing ratio of $0.25$ and representation alignment of weight $0.8$ between student block 9 and EMA-target block 21.
We use EMA weights, the same 100-step shifted Euler grid and 50,000 class-balanced images.
The unguided runs disable guidance, and the guided runs use CFG at scale $1.5$ over $\sigma\in[0,0.7]$.
The SelfFlow rows report this local RAEv2-latent adaptation rather than the original paper's ImageNet configuration.

% Metric sources: RAEv2/.venv/lib/python3.10/site-packages/fd_evaluator/
% {api.py,catalogue.yaml,extractors.py}, plus
% RAEv2/scripts/eval/compute_imagenet_fdr.py and its saved protocol metadata.
\paragraph{Inception-based metrics.}
We report Inception Score (IS, higher is better)~\citep{salimans2016improvedtechniquestraininggans} and Fr\'echet Inception Distance (FID, lower is better)~\citep{heusel2018ganstrainedtimescaleupdate}.
IS measures the confidence and marginal diversity of Inception-v3 class predictions on generated images.
For a generated image $\mathbf{x}$, let $p(y\mid\mathbf{x})$ be the classifier distribution over class labels $y$, and let $p(y)$ be its marginal over generated samples.
Then $\mathrm{IS}=\exp\!\bigl(\mathbb{E}_{\mathbf{x}}\mathrm{KL}(p(y\mid\mathbf{x})\Vert p(y))\bigr)$, where $\mathrm{KL}$ denotes Kullback--Leibler divergence.
Our evaluator shuffles the generated samples before splitting them into ten equal groups and averages the resulting scores.
FID compares the mean and covariance of generated and real images in the 2,048-dimensional Inception-v3 feature space.
For our local evaluations, both metrics use the Inception-v3-compatible extractor in the RAEv2 evaluation package, with its released ImageNet-256 reference statistics for FID.

\paragraph{Representation-based metrics.}
We also report Representation Fr\'echet Distance ($\mathrm{FD}_r$, lower is better) in ConvNeXt, DINOv2 and CLIP feature spaces.
These metrics compare generated distributions through complementary supervised, self-supervised and image--text representations.
Following the normalized-distance convention of FD-loss~\citep{yang2026representationfrechetlossvisual}, each table entry is a ratio rather than a raw Fr\'echet distance.
For a frozen feature extractor $\bm{\phi}$, let $(\bm{\mu}_G,\Sigma_G)$ and $(\bm{\mu}_R,\Sigma_R)$ denote the feature mean vector and covariance operator of generated images $G$ and the real training reference $R$.
We compute
\begin{equation}
\label{eq:appendix_image_fd}
\mathrm{FD}_{\bm{\phi}}(G,R)
=\lVert\bm{\mu}_G-\bm{\mu}_R\rVert_2^2
+\operatorname{Tr}\!\left(\Sigma_G+\Sigma_R
-2\bigl(\Sigma_R^{1/2}\Sigma_G\Sigma_R^{1/2}\bigr)^{1/2}\right),
\end{equation}
and normalize it by the distance from real ImageNet validation images $V$ to the same training reference,
\begin{equation}
\label{eq:appendix_image_fdr}
\mathrm{FD}_{r,\bm{\phi}}(G)
=\frac{\mathrm{FD}_{\bm{\phi}}(G,R)}{\mathrm{FD}_{\bm{\phi}}(V,R)}.
\end{equation}
The ConvNeXt column uses spatially averaged ConvNeXt V2-Base features~\citep{woo2023convnextv2codesigningscaling} at input resolution $224\times224$.
The DINOv2 and CLIP columns use the class-token features of DINOv2 ViT-L/14~\citep{oquabdinov2} and OpenAI CLIP ViT-L/14~\citep{radford2021learning}, respectively, both evaluated at $256\times256$ resolution.
All three feature vectors have dimension 1,024.
We retain the released extractor-specific normalization and reference statistics, using bicubic resizing with antialiasing when resizing is needed.
The normalization constants in \eqref{eq:appendix_image_fdr} are $56.87$, $14.19$ and $5.60$ for ConvNeXt, DINOv2 and CLIP, respectively.
A value of 1 corresponds to the real validation-set reference in that feature space.
We report the three distances individually to expose differences between representations.

\subsection{Text-to-Video Generation}
\label{sec:appendix_t2v_implementation}

% Model sources: DiffSynth-Studio/examples/wanvideo/model_training/train_joint_copilot.py,
% video_copilot/model_v3.py and paper/backup/t2v_parameter_audit_20260920/parameter_counts.json.
\paragraph{Model configuration.}
We build on Wan2.1-T2V-1.3B~\citep{wan2025wan}, keeping its pretrained video variational autoencoder (VAE) and UMT5-XXL text encoder fixed while updating the full base transformer and the refiner.
The base transformer has 30 blocks of width 1,536 with 12 attention heads.
The VAE maps an 81-frame video at $480\times832$ resolution to a latent tensor with 16 channels, 21 temporal positions and spatial size $60\times104$.
Both networks use latent patches of size $1\times2\times2$, giving 32,760 tokens in the same spatiotemporal order.
The refiner has 10 blocks of width 1,024, 16 attention heads and an MLP expansion ratio of 4.
Its blocks alternate between self-attention with three-dimensional rotary position embeddings and cross-attention to text and base-model features.
The base model's noise-level embedding modulates each refiner block through adaptive layer normalization.

We read hidden states after base blocks 4, 12, 20 and 30, corresponding to zero-based code indices 3, 11, 19 and 29.
At each token position, we concatenate the four 1,536-dimensional states and fuse them with a $6144\rightarrow1024\rightarrow1024$ MLP with a GELU activation.
We project the 4,096-dimensional text embeddings to the same width, add learned segment embeddings to distinguish text from hidden-state features and concatenate the two token sequences before layer normalization.
The refiner processes the one-step clean prediction and maps its output back to a velocity residual of the original latent shape.
It contains 146.58 million parameters, approximately $10.3\%$ of the 1.419-billion-parameter base transformer.
The combined transformer and refiner contain 1.566 billion parameters, excluding the frozen VAE and text encoder.

% Data sources: datasets/WISA-80K/data/wisa-80k-q4.0-motion0.1.jsonl,
% datasets/OpenVidHD-10pct/openvidhd_10pct_curated20k.jsonl and
% DiffSynth-Studio/examples/wanvideo/model_training/wisa_dataset.py.
\paragraph{Data and preprocessing.}
The final training mixture contains 46,572 WISA-80K clips and 20,000 curated OpenVidHD clips, for 66,572 video--text pairs.
The WISA subset retains clips with visual-quality scores of at least 4.0 and motion scores of at least 0.1.
For WISA, we concatenate the video caption with its physical-law annotation, following the dataset's level-B conditioning.
For OpenVidHD~\citep{nan2025openvid}, we use the supplied caption.
We shuffle the concatenated datasets without additional source reweighting.
Each training example contains 81 frames sampled uniformly across the complete clip, preserving the temporal extent of the depicted event.
Short clips reuse sampled frame indices to obtain the same length.
We resize frames while preserving their aspect ratio and center-crop them to $480\times832$ before VAE encoding.

% Checkpoint lineage: outputs/wan21_t2v_1_3b_wisa_joint_copilot_v3/step-5000
% -> outputs/wan21_t2v_1_3b_wisa_joint_copilot_20260719_224914/step-1500
% -> outputs/wan21_t2v_1_3b_wisa_joint_copilot_20260720_143102/step-5700
% -> outputs/wan21_t2v_1_3b_wisa_joint_copilot_20260724_092738/step-20805.
% The last three stages' args and hardware are in their saved wandb-metadata.json.
\paragraph{Training.}
We jointly train the base model and refiner for five epochs on the processed 66,572-clip dataset using AdamW, with learning rates of $5\times10^{-6}$ and $10^{-4}$, respectively, momentum coefficients $(0.9,0.999)$ and weight decay $0.01$.
The learning rates start at one-third of these values for five scheduler steps and then remain at the nominal rates.
Training uses four NVIDIA H200 GPUs with a global batch size of 16, bfloat16 mixed precision and gradient checkpointing, taking approximately 201 hours including checkpointing.
The base uses the scheduler-weighted flow-matching loss, while the refiner minimizes residual mean squared error plus a cosine-direction loss weighted by $0.1$, with unit overall refiner-loss weight.
All three benchmarks use the resulting paired weights.

% Inference sources: run_meta_shard_*.json under outputs/vbench_t2v/joint_copilot_v3_full,
% outputs/videophy2_t2v/joint_copilot_v3_videophy2_paper590_5candidates_cfg5_pos_s02,
% and outputs/finephyeval_t2v/joint_copilot_v3_finephyeval_5candidates.
% Implementation: DiffSynth-Studio/examples/wanvideo/model_inference/
% Wan2.1-T2V-1.3B-OmniWorld-CopilotInference.py and diffsynth/diffusion/flow_match.py.
\paragraph{Sampling and deployment.}
We generate 81-frame videos at $480\times832$ resolution and save them at 16 frames per second.
Sampling starts from standard Gaussian latent noise and uses 50 Euler updates with noise shift 5 and classifier-free guidance (CFG) scale 5.
For $s_i=1-i/50$, with $i=0,\ldots,50$, the noise grid is $\sigma_i=5s_i/(1+4s_i)$.
We use the fixed Wan negative prompt and apply the refiner to the positive text branch with gain $\kappa=0.2$ when $\sigma\in[0.15,0.92]$.
Let $\bm{v}_{+}$ and $\bm{v}_{-}$ denote the base velocities conditioned on the benchmark prompt and negative prompt, respectively, and let $\hat{\bm{r}}_{+}$ denote the positive-branch residual prediction.
At active steps, the guided velocity is
\begin{equation}
\label{eq:appendix_t2v_cfg}
\bm{v}_{\mathrm{sample}}
=\bm{v}_{-}+5\bigl(\bm{v}_{+}+0.2\hat{\bm{r}}_{+}-\bm{v}_{-}\bigr)
=\bm{v}_{-}+5(\bm{v}_{+}-\bm{v}_{-})+\hat{\bm{r}}_{+}.
\end{equation}
Thus the branch-level gain yields a unit-strength residual contribution to the guided velocity.
We recompute the clean prediction and hidden states from the current latent at every active step.
The base runs in bfloat16 and the refiner in float32, and the frozen VAE decodes the final latent to RGB frames.

\paragraph{VBench.}
We use the released VBench prompt suite and evaluators~\citep{huang2024vbench} to measure 16 dimensions of video quality and text alignment.
Seven dimensions measure visual quality and temporal behavior, and nine measure semantic consistency.
Tables~\ref{tab:vbench_part1} and~\ref{tab:vbench_part2} display dimension scores on a 0--100 scale.
Quality Score and Semantic Score use the benchmark's dimension-specific normalization and weights, and Overall Score combines them with weights $0.8$ and $0.2$, respectively.
Temporal Flickering measures resistance to frame-to-frame flicker, so a larger score indicates better temporal stability.
Dynamic Degree measures the presence of motion through optical flow and complements Motion Smoothness, which evaluates temporal regularity.

\paragraph{VideoPhy2.}
We evaluate the 590-prompt test set, including its 180-prompt HARD subset, with the released VideoPhy2 automatic evaluator~\citep{bansal2026videophy}.
Generation uses the benchmark's expanded captions, while scoring uses the original action captions.
The evaluator assigns ratings from 1 to 5 for physical commonsense (PC) and semantic adherence (SA).
The PC and SA columns report the percentage of ratings at least 4, and Joint reports the percentage satisfying both thresholds.
Joint therefore measures whether the requested event is depicted with plausible physics, combining two requirements that the separate scores assess individually.

% FinePhyEval evaluator: pqsg/scripts/evaluate_finephyeval_candidates.py,
% propagated_category_metrics(), and the saved evaluation summary.
\paragraph{FinePhyEval.}
We evaluate all 65 FinePhyEval prompts using their released question graphs~\citep{pothiraj2026physics}.
These graphs decompose each event into object existence, action execution and physical behavior.
Our local evaluation uses GPT-5.5 to answer the fixed questions from sampled video frames in the two-step PQSG question-answering procedure.
A failed prerequisite propagates failure to dependent questions, so a physics judgment also requires the relevant objects and actions to be present.
Object, Action and Physics are the fractions of satisfied questions in each category after this propagation, with all questions retained in their respective denominators.
Overall is the corresponding fraction across all questions, weighting categories by their question counts.
All four scores lie between 0 and 1.

\paragraph{Comparison methods.}
We compare with the released Wan2.1 models and methods that adapt video generators for motion or physical consistency, including VideoJAM~\citep{chefer2025videojam}, DreamWorld~\citep{tan2026dreamworld}, VPT~\citep{zheng2026enhancing}, VideoREPA~\citep{zhang2025videorepa} and PhysRAG~\citep{cheng2026physrag}, as available for each benchmark.
The model-size suffixes identify the base-generator family, while the refiner parameter count is reported above.
The tables combine published benchmark results with local evaluations of released implementations, retaining method-specific generation settings.

\subsection{Image-to-Video Generation}
\label{sec:appendix_i2v_implementation}

% Model and training sources: model/Wan2.2-TI2V-5B/config.json,
% DiffSynth-Studio/examples/wanvideo/model_training/train_joint_copilot.py,
% DiffSynth-Studio/diffsynth/diffusion/loss.py and video_copilot/model_v3.py.
% Evaluated checkpoint: outputs/wan22_ti2v_5b_wisa_joint_copilot_v3_
% 4gpu_bs1_ga4_resume_s4200_20260814_132628/{step-20805.safetensors,copilot_final.pt}.
% Saved wandb-metadata.json and state-20805/train_state.json record the arguments.
\paragraph{Model configuration.}
We build on Wan2.2-TI2V-5B~\citep{wan2025wan}, whose base transformer has 30 blocks of width 3,072 with 24 attention heads.
As in Appendix~\ref{sec:appendix_t2v_implementation}, we freeze the video VAE and text encoder and update the full base transformer together with the refiner.
The Wan2.2 VAE encodes an 81-frame training video at $704\times1248$ resolution into 48 latent channels with 21 temporal positions and spatial size $30\times52$.
The refiner uses 13 blocks of width 2,048 with 32 attention heads.
Its latent patch size, MLP expansion ratio, alternating attention blocks, hidden-state selection and conditioning design are the same as in text-to-video generation.
To match the larger base features, its fusion MLP maps the concatenated hidden states through $12288\rightarrow2048\rightarrow2048$, and its noise-level projection accepts 3,072-dimensional embeddings.
The refiner predicts a 48-channel velocity residual and is initialized from scratch with a zero-initialized output projection.

\paragraph{Training.}
The training mixture, caption construction, frame sampling and spatial preprocessing are the same as in Appendix~\ref{sec:appendix_t2v_implementation}.
We also use the same five-epoch schedule, global batch size, AdamW configuration, learning rates, loss weights, mixed precision and gradient checkpointing.
For each training example, we enable image conditioning with probability $0.4$ by encoding the first sampled frame with the frozen VAE.
For these examples, we replace the first temporal latent with this clean image latent and assign its tokens noise level zero in the base transformer.
Both the base flow-matching loss and the refiner residual loss are computed over the remaining temporal latents.
The other examples use text conditioning alone.
We detach the residual target and base-model inputs to the refiner, so each network is updated by its own objective.
Training begins on 16 NVIDIA H200 GPUs with one clip per GPU and resumes after 4,200 optimizer steps on four H200 GPUs with four gradient-accumulation steps, preserving the global batch size of 16.
We evaluate the paired base and refiner weights at the end of the fifth epoch.

% Inference source: outputs/physics_iq_verified_ti2v/
% wan22_ti2v_5b_copilot_v3_step20805_physiq_c5_704x1248/raw_videos/run_meta_shard_0.json.
% Implementation: DiffSynth-Studio/examples/wanvideo/model_inference/
% Wan2.1-T2V-1.3B-OmniWorld-CopilotInference.py.
\paragraph{Sampling and image conditioning.}
We generate at $704\times1248$ resolution, using the same video length, frame rate, Euler sampler, noise grid, CFG scale, negative prompt, refiner gain and active noise interval as in Appendix~\ref{sec:appendix_t2v_implementation}.
The positive-branch correction follows \eqref{eq:appendix_t2v_cfg}, and the base and refiner use the same inference precisions as in text-to-video generation.
For each benchmark example, we supply the released image at the start of the continuation and its corresponding text prompt, including the benchmark's fixed-camera instructions and action description.
We resize the image to the generation resolution, encode it with the frozen VAE and keep its latent fixed at the first temporal position.
We restore this condition in the refiner's clean prediction and after every sampling update.

% Evaluation sources in code/physics-IQ-benchmark: physiq/ti2v_pipeline.py,
% physiq/calculate_and_write_metrics_to_csv.py and physiq/calculate_iq_score_stable.py.
\paragraph{Physics-IQ Verified.}
The image-to-video evaluation covers 66 physical scenarios from three camera viewpoints, yielding 198 image--text pairs.
We retain the first 80 frames of each generated video for the benchmark's five-second evaluation at 16 frames per second.
The released evaluator compares generated continuations with recorded reference videos at $540\times960$ resolution, using motion masks and RGB frames without an LLM judge.
Spatial measures the intersection-over-union of regions that exhibit motion anywhere in the clip.
Spatio-temporal measures motion-mask overlap at corresponding frames, while Weighted Spatial compares how frequently each pixel belongs to a moving region.
MSE measures the mean squared difference between generated and reference RGB frames.

For each scenario and viewpoint, the evaluator normalizes these measurements by the corresponding comparison between two real recordings of the same event.
The three overlap metrics are divided by this real-to-real reference, while MSE uses the inverse ratio so that higher is better.
Each normalized component is clipped to $[0,1]$, and Verified Score averages the four components and then the scenario--viewpoint pairs.
Table~\ref{tab:physics_iq_verified} reports all scores on a 0--100 scale. All rows in Table~\ref{tab:physics_iq_verified} use the benchmark's setting without LLM support.

\citet{yuan2026inference} describe two inference-time uses of the WMReward signal: gradient guidance, which steers each denoising trajectory with the reward gradient, and Best-of-N selection, which draws multiple guided samples and keeps the one with the highest reward.
We enable only reward guidance and generate a single sample per condition, without Best-of-N selection.
Because \citet{yuan2026inference} attribute most of their Physics-IQ improvement to the selection step, the WMReward row is expected to score below their guidance-plus-selection results.
This choice keeps the comparison at a matched sampling budget. The other baselines are run using their default configurations as released.

% Keep the shared appendix tables on pages with prose.
% Their placement in other sections is unaffected.
\let\appendixtable\table
\renewcommand{\table}[1][]{\appendixtable[!htb]}
\renewenvironment{table*}[1][]{\begin{table}[!htb]}{\end{table}}

% Modest float spacing lets nearby prose fill the remaining page area.
\setlength{\textfloatsep}{10pt plus 2pt minus 2pt}
\setlength{\floatsep}{10pt plus 2pt minus 2pt}
\setlength{\intextsep}{10pt plus 2pt minus 2pt}
\setlength{\abovecaptionskip}{6pt}

\section{More Results}
\label{sec:appendix_more_results}

\subsection{More Results on ImageNet Generation}
\label{sec:appendix_imagenet_results}

\subsubsection{Quantitative Results}
Table~\ref{tab:imagenet_generation_comparison} extends the comparison in
Table~\ref{tab:imagenet_generation_comparison_compact} with additional methods,
training budgets, base-only evaluations and FID scores.

\begin{table*}[htbp]
    \centering
    \caption{
        Comparison on ImageNet-256~\citep{deng2009imagenet} generation.
        For published baselines, the FID column reports guided FID (gFID) as in the original papers.
        Models are evaluated with and without guidance as indicated.
        For our method, $80+40$ and $80+80$ indicate that we start from the 80-epoch RAEv2 checkpoint and train our refiner for an additional 40 and 80 epochs, respectively.
        ``Ours (base only)'' evaluates the corresponding base model without the refiner at inference and reports the total base-model training epochs.
        Representation Fr\'echet Distance ($\mathrm{FD}_r$) is reported in the feature spaces of ConvNeXt~\citep{liu2022convnet}, DINOv2~\citep{oquabdinov2} and CLIP~\citep{radford2021learning}.
    }
    \label{tab:imagenet_generation_comparison}
    \resizebox{\textwidth}{!}{
    \begin{tabular}{lccccccc}
        \toprule
        \multirow{2}{*}{Method}
        & \multirow{2}{*}{Epochs}
        & \multirow{2}{*}{Guidance}
        & \multirow{2}{*}{IS $\uparrow$}
        & \multirow{2}{*}{FID $\downarrow$}
        & \multicolumn{3}{c}{Representation Fr\'echet Distance ($\mathrm{FD}_r$) $\downarrow$} \\
        \cmidrule(lr){6-8}
        & & & &
        & ConvNeXt & DINOv2 & CLIP \\
        \midrule

        \multicolumn{8}{l}{\textit{Published baselines}} \\

        SiT-XL/2~\citep{ma2024sit}
        & 800 & Yes & 256.7 & 2.12
        & 2.02 & 7.89 & 17.69 \\

        DDT-XL~\citep{wang2026ddt}
        & 800 & Yes & \fst{309.3} & 1.26
        & 1.02 & 4.26 & 13.86 \\

        SiT-XL/2-REPA~\citep{yu2025repa}
        & 800 & Yes & \snd{306.1} & 1.42
        & 1.22 & 4.27 & 12.65 \\

        LightningDiT~\citep{yao2025vavae}
        & 800 & Yes & 294.3 & 1.42
        & 1.09 & 3.76 & 10.21 \\

        REG~\citep{wu2026representation}
        & 800 & Yes & \trd{302.9} & 1.54
        & 1.14 & 3.45 & 10.86 \\

        REPA-E~\citep{leng2025repae}
        & 800 & Yes & 298.3 & 1.17
        & 1.28 & 2.44 & 6.28 \\

        RAE-XL~\citep{zheng2026rae}
        & 800 & Yes & 261.0 & 1.16
        & 1.79 & 2.11 & 7.87 \\

        \midrule
        \multicolumn{8}{l}{\textit{Evaluations without guidance}} \\

        % Source row: noguid-base_ema (80ep).
        RAEv2~\citep{singh2026improved}
        & 80 & No & 225.21 & 1.6038
        & 1.4127 & 2.3965 & 7.3809 \\

        % Source row: noguid-ep120_ema.
        RAEv2~\citep{singh2026improved}
        & 120 & No & 230.53 & 1.5317
        & 1.3050 & 2.2879 & 7.2028 \\

        % Source row: self-flow-ft-noguid-ep120_ema.
        SelfFlow~\citep{chefer2026selfflow}
        & 120 & No & 144.58 & 4.8232
        & 4.8635 & 7.1880 & 12.2924 \\

        % Source row: noguid-ep80+40_ema_copilot.
        \textbf{Ours}
        & 80+40 & No & 231.96 & 1.5582
        & 1.2818 & 2.2635 & 7.2213 \\

        % Source row: noguid-ep160_ema.
        RAEv2~\citep{singh2026improved}
        & 160 & No & 233.94 & 1.5070
        & 1.2753 & 2.2226 & 7.1804 \\

        % Source row: self-flow-ft-noguid-ep160_ema.
        SelfFlow~\citep{chefer2026selfflow}
        & 160 & No & 146.37 & 4.7480
        & 4.8110 & 7.0750 & 12.1329 \\

        % Source row: noguid-ep80+80_ema_copilot.
        \textbf{Ours}
        & 80+80 & No & 234.02 & 1.5367
        & 1.2396 & 2.2051 & 7.1445 \\

        \midrule
        \multicolumn{8}{l}{\textit{Evaluations with guidance}} \\

        RAEv2~\citep{singh2026improved}
        & 80 & Yes & 260.44 & 1.0567
        & 0.7745 & 1.1020 & 5.2134 \\

        % Source row: in256-k7-ft-copilot-ep40_ema.
        RAEv2~\citep{singh2026improved}
        & 120 & Yes & 265.54 & \fst{1.0235}
        & \trd{0.7451} & \snd{1.0732} & 4.8667 \\

        % Source: RAEv2/results/stage2/offline_eval/self_flow_ft_cfg1.5_metrics/imagenet_metrics.csv.
        % CFG scale 1.5, interval [0.0, 0.7]. NPZ ep40/ep80 correspond to 120/160 total epochs.
        % Source row: self-flow-ft-cfg1.5-ep120_ema.
        SelfFlow~\citep{chefer2026selfflow}
        & 120 & Yes & 145.03 & 4.7948
        & 4.8228 & 7.0984 & 12.2674 \\

        % Source row: in256-k7-ft-copilot-ep40_ema_copilot.
        \textbf{Ours}
        & 80+40 & Yes & 264.95 & \trd{1.0446}
        & 0.7475 & 1.0869 & \trd{4.8648} \\

        % Source row: in256-k7-ft-copilot-ep80_ema.
        RAEv2~\citep{singh2026improved}
        & 160 & Yes & 264.97 & \snd{1.0337}
        & \snd{0.7363} & \trd{1.0797} & \snd{4.8549} \\

        % Source row: self-flow-ft-cfg1.5-ep160_ema.
        SelfFlow~\citep{chefer2026selfflow}
        & 160 & Yes & 146.79 & 4.6462
        & 4.7155 & 6.9849 & 12.1602 \\

        % Source row: in256-k7-ft-copilot-ep80_ema_copilot.
        \textbf{Ours}
        & 80+80 & Yes & 266.60 & 1.0593
        & \fst{0.7258} & \fst{1.0339} & \fst{4.8415} \\

        \bottomrule
    \end{tabular}
    }
\end{table*}

\subsubsection{Qualitative Results}
\label{sec:appendix_imagenet_qualitative}

Figures~\ref{fig:imagenet_generation_qualitative_1}
and~\ref{fig:imagenet_generation_qualitative_2} compare selected samples from
Self-Flow, RAEv2 and \method on class-conditional ImageNet generation at
$256\times256$ resolution across eight categories.
In these examples, \method produces clearer object
structures and more recognizable category-specific details, including the lattice
of shoji screens, the heads and handles of hatchets and the gills of axolotls.

% Source: RAEv2/results/stage2/offline_eval/
% qualitative_8categories_no_tricycle_20260916_172158/.
% Training budgets confirmed by the user: all base DiTs 160 epochs, our refiner 80.
% Only the two full-height image-to-video panels below use float-only pages.
% Other figures may move with the text to fill the available space.
\begin{figure}[!htb]
    \centering
    \includegraphics[width=\textwidth]{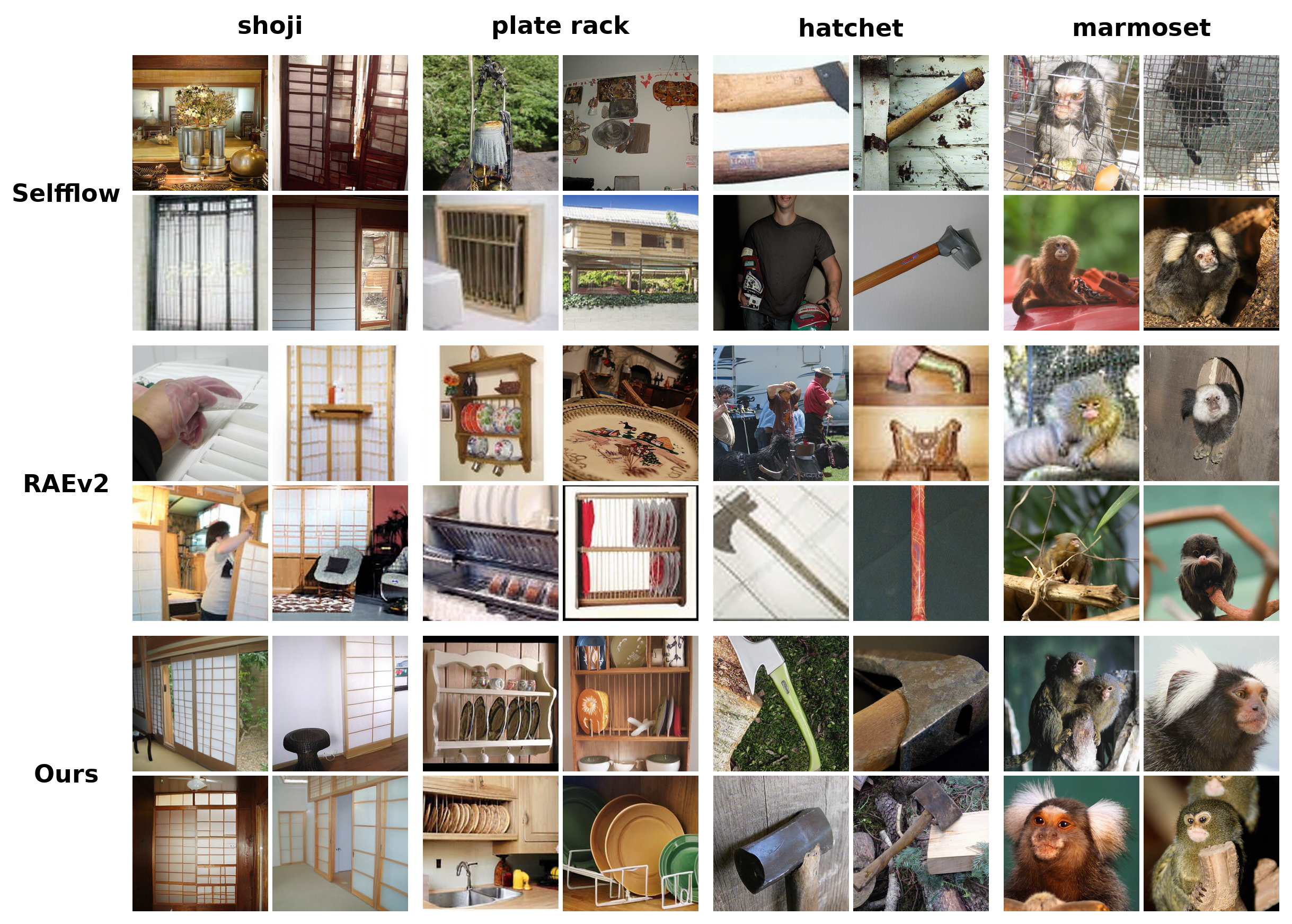}
    \caption{
        Qualitative results on ImageNet at $256\times256$ resolution.
        Columns show shoji, plate rack, hatchet and marmoset.
        Row groups show Self-Flow (Selfflow), RAEv2 and \method (Ours).
        Each method--category block contains four samples in a $2\times2$ grid,
        giving 48 images in total.
        All three base DiTs are trained for 160 epochs.
        The \method refiner is trained for 80 epochs.
    }
    \label{fig:imagenet_generation_qualitative_1}
\end{figure}

\begin{figure}[!htb]
    \centering
    \includegraphics[width=\textwidth]{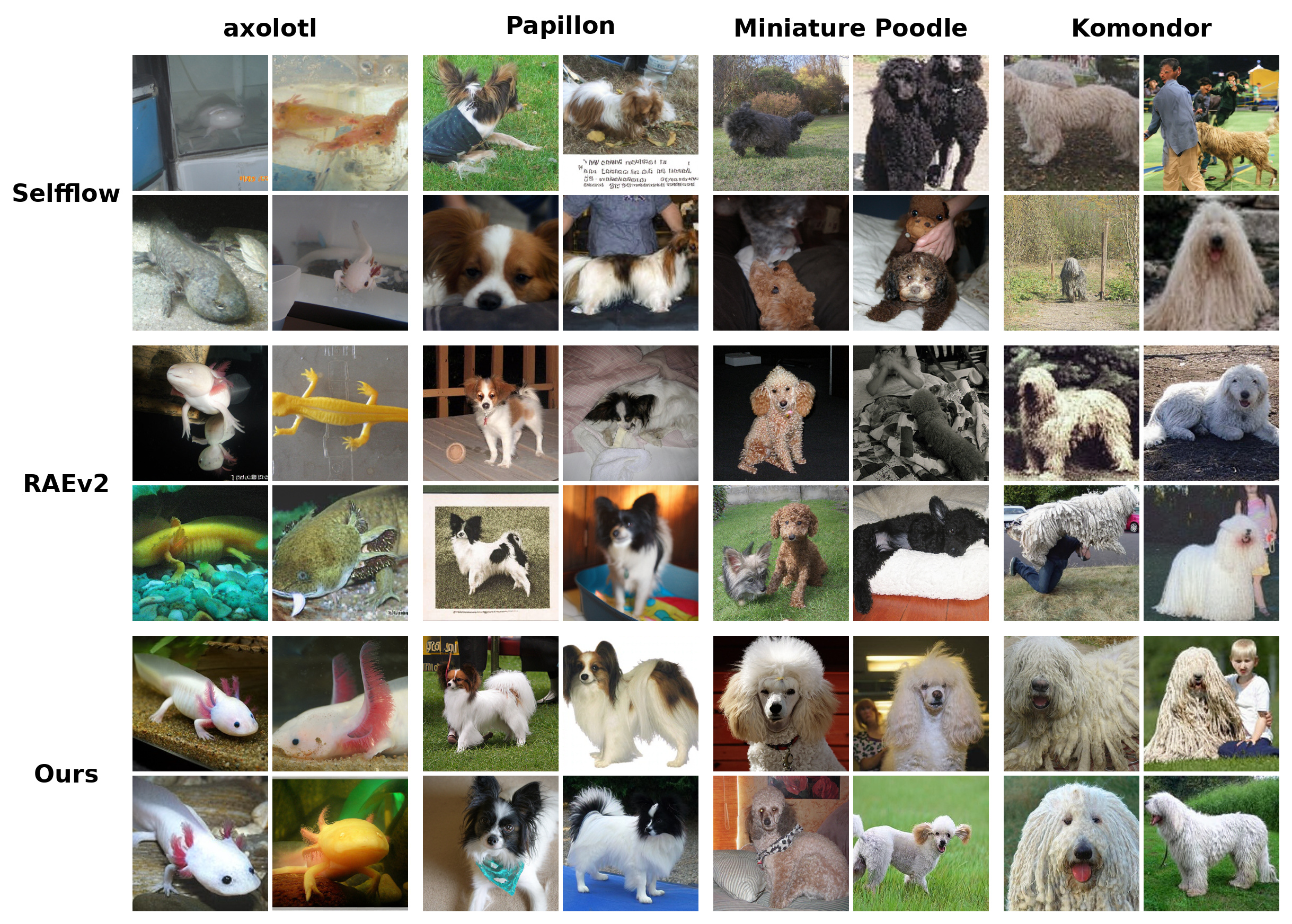}
    \caption{
        Qualitative results on ImageNet at $256\times256$ resolution.
        Columns show axolotl, Papillon, Miniature Poodle and Komondor.
        Row groups show Self-Flow (Selfflow), RAEv2 and \method (Ours).
        Each method--category block contains four samples in a $2\times2$ grid,
        giving 48 images in total.
        All three base DiTs are trained for 160 epochs.
        The \method refiner is trained for 80 epochs.
    }
    \label{fig:imagenet_generation_qualitative_2}
\end{figure}

At the same 160-epoch base-model training budget, \method improves all three $\mathrm{FD}_r$ metrics over RAEv2 both with and without guidance, consistent with the clearer shoji lattices, hatchet geometry and axolotl gills in Figures~\ref{fig:imagenet_generation_qualitative_1} and~\ref{fig:imagenet_generation_qualitative_2}.
With guidance, IS increases from $264.97$ to $266.60$ and DINOv2 $\mathrm{FD}_r$ decreases from $1.0797$ to $1.0339$, although FID worsens from $1.0337$ to $1.0593$ (Table~\ref{tab:imagenet_generation_comparison}).
The base model exhibits a similar discrepancy: continuing RAEv2 fine-tuning from 120 to 160 total epochs raises guided FID from $1.0235$ to $1.0337$, while ConvNeXt and CLIP $\mathrm{FD}_r$ continue to decrease.
These discrepancies suggest practical saturation of FID as a perceptual ranking signal in this regime near 1, where small score differences are an unreliable standalone indicator of generation quality.
FID compares only the means and covariances of Inception-v3 features, so improvements in object structure and fine detail need not yield a lower score.
\citet{Jayasumana_2024_CVPR} document disagreements between FID and human judgments and show that FID can fail to track progressive visual improvements, while \citet{singh2026improved} observe that incremental gFID gains provide little practical signal for strong RAEv2 models.
The consistent improvements across complementary representation spaces, together with the visual examples, support the benefit of learned residual corrections beyond what FID alone captures.

\subsection{More Results on Text-to-Video Generation}
\label{sec:appendix_t2v_results}

\subsubsection{Quantitative Results}

Tables~\ref{tab:vbench_part1} and~\ref{tab:vbench_part2} report the full VBench
results for the methods in Table~\ref{tab:vbench_compact}, covering all 16
evaluation dimensions and the quality, semantic and overall scores.

\begin{table*}[htbp]
    \centering
    \caption{VBench results (Part 1/2): consistency, motion and object-level metrics. All scores are reported on a 0--100 scale.}
    \label{tab:vbench_part1}
    \resizebox{\textwidth}{!}{%
        \begin{tabular}{@{}l*{10}{c}@{}}
            \toprule
            Method & \shortstack{Subject\\Consistency} & \shortstack{Background\\Consistency} & \shortstack{Temporal\\Flickering} & \shortstack{Motion\\Smoothness} & \shortstack{Dynamic\\Degree} & \shortstack{Aesthetic\\Quality} & \shortstack{Imaging\\Quality} & \shortstack{Object\\Class} & \shortstack{Multiple\\Objects} & \shortstack{Human\\Action} \\
            \midrule
            VideoREPA-5B~\citep{zhang2025videorepa} & \trd{95.47} & 95.10 & 98.23 & \snd{98.54} & 47.22 & 57.19 & 59.80 & 67.56 & 40.32 & \fst{83.00} \\
            PhysRAG-5B~\citep{cheng2026physrag} & \snd{96.31} & 95.66 & 98.71 & 96.93 & \fst{81.94} & 56.80 & \trd{65.56} & \trd{80.50} & \fst{75.00} & 63.00 \\
            Wan 2.1-1.3B~\citep{wan2025wan} & 91.83 & 94.71 & \snd{99.17} & 96.51 & 65.00 & 57.24 & 59.14 & 76.09 & 53.66 & 74.60 \\
            VideoJAM-1.3B~\citep{chefer2025videojam} & 91.51 & \trd{96.01} & \trd{99.13} & 96.05 & \trd{73.88} & \trd{58.41} & 62.68 & 79.22 & 59.66 & \trd{79.20} \\
            DreamWorld-1.3B~\citep{tan2026dreamworld} & 93.62 & 94.95 & 98.81 & \trd{98.07} & \snd{79.16} & \snd{59.59} & \snd{66.76} & \snd{81.32} & \trd{65.03} & \snd{81.20} \\
            VPT-1.3B~\citep{zheng2026enhancing} & 92.65 & \fst{97.27} & 98.64 & 97.85 & 70.83 & 58.09 & 63.68 & 73.43 & 51.92 & 74.80 \\
            \midrule
            \textbf{\method (Ours)-1.3B} & \fst{96.68} & \snd{96.95} & \fst{99.81} & \fst{98.85} & 63.54 & \fst{61.65} & \fst{68.66} & \fst{82.81} & \snd{67.72} & 78.00 \\
            \bottomrule
        \end{tabular}%
    }
    \vspace{2pt}
\end{table*}

\begin{table*}[htbp]
    \centering
    \caption{VBench results (Part 2/2): appearance, style and overall scores. All scores are reported on a 0--100 scale.}
    \label{tab:vbench_part2}
    \resizebox{\textwidth}{!}{%
        \begin{tabular}{@{}l*{9}{c}@{}}
            \toprule
            Method & Color & \shortstack{Spatial\\Relationship} & Scene & \shortstack{Appearance\\Style} & \shortstack{Temporal\\Style} & \shortstack{Overall\\Consistency} & \shortstack{Quality\\Score} & \shortstack{Semantic\\Score} & \shortstack{Overall\\Score} \\
            \midrule
            VideoREPA-5B~\citep{zhang2025videorepa} & 86.12 & 51.85 & \fst{32.27} & 21.37 & 22.71 & 24.20 & 79.96 & 63.53 & 76.67 \\
            PhysRAG-5B~\citep{cheng2026physrag} & \fst{93.95} & \snd{77.33} & 15.55 & 22.00 & 23.18 & 23.18 & \trd{83.07} & 68.12 & \trd{80.08} \\
            Wan 2.1-1.3B~\citep{wan2025wan} & 89.93 & 62.37 & 20.03 & \fst{22.91} & 22.23 & 24.25 & 79.81 & 65.43 & 76.93 \\
            VideoJAM-1.3B~\citep{chefer2025videojam} & 88.92 & 66.17 & \trd{28.34} & \snd{22.67} & \snd{24.03} & \fst{25.01} & 81.18 & \trd{69.08} & 78.76 \\
            DreamWorld-1.3B~\citep{tan2026dreamworld} & \snd{92.61} & 70.47 & \snd{29.71} & \trd{22.14} & \trd{23.82} & \snd{24.87} & \snd{83.49} & \fst{70.89} & \snd{80.97} \\
            VPT-1.3B~\citep{zheng2026enhancing} & 86.99 & \trd{73.35} & 20.13 & 20.87 & 22.52 & 23.42 & 82.26 & 64.91 & 78.79 \\
            \midrule
            \textbf{\method (Ours)-1.3B} & \trd{91.71} & \fst{77.88} & 25.15 & 20.98 & \fst{24.05} & \trd{24.31} & \fst{84.69} & \snd{70.55} & \fst{81.86} \\
            \bottomrule
        \end{tabular}%
    }
    \vspace{2pt}
\end{table*}

\paragraph{Visual quality and temporal consistency.}
\method achieves the highest VBench Quality Score ($84.69$) and Overall Score ($81.86$) among the compared methods.
Relative to Wan2.1-T2V-1.3B, these gains are $4.88$ and $4.93$ points, respectively.
The improvements span both individual-frame quality and consistency over time.
Imaging Quality increases from $59.14$ to $68.66$, Aesthetic Quality from $57.24$ to $61.65$ and Subject Consistency from $91.83$ to $96.68$.
Temporal Flickering and Motion Smoothness also improve, reaching $99.81$ and $98.85$.
Together, these metrics indicate that the adapted model improves the appearance of generated frames while maintaining more coherent subjects and smoother temporal evolution.
Compared with DreamWorld, the strongest alternative by Overall Score in this table, \method improves Quality Score by $1.20$ points and Overall Score by $0.89$ points.

\paragraph{Semantic alignment and composition.}
\method improves prompt alignment across object identity, composition and action.
Relative to Wan2.1-T2V-1.3B, Object Class rises from $76.09$ to $82.81$, Multiple Objects from $53.66$ to $67.72$ and Spatial Relationship from $62.37$ to $77.88$.
These improvements support more accurate rendering of the requested entities and their arrangement, and Semantic Score increases from $65.43$ to $70.55$.
Human Action increases from $74.60$ to $78.00$, and Temporal Style from $22.23$ to $24.05$.
The gains therefore extend from the identity and spatial arrangement of objects to the actions and temporal style specified by the prompt.

\paragraph{Physical plausibility with semantic adherence.}
On VideoPhy2 (Table~\ref{tab:physics_generation}), \method improves the ALL-set PC score from $57.8$ to $65.7$ and SA from $30.0$ to $31.3$ relative to Wan2.1-T2V-1.3B.
Joint increases by $2.2$ percentage points to $27.0$, the highest score in the comparison.
This Joint score also exceeds those of Wan2.1-14B ($24.7$) and VideoREPA ($23.6$).
On HARD, \method raises PC from $36.7$ to $51.9$ and Joint from $5.6$ to $6.9$, corresponding to gains of $15.2$ and $1.3$ percentage points.
The improvements on both test splits show that more generated videos depict the requested event while satisfying physical commonsense.

\paragraph{Object, action and physics requirements.}
FinePhyEval provides a complementary decomposition of these gains.
Relative to Wan2.1-T2V-1.3B, \method improves Object from $0.84$ to $0.90$, Action from $0.49$ to $0.55$, Physics from $0.36$ to $0.44$ and Overall from $0.54$ to $0.60$.
It ties DreamWorld on Object and obtains the highest Action and Overall scores among the compared methods.
Its Overall score of $0.60$ also exceeds the $0.59$ achieved by Wan2.1-14B.
The question-graph metric rewards physical behavior together with the objects and actions on which it depends.
The gains across all three categories therefore support more complete realization of the requested physical events, complementing the visual and temporal improvements measured by VBench.

\subsubsection{Qualitative Results}

Figures~\ref{fig:appendix_t2v_objects}, \ref{fig:appendix_t2v_paddling_racing} and~\ref{fig:appendix_t2v_parasailing} provide additional qualitative comparisons of VideoREPA, VPT, PhysRAG and \method (Ours).
The five examples cover domino interactions, multiple-object composition, canoe paddling, jetski racing and parasailing.
For each case, rows follow this method order from top to bottom, and columns show five frames sampled at evenly spaced relative positions across each video, progressing from left to right.
The text below each case gives the corresponding prompt description.

\begin{figure}[!htb]
    \centering
    \includegraphics[width=\textwidth]{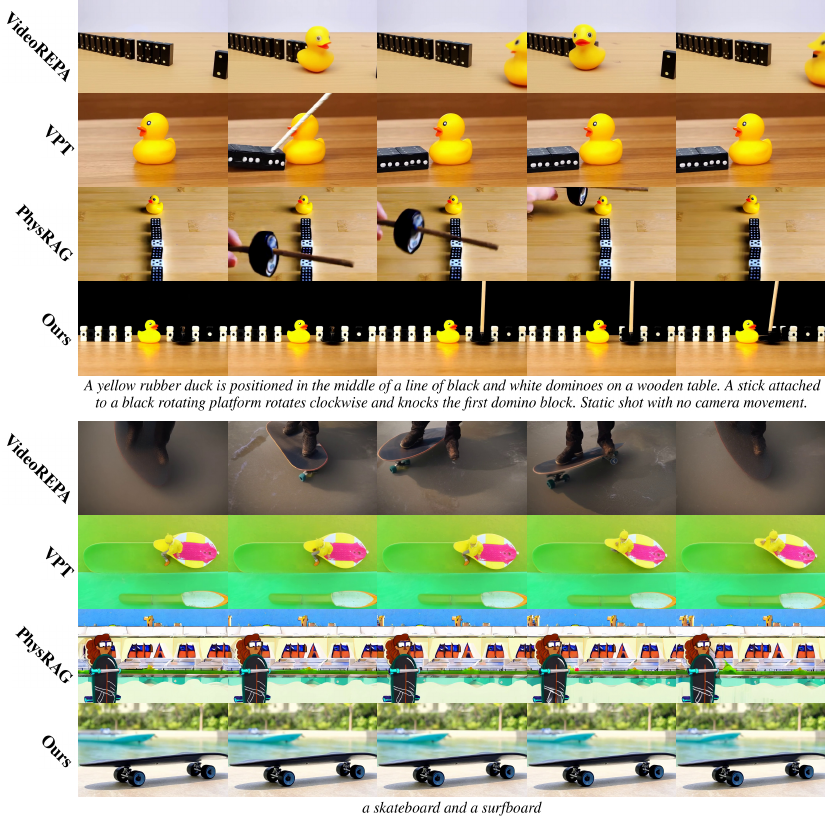}
    \caption{
        \textbf{Qualitative results on text-to-video generation.}
        The upper case shows a rubber duck among dominoes, and the lower case shows a skateboard and a surfboard.
        Within each case, rows show VideoREPA, VPT, PhysRAG and Ours from top to bottom.
        Each row contains five frames sampled at evenly spaced relative positions across the video, ordered from left to right.
    }
    \label{fig:appendix_t2v_objects}
\end{figure}

\begin{figure}[!htb]
    \centering
    \includegraphics[width=\textwidth]{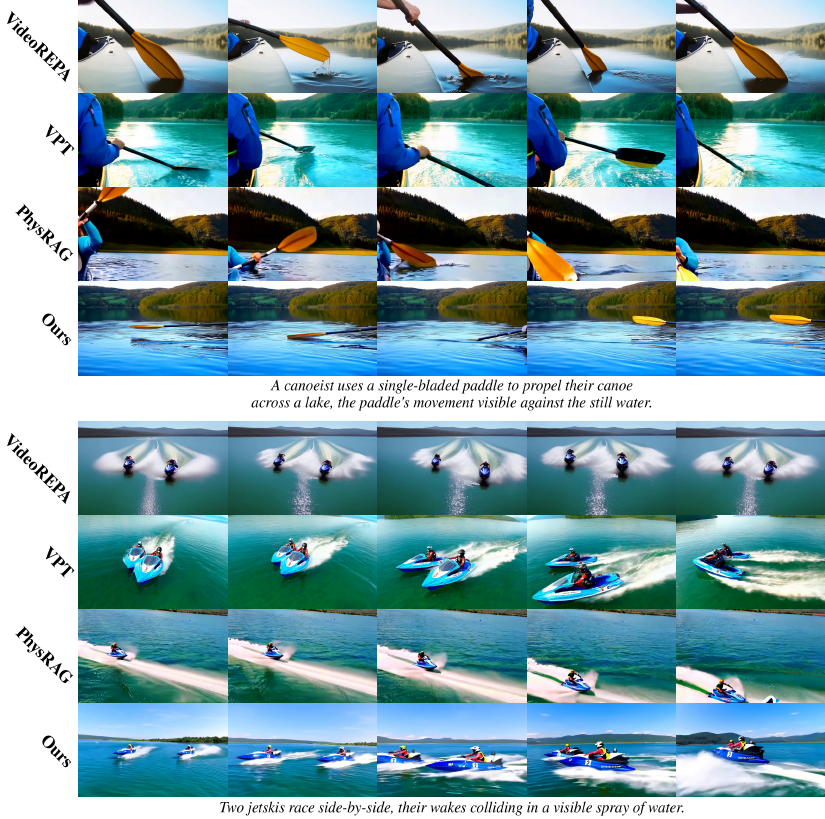}
    \caption{
        \textbf{Qualitative results on text-to-video generation.}
        The upper case shows a canoeist paddling across a lake, and the lower case shows two jetskis racing side by side.
        Within each case, rows show VideoREPA, VPT, PhysRAG and Ours from top to bottom.
        Each row contains five frames sampled at evenly spaced relative positions across the video, ordered from left to right.
    }
    \label{fig:appendix_t2v_paddling_racing}
\end{figure}

\begin{figure}[!htb]
    \centering
    \includegraphics[width=\textwidth]{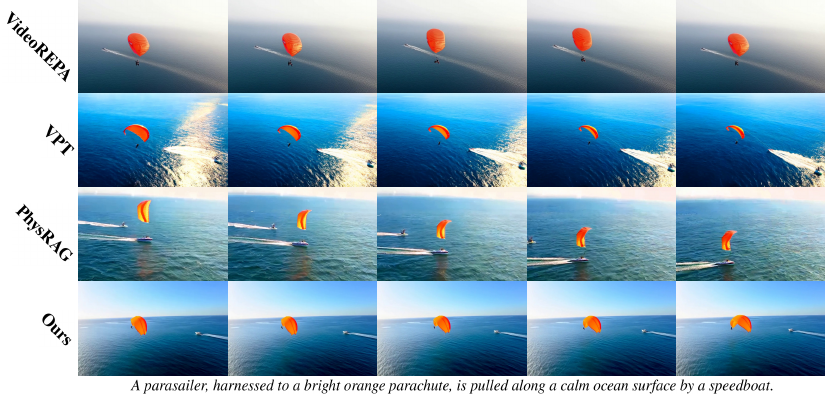}
    \caption{
        \textbf{Qualitative results on text-to-video generation.}
        The prompt describes a parasailer with an orange parachute being pulled across the ocean by a speedboat.
        Rows show VideoREPA, VPT, PhysRAG and Ours from top to bottom.
        Each row contains five frames sampled at evenly spaced relative positions across the video, ordered from left to right.
    }
    \label{fig:appendix_t2v_parasailing}
\end{figure}

\subsection{More Results on Image-to-Video Generation}
\label{sec:appendix_i2v_results}

\subsubsection{Quantitative Results}

\paragraph{Aggregate physical consistency.}
\method achieves the highest Verified Score in Table~\ref{tab:physics_iq_verified}, reaching $30.95$ compared with $27.71$ for Wan2.2-5B.
The gain is $3.24$ points over the base model and $4.47$, $3.52$ and $5.41$ points over WMReward, VLIPP and PhysRAG, respectively.
This score combines motion-region agreement, temporal overlap and frame reconstruction after normalization by variation between real recordings.

\paragraph{Motion coverage and frame appearance.}
Relative to Wan2.2-5B, \method increases Spatial from $35.89$ to $43.43$ and Weighted Spatial from $22.96$ to $24.74$.
These gains indicate closer agreement with the reference in both the regions traversed by moving objects and the fraction of time each region exhibits motion.
MSE Score increases from $25.54$ to $29.24$, indicating improved agreement in RGB frame appearance under the benchmark's normalization.
\method leads the compared methods on all three metrics, supporting the aggregate improvement across complementary aspects of the generated continuation.

\subsubsection{Qualitative Results}

Figure~\ref{fig:i2v_qualitative_comparison} compares \method with WMReward and PhysRAG under the same image and text conditions.
In the moving-light example, \method retains a recognizable statuette silhouette across the displayed frames as its shadow changes position.
The mandarin example shows the knife and fruit remaining identifiable through the slicing motion.
These examples illustrate object and shadow consistency alongside the motion and appearance metrics above.

Figures~\ref{fig:appendix_i2v_water} and~\ref{fig:appendix_i2v_reflections} extend this comparison to three additional scenarios, with ground-truth videos (GT) alongside WAN, VLIPP, WMReward, PhysRAG and \method.
The paper-in-water example shows how each method depicts the paper after release and preserves the surrounding scene, while the mirror examples show the relationship between object motion and its reflection.
Each row presents five video frames in temporal order.

Figure~\ref{fig:appendix_i2v_ball_motion} adds a bouncing tennis ball and a blue ball rolling from a table onto glass, with the same six rows per case.
These examples show how each model continues object motion from the initial frame.
Red boxes mark the visible balls and adjacent reflections in the four baseline methods, while GT and \method remain unmarked.

\begin{figure}[!htb]
    \centering
    \includegraphics[width=\textwidth]{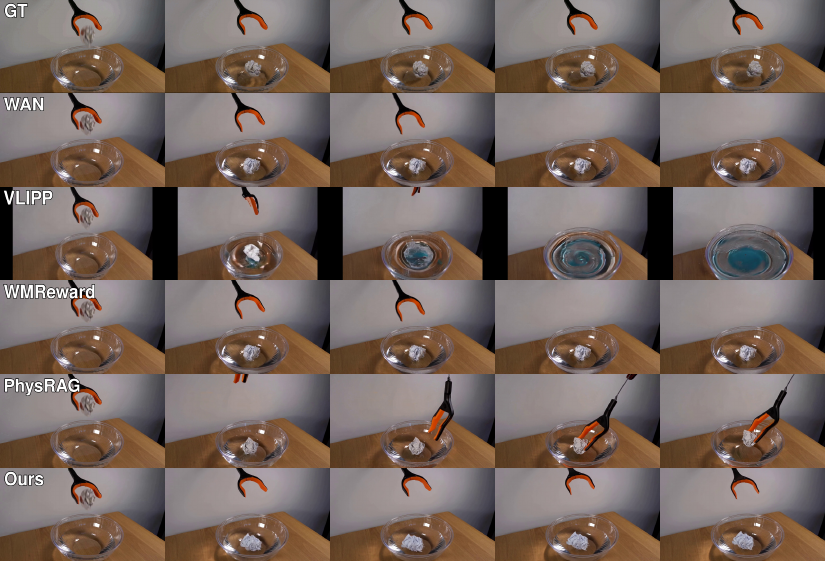}
    \caption{
        \textbf{Additional image-to-video comparison on Physics-IQ Verified: paper falling into water.}
        A grabber releases crumpled paper into a bowl of water.
        Rows show GT, WAN, VLIPP, WMReward, PhysRAG and \method (Ours).
        Each row contains five video frames in temporal order.
    }
    \label{fig:appendix_i2v_water}
\end{figure}

\begin{figure}[p]
    \centering
    \includegraphics[width=\textwidth]{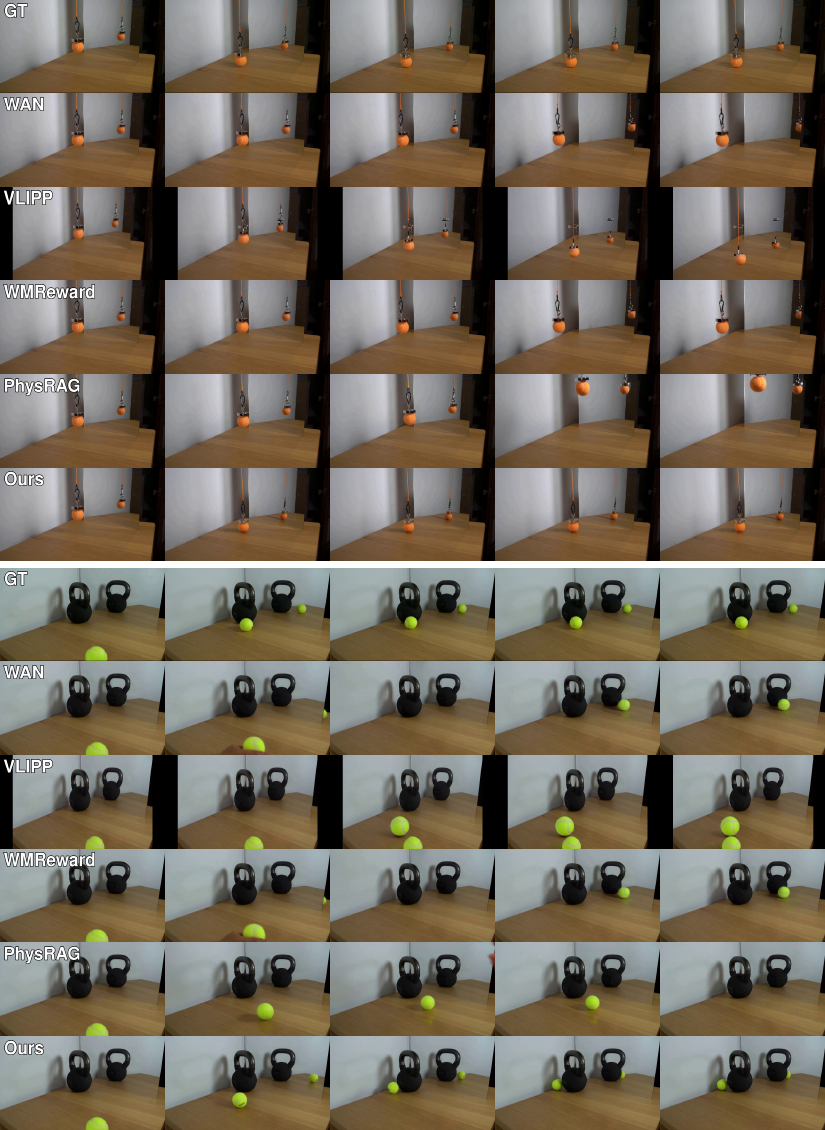}
    \caption{
        \textbf{Additional image-to-video comparisons on Physics-IQ Verified: motion and reflections.}
        The upper case shows a suspended tennis ball being lowered in front of a mirror, and the lower case shows a tennis ball rolling toward a kettlebell beside a mirror.
        Within each case, rows show GT, WAN, VLIPP, WMReward, PhysRAG and \method (Ours).
        Each row contains five video frames in temporal order.
    }
    \label{fig:appendix_i2v_reflections}
\end{figure}

\begin{figure}[p]
    \centering
    \includegraphics[width=\textwidth]{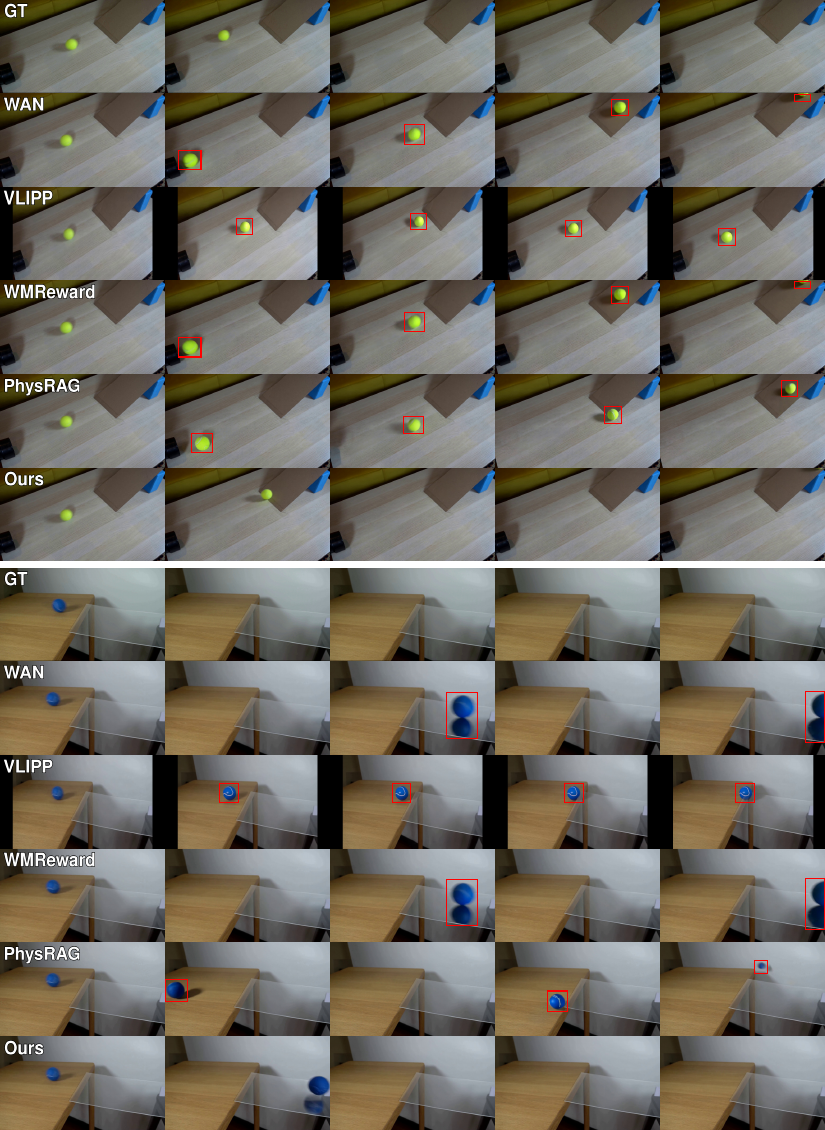}
    \caption{
        \textbf{Additional image-to-video comparisons on Physics-IQ Verified: ball motion.}
        The upper case shows a bouncing tennis ball near a ramp, and the lower case shows a blue ball rolling from a table onto glass.
        Within each case, rows show GT, WAN, VLIPP, WMReward, PhysRAG and \method (Ours).
        Each row contains five video frames in temporal order.
        Red boxes highlight visible balls and adjacent reflections in the four baselines after the initial frame.
    }
    \label{fig:appendix_i2v_ball_motion}
\end{figure}

\subsection{Full-Precision Ablation Results}
\label{sec:appendix_ablation_results}

Table~\ref{tab:ablation_full} reports all ablations from Section~\ref{sec:ablation} on the original 0--1 scale with four decimal places.
Table~\ref{tab:ablation} displays the same scores as percentages with three significant digits, while retaining the ranking colors determined at four-decimal precision.
Rounding can make distinct scores appear tied. For example, the Imaging Quality scores of $0.6369$ and $0.6368$ both display as $63.7\%$, and the Quality Scores of $0.7980$ and $0.7979$ both display as $79.8\%$.

% Original four-decimal scores from tables/mistakeflow_ablation.tex,
% preserved before converting the main-text table to rounded percentages.
\begin{table}[!htbp]
    \centering
    \caption{Component ablations at four-decimal precision (0--1, higher is better). Column abbreviations match Table~\ref{tab:ablation}. Colors indicate the top three scores in each column.}
    \label{tab:ablation_full}
    \resizebox{\textwidth}{!}{%
    \begin{tabular}{lcccccc}
        \toprule
        Method
        & \shortstack{Aesthetic\\Quality (AQ)}
        & \shortstack{Background\\Consistency (BC)}
        & \shortstack{Imaging\\Quality (IQ)}
        & \shortstack{Overall\\Consistency (OC)}
        & \shortstack{Subject\\Consistency (SC)}
        & \shortstack{Quality\\Score (QS)} \\
        \midrule
        w/o time-step
        & \trd{0.5147} & 0.9236 & \snd{0.6369} & 0.5850 & 0.9120 & 0.7970 \\

        w/o self-attention
        & 0.4892 & 0.9213 & 0.6043 & 0.5812 & 0.9000 & 0.7854 \\

        w/o cross-attention
        & 0.4848 & 0.9144 & 0.6066 & 0.5821 & 0.9022 & 0.7846 \\

        w/o hidden-state MLP
        & 0.4928 & 0.9238 & 0.6359 & 0.5817 & 0.9086 & 0.7928 \\

        w/o refiner
        & 0.5129 & \trd{0.9249} & 0.6204 & \trd{0.5914} & \trd{0.9126} & 0.7863 \\

        $\bm{v}$ input
        & 0.5038 & 0.9238 & \trd{0.6368} & \snd{0.6021} & \snd{0.9174} & \snd{0.7980} \\

        w/o condition
        & \snd{0.5154} & \snd{0.9274} & 0.6123 & 0.5812 & 0.9102 & \trd{0.7979} \\

        \textbf{Ours}
        & \fst{0.5260}
        & \fst{0.9359}
        & \fst{0.6425}
        & \fst{0.6179}
        & \fst{0.9260}
        & \fst{0.8026} \\
        \bottomrule
    \end{tabular}%
    }
\end{table}

\section{Reproducibility Statement}
\label{sec:appendix_reproducibility}

To support reproducibility, Section~\ref{sec:methodology} specifies the refiner architecture, residual training objective and inference procedure of \method.
Appendix~\ref{sec:appendix_implementation_details} documents the pretrained models, training data, preprocessing, optimization hyperparameters, hardware and sampling settings for image, text-to-video and image-to-video generation, together with the benchmark protocols and metric definitions.
Section~\ref{sec:ablation} describes the ablation setup, and Appendix~\ref{sec:appendix_ablation_results} provides the corresponding scores at full precision.
The assumptions and complete proofs underlying the theoretical results in Section~\ref{sec:theory} are provided in Appendices~\ref{p1:sec:setup}, \ref{p2:app:matched-resource-proof} and~\ref{p2:app:deployment}.

\end{document}